\documentclass[final,12pt]{colt2021} 

\title[Robust learning under clean-label attack]{Robust learning under clean-label attack}
\usepackage{times}
\usepackage{hyperref}
\usepackage{bm}
\usepackage{enumitem,array,booktabs}
\usepackage{algorithm}
\usepackage{algorithmic}
\usepackage{bbm}
\usepackage{tabulary,multirow}
\usepackage{dsfont}
\usepackage{xcolor}

\makeatletter
\def\set@curr@file#1{\def\@curr@file{#1}} 
\makeatother

\DeclareMathOperator{\Log}{Log}
\DeclareMathOperator{\poly}{poly}

\newcommand{\inner}[2]{\left\langle #1, #2 \right\rangle}

\newcommand{\ceil}[1]{\left\lceil#1\right\rceil}
\newcommand{\floor}[1]{\left\lfloor#1\right\rfloor}

\newcommand{\la}{\leftarrow}

\newcommand{\Ber}{\mathrm{Ber}}
\newcommand{\Unif}{\mathrm{Unif}}
\newcommand{\Bin}{\mathrm{Bin}}

\newtheorem{conj}{Conjecture}
\newtheorem{lm}{Lemma}
\newtheorem{thm}{Theorem}
\newtheorem{dfn}{Definition}
\newtheorem{crl}{Corollary}
\newtheorem{rmk}{Remark}

\newtheorem{eg}{Example}

\newcommand{\R}{\mathbb{R}}

\newcommand{\nats}{\mathbb{N}}

\newcommand{\NN}{{\mathbb N}}

\newcommand{\ind}[1]{\mathds{1}[#1]}

\newcommand{\bZero}{{\bm 0}}
\renewcommand{\P}{\mathbb{P}}
\newcommand{\E}{\mathbb{E}}
\newcommand{\EE}[1]{\mathbb{E}\left[#1\right]}

\newcommand{\EEs}[2]{\mathbb{E}_{#1}\left[#2\right]}

\newcommand{\PP}[1]{\mathbb{P}\left(#1\right)}
\newcommand{\PPs}[2]{\mathbb{P}_{#1}\left(#2\right)}

\newcommand{\norm}[1]{\left\|#1\right\|}

\newcommand{\abs}[1]{\left|#1\right|}

\newcommand*{\eqdef}{\triangleq}

\newcommand{\cA}{\mathcal{A}}
\newcommand{\cB}{\mathcal{B}}

\newcommand{\cD}{\mathcal{D}}
\newcommand{\cE}{\mathcal{E}}

\newcommand{\cH}{\mathcal{H}}

\newcommand{\cL}{\mathcal{L}}

\newcommand{\cM}{\mathcal{M}}

\newcommand{\cP}{\mathcal{P}}

\newcommand{\cX}{\mathcal{X}}

\newcommand{\cY}{\mathcal{Y}}

\newcommand{\eps}{\varepsilon}
\renewcommand{\epsilon}{\varepsilon}
\renewcommand{\hat}{\widehat}
\renewcommand{\tilde}{\widetilde}
\renewcommand{\bar}{\overline}

\newcommand{\nothere}[1]{}

\newcommand{\err}{\mathrm{err}}
\newcommand{\vcd}{\mathrm{VCdim}}

\newcommand{\VS}{\mathrm{VS}}
\newcommand{\SVM}{\mathrm{SVM}}

\newcommand{\DIS}{\mathrm{DIS}}
\newcommand{\ATK}{\mathrm{ATK}}
\newcommand{\Adv}{\mathrm{Adv}}
\newcommand{\atk}{\mathrm{atk}}

\newcommand{\trn}{\mathrm{trn}}

\newcommand{\ERM}{{\mathrm{ERM}}}

\newcommand{\Proj}{{\mathrm{Proj}}}
\newcommand{\Rfl}{{\mathrm{Ref}}}
\newcommand{\A}{\mathcal{A}}

\newcommand{\Major}{{\mathrm{Major}}}

\newcommand{\supp}{{\mathrm{supp}}}
\newcommand{\po}{\leq_f^\cH}
\newcommand{\clsr}{\mathrm{Closure}}
\newcommand{\sph}{\Gamma}
\newcommand{\rbst}{{\mathrm{rbst}}}

\newcommand{\wt}[1]{\widetilde{#1}}

\newcommand{\CR}{\mathrm{CR}}

\renewenvironment{proof}[1][]{\par\noindent{\bf Proof #1\ }}{\hfill\BlackBox\\[2mm]}

\coltauthor{%
  \Name{Avrim Blum} \Email{avrim@ttic.edu}  \\
  \addr 
  Toyota Technological Institute at Chicago
  \AND
  \Name{Steve Hanneke} \Email{steve.hanneke@gmail.com} \\
  \addr 
  Toyota Technological Institute at Chicago
  \AND
  \Name{Jian Qian} \Email{jianqian@mit.edu} \\
  \addr Massachusetts Institute of Technology
  \AND
  \Name{Han Shao} \Email{han@ttic.edu}\\
  \addr 
  Toyota Technological Institute at Chicago
}

\begin{document}

\maketitle
\begin{abstract}%
    We study the problem of robust learning under clean-label data-poisoning attacks, where the attacker injects (an arbitrary set of) \emph{correctly-labeled} examples to the training set to fool the algorithm into making mistakes on \emph{specific} test instances at test time. The learning goal is to minimize the attackable rate (the probability mass of attackable test instances), which is more difficult than optimal PAC learning. As we show, any robust algorithm with diminishing attackable rate can achieve the optimal dependence on $\epsilon$ in its PAC sample complexity, i.e., $O(1/\epsilon)$. On the other hand, the attackable rate might be large even for some optimal PAC learners, e.g., SVM for linear classifiers. Furthermore, we show that the class of linear hypotheses is not robustly learnable when the data distribution has zero margin and is robustly learnable in the case of positive margin but requires sample complexity exponential in the dimension. For a general hypothesis class with bounded VC dimension, if the attacker is limited to add at most $t>0$ poison examples, the optimal robust learning sample complexity grows almost linearly with $t$.
    \end{abstract}
    
\section{Introduction}
Data poisoning is an attack on machine learning algorithms where the attacker adds examples to the training set with the goal of causing the algorithm to produce a classifier that makes specific mistakes the attacker wishes to induce at test time. In this paper, we focus on clean-label attacks in which an attacker, with knowledge of the training set $S$ and the test instance $x$, injects a set of examples labeled by the target function into the training set with the intent of fooling the learner into misclassifying the test instance $x$. This type of attack is called a clean-label attack because the attacker can only add correctly-labeled examples to the training set, and it has been proposed and studied empirically by~\cite{shafahi2018poison}.

In the realizable setting when the target function belongs to the hypothesis class $\cH$, any empirical risk minimizer (ERM) will achieve error of $\tilde{O}(\frac{\text{VCdim}(\cH)}{m})$ with training set size $m$.  This means that an ERM learner will still have error rate at most $\tilde{O}(\frac{\text{VCdim}(\cH)}{m})$ even in the presence of a clean-label attack; i.e., the attacker cannot significantly increase the overall {\em error rate}.   However, an attacker could still cause the ERM learner to make {\em specific} mistakes that the attacker wishes. For example, consider an ERM learner for the hypothesis class of intervals over $[0,1]$ that predicts the positive interval of maximum length consistent with the training data, in the case that the target function labels all of $[0,1]$ negative. Then any test instance not in the training set is attackable for this ERM learner by an adversary that adds enough poison examples so that the interval that the test instance is in becomes the largest interval in the training set.  On the other hand, for any target interval, for the ERM learner that outputs the {\em smallest} consistent interval, the attackable test instances will only have probability mass $O(1/m)$ (see Example~\ref{eg:interval} for more details).  Also, notice that for the hypothesis class of threshold functions over $[0,\infty)$, any ERM learner has a small portion of attackable test instances because the disagreement region of all consistent hypotheses is small and only test instances in the disagreement region are attackable.

From these examples, we can see that given an ERM learning algorithm $\A$ and a training set $S$, the probability mass of the attackable region (the set of attackable test instances) is at least as large as the error rate of the ERM learner and no greater than the disagreement region of all consistent hypotheses, and it depends on the specific algorithm $\A$. In this paper, we study the problem of whether we can obtain a small rate of attackable test instances in binary classification.
In the process we find interesting connections to existing literature on the sample complexity of PAC learning, and complexity measures arising in that literature.  Specifically, we study this problem in the realizable setting as it is unclear how to best define ``clean-label'' in the agnostic case.



\paragraph{Related work}
Clean-label data-poisoning attacks have been studied extensively in the literature \citep{shafahi2018poison,suciu2018does}, and  \cite{shafahi2018poison} show that clean-label attacks can be very  effective on neural nets empirically. 
For example, \cite{shafahi2018poison} show that in natural image domains, given the knowledge of the training model and of the test point to be attacked, the attacker can cause the model retrained with an injection of clean-label poisoned data to misclassify the given test instance with high success rate. Moreover, the attacker is able to succeed even though the overall error rate of the trained classifier remains relatively unchanged. 

\cite{mahloujifar2017blockwise, mahloujifar2018learning, mahloujifar2019universal} study a class of clean-label poisoning attacks called $p$-tampering attacks, where the attacker can substitute each training example with a correctly labeled poison example with independent probability $p$, and its variants.  \cite{mahloujifar2019can,mahloujifar2019curse,etesami2020computational} consider a more powerful adversary that can attack training examples of its choosing (rather than chosen at random) and show that the attacker can increase the probability of failing on a particular test instance from any non-negligible probability $\Omega(1/\poly(m))$ to $\approx 1$ by replacing $\tilde{O}(\sqrt{m})$ training examples with other correctly labeled examples.  In contrast, in our setting the attacker cannot modify any of the existing training examples and can only add new ones.  In addition, we mainly focus on attacks with an unlimited budget.

Data poisoning without requiring the poisoned data to be clean has been studied extensively (see \cite{biggio2012poisoning,barreno2006can,papernot2016towards,steinhardt2017certified} for a non-exhaustive list). Robustness to data poisoning with a small portion of poison examples has been studied by~\cite{ma2019data,levine2020deep}. The concurrent work of \cite{gao2021learning} studies the instance-targeted poisoning risk (which is the probability mass of the attackable region in the classification task) by various attacker classes, which have a budget controlling the amount of training data points they can change. They mainly focus on the relationship between robust learnability and the budget.

There are other studied attacking methods, including perturbation over training examples~\citep{koh2017understanding}, perturbation over test examples~\citep{szegedy2013intriguing,goodfellow2014explaining,bubeck2019adversarial,cullina2018pac,montasser2019vc,montasser2020efficiently} and etc. Another different notion of robust learning is studied by~\cite{xu2012robustness}, where the data set is partitioned into several subsets and the goal is to ensure the losses of instances falling into the same subset are close. Another line of related work is covariate shift, where the training distribution is different from the test distribution (see~\cite{quionero2009dataset} for an extensive study).

\paragraph{Notation}
For any vectors $u,v$, we let $\norm{u}$ denote the $\ell_2$ norm of $u$ and $\theta(u,v)$ denote the angle of $u$ and $v$. We denote by $e_i\in \R^n$ the one-hot vector with the $i$-th entry being one and others being zeros. We let $\cB^n(c,r) = \{x|\norm{x-c} \leq r\}$ denote the the ball with radius $r$ centered at $c\in \R^n$ in the $n$-dimensional space and $\sph^n(c,r)$ denote the sphere of $\cB^n(c,r)$. We omit the supscript $n$ when it is clear from the context. For any $a,b\in \R$, denote $a\wedge b = \min (a,b)$ and $a\vee b = \max (a,b)$. 
We use $\ln$ to represent natural logarithms and $\log$ to represent logarithms with base $2$.
Given a data set $S=\{(x_1,y_1),\ldots,(x_m,y_m)\}$ with size $m$, for any hypothesis $h$, we let $\err_S(h) = \frac{1}{m}\sum_{i=1}^m \ind{h(x_i)\neq y_i}$ denote the empirical error of $h$ over $S$. For a data distribution $\cD$, we let $\err_\cD(h) = \EEs{(x,y)\sim \cD}{\ind{h(x)\neq y}}$ denote the error of $h$. For any $A\subseteq \cX$, we let $\cP_\cD(A) = \PPs{(x,y) \sim \cD}{x \in A}$ denote the probability mass of $A$. The subscript $\cD$ is omitted when it is clear from the context. For any data set $S$, we let $S_\cX=\{x|(x,y)\in S\}$ and for $(x,y)\sim \cD$, we let $\cD_{\cX}$ denote the marginal distribution of $x$. For a finite set of hypotheses $\cH$, we let $\Major(\cH)$ denote the majority vote of $\cH$ and for simplicity denote $\Major(\cH,x) = \Major(\cH)(x) \triangleq \ind{\sum_{h\in\cH}h(x)\geq \ceil{{\abs{H}}/{2}}}$.
\section{Problem setup and summary of results}
Let $\cX$ denote the instance space and $\cY = \{0,1\}$ denote the label space. Given a hypothesis class $\cH\subseteq \cY^\cX$, we study the realizable case where there exists a deterministic target function $h^*\in \cH$ such that the training set and the test set are realized by $h^*$. Let $D_{h^*} = \{(x,h^*(x))|x\in \cX\}$ denote the data space where every instance is labeled by $h^*$. A learning algorithm $\A$ is a map (possibly including randomization), from a labeled data set $S$ (an unordered multiset) of any size, to a hypothesis $h$, and for simplicity we denote by $\cA(S,x) = \cA(S)(x)$ the prediction of $\cA(S)$ at an instance $x$. An attacker $\Adv$ maps a target function $h^*$, a training data set $S_\trn$ and a specific test instance $x$ to a data set $\Adv(h^*,S_\trn,x)$ (a multiset) and injects $\Adv(h^*,S_\trn,x)$ into the training set with the intent of making the learning algorithm misclassify $x$. We call $\Adv$ a \emph{clean-label} attacker if $\Adv(h^*,S_\trn,x)$ is consistent with $h^*$. Then for any deterministic algorithm $\A$, we say a point $x\in\cX$ is attackable if there exists a clean-label attacker $\Adv$ such that
\begin{align*}
    \A(S_\trn\cup \Adv(h^*,S_\trn,x),x)\neq h^*(x)\,.
\end{align*}
To be clear, we are defining $S_\trn\cup \Adv(h^*,S_\trn,x)$ as an unordered multiset. 
Formally, we define clean-label attackable rate as follows.
\begin{dfn}[clean-label attackable rate]\label{dfn:rate}
  For a target function $h^*$, a training data set $S_\trn$ and a (possibly randomized) algorithm $\A$, for any distribution $\cD$ over $D_{h^*}$, the attackable rate by $\Adv$ for $(h^*, S_\trn, \A)$ is defined as
    \begin{align*}
      \atk_\cD(h^*, S_\trn,\A,\Adv) \triangleq \EEs{(x,y)\sim \cD,\cA}{\ind{\A(S_\trn\cup \Adv(h^*,S_\trn,x),x)\neq h^*(x)}}\,.  
    \end{align*}
    The clean-label attackable rate is defined by the supremum over all clean-label attackers, i.e., 
    \begin{align*}
      \atk_\cD(h^*, S_\trn,\A) \triangleq \sup_\Adv \atk_\cD(h^*, S_\trn,\A,\Adv)\,.  
    \end{align*}
\end{dfn}
Then we define our learning problem as follows.
\begin{dfn}[$(\epsilon,\delta)$-robust learnability]\label{dfn:learnablity}
  For any $\epsilon,\delta \in (0,1)$, the sample complexity of $(\epsilon,\delta)$-robust learning of $\cH$, denoted by $\cM_{\rbst}(\epsilon,\delta)$, is defined as the smallest $m\in \NN$ for which there exists an algorithm $\A$ such that for every target function $h^*\in\cH$ and data distribution over $D_{h^*}$, with probability at least $1-\delta$ over $S_\trn\sim \cD^m$,
  \[\atk_\cD(h^*, S_\trn,\A)\leq \epsilon\,.\]
  If no such $m$ exists, define $\cM_{\rbst}(\epsilon,\delta)=\infty$. We say that $\cH$ is $(\epsilon,\delta)$-robust learnable if $\forall \epsilon,\delta \in (0,1)$, $\cM_{\rbst}(\epsilon,\delta)$ is finite.
\end{dfn}
It is direct to see that the error of $\A(S_\trn)$ is the attackable rate by attacker $\Adv_0$ which injects an empty set to the training set, i.e., $\Adv_0(\cdot)=\emptyset$. Therefore, for any algorithm $\A$, we have 
\[
  \atk_\cD(h^*, S_\trn,\A)\geq \atk_\cD(h^*, S_\trn,\A,\Adv_0) = \err_\cD(\A(S_\trn))\,,
  \]
which indicates any hypothesis class that is not PAC learnable is not robust learnable. For any deterministic $\A$, let us define $\ATK(h^*, S_\trn,\A,\Adv) \triangleq \{x\in \cX| \A(S_\trn\cup \Adv(h^*,S_\trn,x),x)\neq h^*(x)\}$ the attackable region by $\Adv$. For any ERM learner and any clean-label attacker $\Adv$, we have $\ATK(h^*, S_\trn,\ERM,\Adv)\subseteq \DIS(\VS_{\cH,S_\trn})$,
where $\VS_{\cH,S_\trn}$ is the version space of $S_\trn$, i.e., the set of all hypotheses in $\cH$ that classify $S_\trn$ correctly and $\DIS(\VS_{\cH,S_\trn}) = \{x|\exists h,h'\in \VS_{\cH,S_\trn}, h(x)\neq h'(x)\}$ is the disagreement region of the version space. Therefore, we have
\[
  \inf_\A\atk_\cD(h^*, S_\trn,\A) \leq \cP_{\cD}(\DIS(\VS_{\cH,S_\trn}))\,.
\]
However, large $\cP_{\cD}(\DIS(\VS_{\cH,S_\trn}))$ does not always result in large clean-label attackable rate. Below is an example showing the gap between them.


\begin{eg}[Interval over ${[0,1]}$]\label{eg:interval}
  The hypothesis class $\cH=\{\ind{(a,b)}: 0\leq a\leq b\leq 1\}\cup \{\ind{[a,b]}: 0\leq a\leq b\leq 1\}$ contains all intervals on $[0,1]$. We consider the following two learners.
  \begin{itemize}
    \setlength{\itemsep}{0pt}
    \setlength{\parsep}{0pt}
    \setlength{\parskip}{0pt}
    \item $\A_1(S)$: return $\ind{\emptyset}$ (the empty interval) if there are no positive examples in $S$ and return the consistent positive closed interval with minimum length otherwise.
    \item $\A_2(S)$: return the consistent positive open interval with maximum length.
    \end{itemize}
  Both are ERM learners for $\cH$. For any $h^* \in \cH$,  let the data distribution $\cD$ be a distribution on $D_{h^*}$ and $S_\trn\sim \cD^m$ for any $m>0$, then $\A_1$'s attackable rate is $\atk_\cD(h^*, S_\trn,\A_1)=\err_\cD(\A_1(S_\trn)) = \tilde{O}(1/m)$. However, consider algorithm $\A_2$ with $h^* = \ind{\emptyset}$. For $S_\trn = \{(x_1,y_1),\ldots,(x_m,y_m)\}$, w.l.o.g. assume that $x_1\leq \ldots\leq x_m$ and let $x_0=0$, $x_{m+1}=1$ for notation simplicity. Then for any $x\in (x_i,x_{i+1})$, the attacker can add enough poison data points to intervals $\{(x_j,x_{j+1})\}_{j\neq i}$ to make $(x_i,x_{i+1})$ be the interval with the maximum length. Therefore, so long as $\cD$ has no point masses, $\A_2$'s attackable rate is $\atk_\cD(h^*, S_\trn,\A_2) = \cP_{\cD}(\DIS(\VS_{\cH,S_\trn}))=1$.
\end{eg}

\paragraph{Main results}
We summarize the main contributions of this work.
\begin{itemize}
  \item In Section~\ref{sec:examples}, we present results on robust learnability under assumptions based on some known structural complexity measures, e.g., VC dimension $d=1$, hollow star number $k_o=\infty$, etc. In addition, we show that all robust algorithms can achieve optimal dependence on $\epsilon$ in their PAC sample complexity.
  
  \item In Section~\ref{sec:linear}, we show that the $n$-dimensional linear hypothesis class with $n\geq 2$ is not $(\epsilon,\delta)$-robust learnable. Then we study the linear problem in the case where the data distribution $\cD$ has margin $\gamma>0$. We propose one algorithm with sample complexity $O({n}{({2}/{\gamma})^n}\log(2/\gamma))$ and show that the optimal sample complexity is $ e^{\Omega(n)}$. We propose another algorithm in $2$-dimensional space with sample complexity $O(\log({1}/{\gamma})\log\log({1}/{\gamma}))$. We also show that even in the case where $\gamma$ is large and the attacker is only allowed to inject one poison example into the training set, SVM requires at least $ e^{\Omega(n)}$ samples to achieve low attackable rate.
 
  \item In Section~\ref{sec:finite}, we show that for any hypothesis class $\cH$ with VC dimension $d$, when the attacker is restricted to inject at most $t$ poison examples, $\cH$ is robust learnable with sample complexity $\widetilde{O}(\frac{dt}{\epsilon})$. We also show that there exists a hypothesis class with VC dimension $d$ such that any algorithm requires $\Omega(\frac{dt}{\epsilon})$ samples to achieve $\epsilon$ attackable rate.

\end{itemize}

\section{Connections to some known complexity measures and PAC learning}\label{sec:examples}
In this section, we analyze the robust learnability of hypothesis classes defined by a variety of known structural complexity measures. For some of these, we show they have the good property that there exists an algorithm such that adding clean-label points can only change the predictions on misclassified test instances and thus, the algorithm can achieve $\atk(h^*,S_\trn,\cA)\leq \err(\A(S_\trn))$. For some other structure, we prove that there will be a large attackable rate for any consistent proper learner. We also show the connection to optimal PAC learning in Section~\ref{sec:opt-pac}.

\subsection{Connections to some known complexity measures}
\paragraph{Hypothesis classes with VC dimension $d=1$ are $(\epsilon,\delta)$-robust learnable.} 
First, w.l.o.g., assume that for every $x\neq x'\in\cX$, there exists $h\in\cH$ such that $h(x)\neq h(x')$ (otherwise, operate over the appropriate equivalence classes). Then we adopt the partial ordering $\leq^\cH_f$ for any $f\in\cH$ over $\cX$ proposed by~\cite{ben20152} defined as follows.
\begin{dfn}[partial ordering $\leq^\cH_f$]
  For any $f\in\cH$, 
  \begin{align*}
    \leq^\cH_f \eqdef \{(x,x')|\forall h\in\cH, h(x')\neq f(x')\Rightarrow h(x)\neq f(x)\}\,.
  \end{align*}
\end{dfn} 
By Lemma~5 of~\cite{ben20152}, $\leq^\cH_f$ for $d=1$ is a tree ordering. Due to this structural property of hypothesis classes with VC dimension $d=1$, there is an algorithm 
originally proposed by \cite{ben20152} (Algorithm~\ref{alg:vc1} in Appendix~\ref{appx:vcd1}) such that adding clean-label poison points can only narrow down the error region (the set of misclassified instances). 
Roughly, the algorithm finds a maximal (by $\leq^\cH_f$) point $x'$ in the data such that 
$h^*(x') \neq f(x')$, and outputs the classifier 
labeling all $x \leq^\cH_f x'$ as $1-f(x)$ 
and the rest as $f(x)$.
We show that this algorithm can robustly learn $\cH$ using $m$ samples, where
  \[
    m= \frac{2\ln(1/\delta)}{\epsilon}\,.
  \]
The detailed algorithm and proof are given in Appendix~\ref{appx:vcd1}.
\paragraph{Intersection-closed hypothesis classes are $(\epsilon,\delta)$-robust learnable.}
A hypothesis class $\cH$ is called intersection-closed if the collection of sets $\{\{x|h(x)=1\}|h\in \cH\}$ is closed under intersections, i.e., $\forall h,h'\in\cH$, the classifier $x\mapsto \ind{h(x)=h'(x)=1}$ is also contained in $\cH$. For intersection-closed hypothesis classes, there is a general learning rule, called the Closure algorithm~\citep{helmbold:90,auer:07}. For given data $S$, the algorithm outputs $\hat{h} = \ind{\{x|\forall h\in\VS_{\cH,S}, h(x)=1\}}$. Since $\hat{h}(x)=1$ implies $h^*(x)=1$, and since adding clean-label poison points will only increase the region being predicted as positive,  we have $\atk(h^*, S_\trn,\clsr) = \err(\clsr(S_\trn))$. Then by Theorem~5 of~\cite{hanneke2016refined}, for any intersection-closed hypothesis class $\cH$ with VC dimension $d$, the Closure algorithm can robustly learn $\cH$ using $m$ samples, where
\[
  m= \frac{1}{\epsilon}(21d+16\ln(3/\delta))\,.
  \]

\paragraph{Unions of intervals are $(\epsilon,\delta)$-robust learnable.}
Let $\cH_k = \cup_{k'\leq k}\{\ind{\cup_{i=1}^{k'} (a_i,b_i)}|0\leq a_i< b_i\leq 1,\forall i\in [k']\}$ denote the union of at most $k$ positive open intervals for any $k\geq 1$. This hypothesis class is a generalization of Example~\ref{eg:interval}. There is a robust learning rule: output $\ind{\emptyset}$ if there is no positive sample and otherwise, output the consistent union of minimum number of closed intervals, each of which has minimum length. More specifically, given input (poisoned) data $S= \{(x_1,y_1),\ldots,(x_{m'},y_{m'})\}$ with $x_1\leq x_2\leq \ldots\leq x_{m'}$ w.l.o.g., for notation simplicity, let $y_0 =y_{m'+1}=0$. Then the algorithm $\A$ outputs $\hat{h}=\ind{X}$ where $X = \cup \{[x_i,x_j]|\forall i\leq l \leq j\in[m'], y_{i-1}=y_{j+1}=0, y_i=y_l=y_j=1\}$. The algorithm $\cA$ can robustly learn union of intervals $\cH_k$ using $m$ samples, where
\[
  m= O\!\left(\frac{1}{\epsilon}(k\log(1/\epsilon)+\log(1/\delta))\right)\,.
  \]
The detailed proof can be found in Appendix~\ref{appx:unionintvls}.
\paragraph{Hypothesis classes with finite star number are $(\epsilon,\delta)$-robust learnable.}
The star number, proposed by~\cite{hanneke2015minimax}, can measure the disagreement region of the version space.
\begin{dfn}[star number]
  The star number $\mathfrak{s}$ is the largest integer $s$ such that there exist distinct points $x_1,\ldots,x_s \in \cX$ and classifiers $h_0,\ldots,h_s$ with the property that $\forall i\in [s]$, $\DIS(\{h_0,h_i\})\cap \{x_1,\ldots,x_s\} = \{x_i\}$; if no such largest integer exists, define $\mathfrak{s}=\infty$.
\end{dfn}
By Theorem~10 of~\cite{hanneke2016refined}, for any $\cH$ with star number $\mathfrak{s}$, with probability at least $1-\delta$ over $S_\trn\sim \cD^m$, $\cP(\DIS(\VS_{\cH,S_\trn}))\leq\epsilon$ where
\[
  m= \frac{1}{\epsilon}(21\mathfrak{s}+16\ln(3/\delta))\,.
\]
As aforementioned, $\ATK(h^*, S_\trn,\ERM,\Adv)\subseteq \DIS(\VS_{\cH,S_\trn})$ for any clean-label attacker $\Adv$ and thus any ERM can robustly learn $\cH$ using $m$ samples.
\paragraph{Hypothesis classes with infinite hollow star number are not consistently properly $(\epsilon,\delta)$-robust learnable.}
The hollow star number, proposed by~\cite{bousquet2020proper}, characterizes proper learnability. For any set $S= \{(x_1,y_1),\ldots,(x_k,y_k)\}$, $S^{i} =\{(x_1,y_1'),\ldots,(x_k,y_k')\}$ is said to be a neighbor of $S$ if $y_i'\neq y_i$ and $y_j'=y_j$ for all $j\neq i$, for any $i\in[k]$.

\begin{dfn}[hollow star number]
  The hollow star number $k_o$ is the largest integer $k$ such that there is a set $S = \{(x_1,y_1),\ldots,(x_k,y_k)\}$ (called the hollow star set) which is not realizable by $\cH$, however every set $S'$ which is a neighbor of $S$ is realizable by $\cH$. If no such largest $k$ exists, define $k_o=\infty$.
\end{dfn}
For any hypothesis class $\cH$ with hollow star number $k_o$, for any consistent proper learner $\A$, there exists a target function $h^*$ and a data distribution $\cD$ such that if $m\leq \floor{(k_o-1)/2}$, then the expected attackable rate
\[
  \EEs{S_\trn\sim \cD^m}{\atk(h^*,S_\trn,\A)}\geq 1/4\,,
\]
which implies $\PPs{S_\trn}{\atk(h^*,S_\trn,\A)>1/8}\geq 1/7$ by Markov's inequality. The construction of the target function, the data distribution and the attacker is as described below. 
Consider a hollow star set $S$ as above, with size $k$. By definition, there exists a set of hypotheses $\{h_1,\ldots,h_{k}\}\subseteq \cH$ such that each neighbor $S^{i}$ is realized by $h_i$ for any $i\in[k]$. Consider the target function being $h_{i^*}$ where $i^*$ is drawn uniformly at random from $[k]$ and the marginal data distribution is a uniform distribution over $\{x_i|i\in[k]\setminus \{i^*\}\}$.  For any $\floor{(k-1)/{2}}$ i.i.d. samples from the data distribution, there are at least $k-\floor{(k-1)/{2}}$ instances in $S$ not sampled. To attack an unseen instance $x_i$, the attacker adds all examples in $S$ except $x_i, x_{i^*}$. Then any algorithm cannot tell whether $h_{i^*}$ or $h_i$ is the true target and any consistent proper learner will misclassify $\{x_i,x_{i^*}\}$ with probability $1/2$. 

For hypothesis classes with $k_o=\infty$, there is a sequence of hollow star sets with increasing sizes $\{k_i\}_{i=1}^\infty$. Therefore, any hypothesis class with $k_o=\infty$ is not consistently properly robust learnable. The detailed proof is included in Appendix~\ref{appx:hollow}.

\subsection{All robust learners are optimal PAC learners}\label{sec:opt-pac}

There is an interesting connection between algorithms robust to clean-label poisoning attacks
and the classic literature on the sample complexity of PAC learning.  Specifically, we can show that 
\emph{any} learning algorithm that is robust to clean-label poisoning attacks necessarily obtains
the optimal dependence on $\epsilon$ in its 
PAC sample complexity: that is, $O(1/\epsilon)$.
This is a very strong property, and not many such learning algorithms are known, as most learning 
algorithms have at least an extra $\log(1/\epsilon)$ factor in their sample complexity (see e.g., 
\citealp*{haussler:94,auer:07,hanneke:thesis,hanneke:16a,hanneke2016refined,darnstadt:15,bousquet2020proper}).  
Thus, this property can be very informative regarding what types of learning algorithms one should consider 
when attempting to achieve robustness to clean-label poisoning attacks.  This claim is 
formalized in the following result.  
Its proof is presented in Appendix~\ref{appx:opt-pac}.

\begin{thm}
\label{thm:opt-pac}
Fix any hypothesis class $\cH$.
Let $\A$ be a deterministic learning algorithm that always outputs a deterministic hypothesis.
Suppose there exists a non-negative 
sequence $R(m) \to 0$ such that, 
$\forall m \in \mathbb{N}$, 
for every target function $h^* \in \cH$ 
and every distribution $\cD$ over $D_{h^*}$, 
for $S_\trn \sim \cD^m$, 
with probability at least $1/2$, 
$\atk_{\cD}(h^*,S_\trn,\A) \leq R(m)$.
Then there exists an ($R$-dependent) finite constant $c_{R}$ such that,
for every $\delta \in (0,1)$, 
$m \in \mathbb{N}$, $h^* \in \cH$, and every distribution $\cD$ over $D_{h^*}$, 
for $S_\trn \sim \cD^m$, 
with probability at least $1-\delta$, 
$\atk_{\cD}(h^*,S_\trn,\A) \leq \frac{c_{R}}{m}\log\frac{2}{\delta}$.
\end{thm}

An immediate implication of this result 
(together with Markov's inequality) 
is that any deterministic $\A$ outputting 
deterministic predictors, if 
$\E_{S_\trn \sim \cD^m}[ \atk_{\cD}(h^*,S_\trn,\A) ] \leq R(m)/2 \to 0$
for all $h^* \in \cH$ and $\cD$ on $D_{h^*}$, 
then for any $h^* \in \cH$, 
$\cD$ on $D_{h^*}$, $\delta \in (0,1)$, 
$\err_{\cD}(\A(S_\trn)) \leq \frac{c_{R}}{m}\log\frac{2}{\delta}$ 
with probability at least $1-\delta$.
As mentioned, this is a strong 
requirement of the learning algorithm $\A$; 
for instance, for many classes $\cH$, many 
ERM learning rules would have an extra 
$\log(m)$ factor \citep*{hanneke2016refined}.
This also establishes a further connection 
to the hollow star number, which in some cases strengthens the result mentioned above 
(and detailed in Appendix~\ref{appx:examples}).
Specifically, \citet*{bousquet2020proper} have 
shown that when $k_o = \infty$, 
for any fixed $\delta$ sufficiently small, 
any proper learning algorithm has, 
for some infinite sequence of $m$ values, 
that $\exists h^* \in \cH$ and $\cD$ on $D_{h^*}$ for which, 
with probability greater than $\delta$, 
$\err_{\cD}(\A(S_\trn)) \geq \frac{c \log(m)}{m}$
for a numerical constant $c$.
Together with Theorem~\ref{thm:opt-pac}, 
this implies that for such classes, any 
deterministic proper learning algorithm cannot have a 
sequence $R(m) \to 0$ as in the above theorem.
Formally, using the fact that $\atk_{\cD}(h^*,S,\A)$ 
is non-increasing in $S$ (see the proof of 
Theorem~\ref{thm:opt-pac}), 
we arrive at the following corollary, 
which removes the ``consistency'' 
requirement from the result for  
classes with $k_o=\infty$ stated above, 
but adds a requirement of being deterministic.

\begin{crl}
\label{cor:proper-not-robust}
If $k_o = \infty$, then 
for any deterministic \emph{proper} learning algorithm $\A$ 
that always outputs a deterministic hypothesis,
there exists a constant $c > 0$ such that, 
for every $m \in \nats$, 
$\exists h^* \in \cH$ and distribution 
$\cD$ on $D_{h^*}$ 
such that 
$\E_{S_\trn \sim \cD^m}[ \atk_{\cD}(h^*,S_\trn,\A) ] > c$.
\end{crl}
\section{Linear hypothesis class}\label{sec:linear}
In this section, we first show that $n$-dimensional linear classifiers $\cH = \{\ind{\inner{w}{x}+b\geq 0}|(w,b)\in \R^{n+1}\}$ with $n\geq 2$ are not robust learnable. Then we study a restrictive case where the support of the data distribution has a positive margin to the boundary. We introduce two robust learners and prove a robust learning sample complexity lower bound. In addition, we also show the vulnerability of SVM. 

\subsection{Linear hypothesis class is not robust learnable}
In this section, we show that the class of linear hypotheses is not robust learnable.
\begin{thm}\label{thm:lblinear}
    For $n\geq 2$, the class of linear hypotheses is not robust learnable.
\end{thm}

\paragraph{Proof sketch}
We present the proof idea in the case of $n=3$ here and for simplicity, we allow the decision boundary to be either positive or negative. The construction details of limiting the boundary to be positive and the construction for $n=2$ are deferred to Appendix~\ref{appx:lb_linear}.

Consider the case where $\cX  = \sph^3(\bZero,1)$ is the sphere of the $3$-dimensional unit ball centered at the origin and the target function is uniformly randomly chosen from all linear classifiers with the decision boundary at distance $1/2$ from the origin, and the boundary labeled different from $\bZero$, i.e., $h^*\sim \Unif(\cH^*)$, where $\cH^* = \{\ind{\inner{w}{x}-\frac{1}{2}\geq 0}|\norm{w}=1\}\cup \{1-\ind{\inner{w}{x}-\frac{1}{2}\geq 0}|\norm{w}=1\}$. W.l.o.g., suppose $h^* = \ind{\inner{w^*}{x}-\frac{1}{2}\geq 0}$. The data distribution is the uniform distribution over the intersection of the decision boundary and the sphere, i.e., $\cD_\cX = \Unif(C_{w^*})$, where $C_{w^*} = \{x|\inner{w^*}{x}-\frac{1}{2}= 0\}\cap \sph^3(\bZero,1)$. Then all training data come from the circle $C_{w^*}$ and are labeled positive. 

Given training data $S_\trn\sim \cD^m$ and a test point $x_0 \in C_{w^*}$ (not in $S_{\trn,\cX}$), the attacker constructs a fake circle $C_{w'}$ tangent to $C_{w^*}$ at point $x_0$, i.e., $C_{w'}= \{x|\inner{w'}{x}-\frac{1}{2}= 0\}\cap \sph^3(\bZero,1)$ where $w' = 2\inner{x_0}{w^*}x_0 -w^*$. Then the attacker adds $m$ i.i.d. samples from the uniform distribution over $C_{w'}$ and labels them negative. Any algorithm cannot tell which circle is the true circle and which one of $\{\ind{\inner{w^*}{x}-\frac{1}{2}\geq 0},1-\ind{\inner{w'}{x}-\frac{1}{2}\geq 0} \}$ is the true target. Hence, any algorithm will misclassify $x_0$ with probability $1/2$.

\subsection{Linear hypothesis class is robust learnable under distribution with margin}\label{subsec:linearmrg}
In this section, we discuss linear classifiers in the case where the distribution has a positive margin. Specifically, considering the instance space $\cX\subseteq\cB^n(\bZero,1)$, we limit the data distribution $\cD$ to satisfy that $\forall (x,y)\in \supp(\cD), (2y-1)(\inner{w^*}{x}+b^*)\geq \gamma\norm{w^*}/2$ for some margin $\gamma\in(0,2]$ and target function $h^*(x) = \ind{\inner{w^*}{x}+b^*\geq 0}$.

\vspace{-3pt}
\subsubsection{A learner for arbitrary \texorpdfstring{$n>0$}{n>0}}

The learner $\cA$ fixes a $\gamma/2$-covering $V$ of $\cX$, i.e., $\forall x\in \cX,\exists v\in V, x\in \cB(v,\gamma/2)$, where $\abs{V}\leq (2/\gamma)^n$. It is easy to check that such a $V$ always exists. Then given input data $S$, the learner outputs a classifier: for $x\in\cB(v,\gamma/2)$, if $\exists(x',y')\in S$ s.t. $x'\in \cB(v,\gamma/2)$, predicting $h(x) = y'$; otherwise, predicting randomly. Note that $\Adv$ does not necessarily need to be restricted to such margin.

\begin{thm}
  The algorithm can robustly learn linear classifiers with margin $\gamma$ using $m$ samples where
  \[
    m=\frac{(2/\gamma)^n}{\epsilon}\left(n\ln \frac{2}{\gamma} +\ln \frac{1}{\delta}\right) \,.
  \]
\end{thm}
\begin{proof}
First, for every $v\in V$, at least one of $y \in \{0,1\}$ has $\cD( x \in \cB(v,\gamma/2) : h^*(x)=y ) = 0$. Then with probability at least $1-\abs{V}(1-\epsilon/\abs{V})^m$ over $S_\trn\sim \cD^m$, for every ball $\cB(v,\gamma/2)$ with probability mass at least $\epsilon/\abs{V}$, there exists $(x',y')\in S_\trn$ such that $x'\in \cB(v,\gamma/2)$. Let $m={\abs{V}\ln(\abs{V}/\delta)}/{\epsilon}$, we have with probability at least $1-\delta$, $\atk(h^*, S_\trn, \A)\leq \epsilon$.
\end{proof}
\vspace{-20pt}
\subsubsection{A learner for \texorpdfstring{$n=2$}{n=2}}
In the $2$-dimensional case, the hypothesis class can be represented as $\cH = \{h_{\beta,b}|\beta\in[0,2\pi), b\in[-2,2]\}$ where $h_{\beta,b} = \ind{(\cos \beta,\sin \beta)\cdot x + b \geq 0}$. When there is no ambiguity, we use $(\beta,b)$ to represent $h_{\beta,b}$. The target is $h^* =h_{\beta^*,b^*}$. Then we propose a robust algorithm based on binary-search for the target direction $\beta^*$ as shown in Algorithm~\ref{alg:linear2d}.
\begin{algorithm}[t]\caption{Robust algorithm for $2$-dimensional linear classifiers}\label{alg:linear2d}
  \begin{algorithmic}[1]
    \STATE \textbf{input}: data $S$
    \STATE \textbf{initialize} $l\la 0$, $h\la 2\pi$ and $\beta\la \frac{l+h}{2}$
    \STATE \textbf{if} $\exists b\in [-2,2]$ s.t. $(0,b)$ is consistent \textbf{then} output $(0,b)$
    \WHILE{$\nexists b\in [-2,2]$ s.t. $(\beta,b)$ is consistent with $S$}
    \STATE \textbf{if} $\exists \beta\in (l,\frac{l+h}{2})$ s.t. $\exists b\in [-2,2]$, $(\beta,b)$ is consistent with $S$ \textbf{then} let $h\la \frac{l+h}{2}$, $\beta \la \frac{l+h}{2}$
    \STATE \textbf{else} let $l\la \frac{l+h}{2}$, $\beta \la \frac{l+h}{2}$
    \ENDWHILE
    \STATE \textbf{return} $(\beta,b)$ with any consistent $b$
  \end{algorithmic}
\end{algorithm}

\begin{thm}\label{thm:lin2d}
  For any data distribution $\cD$, let $f(\epsilon'')=\max \{s\geq 0|\cP(\{x|(\cos \beta^*,\sin \beta^*)\cdot x + b^*\in [-s,0]\})\leq \epsilon'', \cP(\{x|(\cos \beta^*,\sin \beta^*)\cdot x + b^*\in [0,s]\})\leq \epsilon''\}$ for $\epsilon''\in[0,1]$ denote the maximum distance between the boundary and two parallel lines (on positive side and negative side respectively) such that the probability between the boundary and either line is no greater than $\epsilon''$. With probability at least $1-\delta$ over $S_\trn\sim \cD^m$, Algorithm~\ref{alg:linear2d} achieves
  \[
    \atk(h^*, S_\trn, \cA)  \leq \log(\frac{32}{f(\epsilon'')\wedge 2})2\epsilon' + 2\epsilon'' \,,
  \]
  for any $\epsilon''\in[0,1]$ using $m$ samples where
  \[
  m= \frac{24}{\epsilon'} \log \frac{13}{\epsilon'}+\frac{4}{\epsilon'} \log \frac{2}{\delta}\,.
  \]
\end{thm}
\paragraph{Proof sketch} First, by uniform convergence bound in PAC learning~\citep{blumer1989learnability}, when $m \geq \frac{24}{\epsilon'} \log \frac{13}{\epsilon'}+\frac{4}{\epsilon'} \log \frac{2}{\delta}$, every linear classifier consistent with $S_\trn$ has error no greater than $\epsilon'$. For any fixed $\beta$, the probability mass of union of error region of all $(\beta,b)$ consistent with the training data is bounded by $2\epsilon'$. Then given a target $(\beta^*,b^*)$, the binary-search path of $\beta$ is unique and adding clean-label poison examples will only change the depth of search.  When $h-l< \arctan(f(\epsilon'')/2)$, the attackable rate caused by deeper search is at most $2\epsilon''$. Combining these results together proves the theorem. The formal proof of Theorem~\ref{thm:lin2d} is included in Appendix~\ref{appx:lin2d}.

\begin{thm}\label{thm:lin2dmargin}
  For any $\gamma \in(0,2]$, Algorithm~\ref{alg:linear2d} can $(\epsilon,\delta)$-robustly learn $2$-dimensional linear classifiers with margin $\gamma$ using $m$ samples where
  \[
    m=\frac{48\log(64/\gamma)}{\epsilon}  \log \frac{26\log(64/\gamma)}{\epsilon}+\frac{8\log(64/\gamma)}{\epsilon} \log\frac 2\delta\,.\]
\end{thm}
Theorem~\ref{thm:lin2dmargin} is the immediate result of Theorem~\ref{thm:lin2d} as $f(0)=\gamma/2$. 

\subsubsection{SVM requires \texorpdfstring{${e^{\Omega(n)}}/{\epsilon}$}{Omega(exp(n)/epsilon)} samples against one-point attacker}
SVM is a well-known optimal PAC learner for linear hypothesis class~\citep{bousquet2020proper}. In this section, we show that even in the case where $\gamma\geq 1/8$ and the attacker is limited to add at most one poison point, SVM requires $e^{\Omega(n)}/\epsilon$ samples to achieve $\epsilon$ attackable rate.
\begin{thm}\label{thm:linSVM}
    For $n$-dim linear hypothesis class, for any $\epsilon< 1/16$, there exists a target $h^*\in\cH$ and a distribution $\cD$ over $D_{h^*}$ with margin $\gamma = 1/8$ such that $\EEs{S_\trn\sim \cD^m}{\atk_\cD(h^*, S_\trn,\SVM)}> \epsilon$ when the sample size $m< \frac{e^{n/128}}{768\epsilon}\vee \frac{1}{8\epsilon}$.
\end{thm}
\paragraph{Proof sketch}
Consider the case where $\cX = \{x\in\R^3|\norm{x}=1, \inner{x}{e_1}\geq0\}\cup \{-e_1\}$ is the union of a half sphere and a point $-e_1$. The target function is $h^* = \ind{\inner{w^*}{x}\geq -\gamma/2}$ with $w^* = e_1$ and margin $\gamma= 1/8$. Note that $h^*$ labels all points on the half sphere positive and $-e_1$ negative. Then we define the data distribution $\cD_\cX$ by putting probability mass $1-8\epsilon$ on $-e_1$ and putting probability mass $8\epsilon$ uniformly on the half sphere.

Then we draw training set $S_\trn\sim \cD^m$ and a test point $(x_0,y_0)\sim \cD$. Condition on that $x_0$ is on the half sphere, with high probability, $\inner{x_0}{w^*}\leq 1/8$. Then we define two base vectors $v_1 = w^*$ and $v_2 = \frac{x_0 - \inner{x_0}{w_*}w^*}{\norm{x_0 - \inner{x_0}{w_*}w^*}}$ in the $2$-dimensional space defined by $w^*$ and $x_0$. With high probability over the choice of $S_\trn$, for all positive training examples $x$ on the half sphere, we have $\inner{x}{v_1}\leq 1/8$ and $\inner{x}{v_2}\leq 1/8$. Then the attacker injects a poison point at $-\gamma v_1+\sqrt{1-\gamma^2}v_2$, which is closer to $x_0$ than all the positive samples in $S_\trn$. Since the poison point is classified as negative by the target function, SVM will misclassify $x_0$ as negative. The detailed proof can be found in Appendix~\ref{appx:linSVM}.

\subsubsection{Lower bound}
Here we show that robust learning of linear hypothesis class under distribution with margin $\gamma>0$ requires sample complexity 
${e^{\Omega(n)}}/{\epsilon}$.
\begin{thm}\label{thm:lblinmrg}
For $n$-dimensional linear hypothesis class with $n> 256$, for any $\epsilon\leq 1/16$ and for any algorithm $\cA$, there exists a target function $h^*\in\cH$ and a distribution $\cD$ over $D_{h^*}$ with margin $\gamma = 1/8$ such that $\EEs{S_\trn\sim \cD^m}{\atk_\cD(h^*, S_\trn,\A)}> \epsilon$ when the sample size $m\leq \frac{e^\frac{n-1}{128}}{192\epsilon}$. For convenience, here we relax the instance space by allowing $\cX\subseteq B^n(\bZero, 9/8)$. 
\end{thm}
The construction of the target function and the data distribution is similar to that in the proof of Theorem~\ref{thm:linSVM}. To attack a test instance $x_0$, the attacker adds the reflection points of all training points through the hyperplane defined by $x_0$ and $w^*$ such that any algorithm will misclassify $x_0$ with probability $1/2$. The detailed proof is included in Appendix~\ref{appx:lb_linmrg}.

\section{Results for finite-point attackers}\label{sec:finite}
In this section, instead of considering the case where the attacker can add a set of poison examples of arbitrary size, we study a restrictive case where the attacker is allowed to add at most $t$ poison examples for some $t<\infty$, i.e., $\abs{\Adv(h^*,S_\trn,x_0)}\leq t$ for any $h^*,S_\trn,x_0$. Following Definition~\ref{dfn:rate} and \ref{dfn:learnablity}, we define $t$-point clean-label attackable rate and $(t,\epsilon,\delta)$-robust learnability as follows. 
\begin{dfn}[$t$-point clean-label attackable rate]
For a target function $h^*$, a training data set $S_\trn$ and a (possibly randomized) algorithm $\A$, for any distribution $\cD$ over $D_{h^*}$, the $t$-point clean-label attackable rate is
\begin{align*}
    \atk_\cD(t, h^*, S_\trn,\A) \triangleq \sup_\Adv \atk_\cD(h^*, S_\trn,\A,\Adv)\text{ s.t. } \abs{\Adv(h^*,S_\trn,x)}\leq t,\forall x\in\cX\,.
\end{align*}
\end{dfn}
\begin{dfn}[$(t,\epsilon,\delta)$-robust learnability]
    A hypothesis class $\cH$ is $(t,\epsilon,\delta)$-robust learnable if there exists a learning algorithm $\A$ such that $\forall \epsilon,\delta\in (0,1)$, $\exists m(t,\epsilon,\delta)\in \NN$ such that $\forall h^*\in\cH, \forall \cD$ over $D_{h^*}$, with probability at least $1-\delta$ over $S_\trn\sim \cD^m$,
    \[\atk_\cD(t,h^*, S_\trn,\A)\leq \epsilon\,.\]

\end{dfn}
\subsection{\texorpdfstring{Algorithms robust to $t$}{t}-point attacker}
Robustness to a small number of poison examples has been studied by~\cite{ma2019data,levine2020deep}. \cite{ma2019data} show that differentially-private learners are naturally resistant to data poisoning when the attacker can only inject a small number of poison examples. \cite{levine2020deep} propose an algorithm called Deep Partition Aggregation (DPA), which partitions the training set into multiple sets by a deterministic hash function, trains base classifiers over each partition and then returns the majority vote of base classifiers. They show that for any instance $x$, the prediction on $x$ is unchanged if the number of votes of the output exceeds half of the number of the total votes by $t$. But the attackable rate of DPA is not guaranteed. Here we propose several algorithms similar to DPA but with guarantees on the attackable rate. In Algorithm~\ref{alg:impropert}, we provide a protocol converting any given ERM learner $\cL$ to a learner with small $t$-point clean-label attackable rate. 
\begin{algorithm}[t] \caption{A robust protocol for $t$-point attacker}\label{alg:impropert}
    {\begin{algorithmic}[1]
        \STATE \textbf{input}: A proper ERM learner $\cL$, data $S$
        \STATE divide $S$ into $10t+1$ blocks $\{S^{(1)},S^{(2)},\dots,S^{(10t+1)}\}$ with size $\floor{\frac{\abs{S}}{10t+1}}$ randomly without replacement (throw away the remaining $\abs{S}-(10t+1) \floor{\frac{\abs{S}}{10t+1}}$ points)
        \STATE \textbf{return} $\Major(\cH')$ where $\cH'=\{\cL(S^{(i)})|i\in[10t+1]\}$
    \end{algorithmic}}
\end{algorithm}

\begin{thm}\label{thm:timproper}
    For any hypothesis class $\cH$ with VC dimension $d$ with any proper ERM learner $\cL$, Algorithm~\ref{alg:impropert} can $(t,\epsilon,\delta)$-robustly learn $\cH$ using $m$ samples where
    \[
        m=O\left(\frac{dt}{\epsilon}\log\frac{dt}{\epsilon}+ \frac{d}{\epsilon} \log \frac{1}{\delta}  \right)\,.
    \]
\end{thm}
\paragraph{Proof sketch} For every misclassified point $x_0\in\cX$, there are at least $5t+1$ classifiers among $\{\cL(S^{(i)})\}_{i=1}^{10t+1}$ misclassifying $x_0$. Since there are at most $t$ blocks containing poison data, there are at least $4t+1$ non-contaminated classifiers (output by blocks without poison data) misclassifying $x_0$. Then $t$-point clean-label attackable rate is bounded by bounding the error of one non-contaminated classifier. The detailed proof is provided in Appendix~\ref{appx:improper_finite}. 

{\vskip 2mm}As we can see, Algorithm~\ref{alg:impropert} is improper even if $\cL$ is proper. Inspired by the projection number and the projection operator defined by~\cite{bousquet2020proper}, we propose a proper robust learner in Algorithm~\ref{alg:propertproj}. First, let us introduce the definitions of the projection number and the projection operator as follows. For a finite (multiset) $\cH'\subseteq \cH$, for $l\geq 2$, define the set $\cX_{\cH',l}\subseteq \cX$ of all the points $x$ on which less than $\frac{1}{l}$-fraction of all classifiers in $\cH'$ disagree with the majority. That is,
\begin{align*}
    \cX_{\cH',l} = \left\{x\in \cX: \sum_{h\in\cH'} \ind{h(x)\neq \Major(\cH',x)} < \frac{|\cH'|}{l}  \right\}.
\end{align*}
\vspace{-5pt}
\begin{dfn}[projection number and projection operator]
The projection number of $\cH$, denoted by $k_p$, is the smallest integer $k\geq 2$ such that, for any finite multiset $\cH'\subseteq\cH$ there exists $h\in \cH$ that agrees with $\Major(\cH')$ on the entire set $\cX_{\cH',k}$. If no such integer $k$ exists, define $k_p=\infty$. If $k_p<\infty$, the projection operator $\Proj_\cH: \cH'\mapsto \cH$ is a deterministic map from $\cH'$ to $\cH$ such that $\Proj_\cH(\cH',x)=\Major(\cH',x),\forall x\in \cX_{\cH',k_p}$.
\end{dfn}
\begin{algorithm}[t]
\caption{A proper robust learner for $t$-point attacker given projection number $k_p$}\label{alg:propertproj}
{\begin{algorithmic}[1]
    \STATE \textbf{input}: A proper $\ERM$ learner $\cL$, data $S$
    \STATE Divide the data $S$ into $10k_pt + 1$ sets $\{S^{(1)},S^{(2)},\dots,S^{(10k_pt+1)}\}$ with size $\floor{\frac{\abs{S}}{10 k_p t+1}}$ randomly without replacement (throw away the remaining $\abs{S}-(10k_pt+1) \floor{\frac{\abs{S}}{10k_pt+1}}$ points)
    \STATE \textbf{return} $\hat{h} = \Proj_\cH(\cH')$, where $\cH'=\{ h_i = \cL(S^{(i)})|i\in [10k_pt+1] \}$
\end{algorithmic}}
\end{algorithm}

\begin{thm}\label{thm:tpropproj}
For any hypothesis class $\cH$ with VC dimension $d$ and projection number $k_p$, with any proper $\ERM$ learner $\cL$, Algorithm~\ref{alg:propertproj} can $(t,\epsilon,\delta)$-robustly learn $\cH$ using $m$ samples where
\[
m=O\left(\frac{k_p^2dt}{\epsilon}\log\frac{k_pdt}{\epsilon} +  \frac{k_pd}{\epsilon}\ln \frac{1}{\delta} \right)\,.
\]
\end{thm}
The proof adopts the same idea as the proof of Theorem~\ref{thm:timproper} and is included in Appendix~\ref{appx:proper_finite}. For the hypothesis class with infinite projection number, we can obtain a proper learner in a similar way: randomly selecting $\floor{{\epsilon\abs{S}}/{3t}}$ samples with replacement from input data set $S$ and run ERM over the selected data. We show that this algorithm can $(t,\epsilon,\delta)$-robustly learn $\cH$ using $O(\frac{dt}{\epsilon^2} \log \frac{d}{\epsilon}+\frac{d}{\epsilon}\log \frac{1}{\delta})$ samples. The details of the algorithm and the analysis can be found in Appendix~\ref{appx:proper_finite}.

\subsection{Lower bound}
\begin{thm}\label{thm:lbfnt}
    For any $d\geq 1$ and $\epsilon\leq \frac{3}{8}$, there exists a hypothesis class $\cH$ with VC dimension $5d$ such that for any algorithm $\A$, there exists a target function $h^*\in \cH$ and a data distribution $\cD$ on $D_{h^*}$, such that $\EEs{S_\trn\sim \cD^m}{\atk_\cD(t,h^*,S_\trn,\A)}> \epsilon$ when the sample size $m< \frac{3td}{64\epsilon}$.
\end{thm}
\paragraph{Proof sketch}
Consider $d$ disjoint spheres in $\mathbb{R}^3$ and the target function is chosen by randomly selecting a circle on each sphere. Then label each circle differently from the rest of the sphere the circle lies on. Specifically, we flip $d$ independent fair coins, one for each circle to decide whether the circle is labeled positive or negative. The data distribution puts probability mass $\frac{t}{8m}$ uniformly on each circle and $1-\frac{td}{8m}$ probability mass on an irrelevant point (not on any of the spheres). Then we can show that with constant probability, every unseen point on each circle can be attacked by an attacker similar to the one in the proof sketch of Theorem~\ref{thm:lblinear}. The detailed proof is included in Appendix~\ref{appx:lb_finite}.
\begin{rmk}
Actually, our algorithms above even work for $t$-point \emph{unclean}-label attackers (where the poison data are not necessarily labeled by the target function) as well, which indicates that cleanness of poison examples does not make the problem fundamentally easier in the worst-case over classes 
of a given VC dimension, in the $t$-point 
attack case (although it can potentially 
make a difference for particular algorithms 
or particular classes $\cH$).
\end{rmk}



\section{Discussion and future directions}
In this paper, we show the impossibility of robust learning in the presence of clean-label attacks for some hypothesis classes with bounded VC dimension, e.g., the class of linear separators, and the robust learnability of some hypothesis classes characterized by known complexity measures, e.g., finite star number. There are several interesting open questions.
\vspace{-5pt}
\begin{itemize}
    \setlength{\itemsep}{1pt}
    \setlength{\parsep}{0pt}
    \setlength{\parskip}{0pt}
    \item The first question is what are necessary and sufficient conditions for $(\epsilon,\delta)$-robust learnability. Finite star number is a sufficient but not necessary condition. Here is an example where the instance space is $\cX=\NN$ and the hypothesis class is $\cH=\{\ind{i}|i\in \NN\}\cup \{0\}$. The star number of $\cH$ is $\mathfrak{s}=\infty$, but $\cH$ is robust learnable since $\vcd(\cH)=1$. One intriguing possible complexity measure is the largest number $k$ such that there is a set of distinct points $S =\{x_1,\ldots,x_k\}\in \cX^k$ and classifiers $\{h_0,\ldots,h_k\}$, where for any $i\in[k]$, there exists an involutory function $f_i:\cX\mapsto \cX$ (i.e., $f_i(f_i(x)) = x ,\forall x\in\cX$) such that $f_i(x_i) = x_i$ and $\DIS(\{h_0,h_i\})\cap \{S\cup f_i(S)\} = \{x_i\}$.
    
    \item For proper robust learning, we prove that any hypothesis class with infinite hollow star number $k_o=\infty$ is neither consistently properly $(\epsilon,\delta)$-robust learnable nor deterministically properly $(\epsilon,\delta)$-robust learnable. Compared with the fact that $k_o=\infty$ only brings an extra $\Omega(\log({1}/{\epsilon}))$ in the optimal PAC sample complexity~\citep{bousquet2020proper}, we see that the hollow star number has a dramatically larger impact on proper robust learnability. On the other hand, finite hollow star number does not suffice for robust learnability (e.g., \citealp{bousquet2020proper}, show linear classifiers on $\mathbb{R}^n$ have $k_o = n+2$), and it is unclear what is the necessary and sufficient condition for proper robust learnability.
    
    \item For linear classifiers with margin $\gamma>0$, the lower bound of the sample complexity presented in Section~\ref{sec:linear} ignores the dependence on $\gamma$. For the two learners introduced in Section~\ref{sec:linear}, the one using the covering set has sample complexity of $O(n(2/\gamma)^n\log(1/\gamma))$ and the other one designed for the $2$-dimension has sample complexity of $O(\log(1/\gamma)\log\log(1/\gamma))$. There is a huge gap between the lower bound and the upper bound and thus far, the optimal dependence on $\gamma$ remains unclear.
    
    \item For finite-point attacks, we construct a hypothesis class such that the $t$-point clean-label attackable rate is $\Omega(\frac{t}{m})$ in the proof of Theorem~\ref{thm:lbfnt} and Algorithm~\ref{alg:impropert} achieves $O(\frac{t\log(m)}{m})$ attackable rate. It is unclear to us for what kind of hypothesis class, there is an algorithm able to achieve $o(\frac{t}{m})$ attackable rate. At the same time, we are curious about its connection to $(\epsilon,\delta)$-robust learnablility. Notice that in all the proofs of the negative results in this paper, the attacker we construct never injects more than $m$ poison examples. This triggers the following suspicion: are infinite-point attackers strictly more powerful than $m$-point attackers? Specifically, we have the following conjecture.
    \begin{conj}[infinite to finite]
    For any hypothesis class $\cH$, for every target function $h^*\in\cH$, data distribution $\cD$ over $D_{h^*}$, there exist a pair of constants $c,c'>0$ such that for any $m>0$, any training data $S_\trn\in D_{h^*}^m$ and any algorithm $\cA$, $\atk_\cD(h^*, S_\trn,\A)\geq c$ iff  $\atk_\cD(m,h^*, S_\trn,\A)\geq c'$.
    \end{conj}
    Assuming that this conjecture holds, hence for any hypothesis class $\cH$, if there exists an algorithm $\A$ able to $(t,\epsilon,\delta)$-robustly learn $\cH$ with attackable rate $o(\frac{t}{m})$, then $\cH$ is $(m,\epsilon,\delta)$-robust learnable and thus, $(\epsilon,\delta)$-robust learnable.
     
    \item Another open question is whether abstention helps. Considering the case where the algorithm is allowed to abstain on $\epsilon$-fraction of inputs if the algorithm detects abnormality. That is to say, the algorithm outputs a selective classifier $(\hat h, \CR(S_\trn))$, where the prediction hypothesis $\hat h$ is a map from $\cX$ to $\cY$ and $\CR(S_\trn)\subseteq \cX$ is the confidence region of the prediction. The algorithm predicts $\A(S_\trn,x) = \hat h(x)$ if $x\in \CR(S_\trn)$ and $\A(S_\trn,x) = \perp$ if $x\notin \CR(S_\trn)$, where $\perp$ means the algorithm predicts ``I don't know''. Then for any deterministic algorithm, we say a test instance $x\in \cX$ is attackable if there is a clean-label attacker such that $x$ is predicted incorrectly as well as $x$ is in the confidence region, i.e.,
    \[\A(S_\trn\cup \Adv(h^*,S_\trn,x),x)\neq h^*(x) \qquad\&\qquad x\in \CR(S_\trn\cup \Adv(h^*,S_\trn,x))\,.\]
    We define the event $\cE(h^*,S_\trn,\cA,\Adv,x,y) = \{\A(S_\trn\cup \Adv(h^*,S_\trn,x),x)\neq h^*(x) \cap x\in \CR(S_\trn\cup \Adv(h^*,S_\trn,x))\}$ and then define the selective attackable rate as
\[\sup\nolimits_\Adv \EEs{(x,y)\sim \cD,\A}{\ind{\cE(h^*,S_\trn,\cA,\Adv,x,y)}}.\]
We are curious about the sample complexity required to achieve $\epsilon$ selective attackable rate while keeping the probability mass of the confidence region $\PPs{(x,y)\sim \cD}{x\in \CR(S)}\leq \epsilon$ for any input $S\supseteq S_\trn$. 
\end{itemize}

\acks{This work was supported in part by the National Science Foundation under grant CCF-1815011 and by the Defense Advanced Research Projects Agency under cooperative agreement HR00112020003. Jian Qian acknowledges support of the ONR through grant
\# N00014-20-1-2336. The views expressed in this work do not necessarily reflect the position or the policy of the Government and no official endorsement should be inferred. Approved for public release; distribution is unlimited.}
\bibliography{Major}

\newpage
\appendix
\section{Proof of results in Section~\ref{sec:examples}}\label{appx:examples}
\subsection{Hypothesis class with VC dimension $d=1$}\label{appx:vcd1}
For any $f\in \cH$, let $\max^{\leq^\cH_f}_{x\in S_0}x$ denote the maximal element w.r.t. the partial ordering $\leq^\cH_f$ in any non-empty ordered finite set $S_0$, i.e., $\forall x'\in S_0, x'\leq^\cH_f \left(\max^{\leq^\cH_f}_{x\in S_0}x\right)$. Then for any arbitratrily chosen but fixed $f\in \cH$, the algorithm 
(originally proposed by \citealp{ben20152}) is described as follows.
\begin{algorithm}[H]\caption{Robust algorithm for $\cH$ with VC dimension $d=1$}\label{alg:vc1}
  \begin{algorithmic}[1]
    \STATE \textbf{input}: data $S$
    \STATE If every $(x,y) \in S$ has $y=f(x)$, \textbf{return} $\hat{h}=f$
    \STATE Let $x_m = \max^{\leq^\cH_f}_{(x,y)\in S, y\neq f(x)} x$
    \STATE $\hat{h}(x)=1-f(x)$ for $x\leq^\cH_f x_m$ and $\hat{h}(x)=f(x)$ otherwise
    \STATE \textbf{return} $\hat{h}$
  \end{algorithmic}
\end{algorithm}
By Lemma~5 of~\cite{ben20152}, $\leq^\cH_f$ for $d=1$ is a tree ordering. Thus, all points labeled differently by $f$ and $h^*$ should lie on one path, i.e., for every $x,x'\in\cX$, if $h^*(x)\neq f(x)$ and $h^*(x')\neq f(x')$, then $x\po x'$ or $x' \po x$. Due to this structure property of hypothesis class with VC dimenison $1$, adding clean-label attacking points can only narrow down the error region of Algorithm~\ref{alg:vc1}.
\begin{thm}
  For any $\cH$ with VC dimension $d=1$, Algorithm~\ref{alg:vc1} can $(\epsilon,\delta)$-robustly learn $\cH$ using $m$ samples, where 
  \[
    m = \left\lceil \frac{2\ln(1/\delta)}{\epsilon} \right\rceil\,.
    \]
\end{thm}

\begin{proof}
  First, we prove that $X=\{x\in \cX|h^*(x)\neq f(x)\}$ is totally ordered by $\leq^\cH_f$. That is, for every $x,x'\in\cX$, if $h^*(x)\neq f(x)$ and $h^*(x')\neq f(x')$, then $x\po x'$ or $x' \po x$. If it is not true, then there exists $h_1,h_2\in\cH$ such that $h_1(x)\neq f(x), h_1(x')= f(x')$ and $h_2(x')\neq f(x'), h_2(x)= f(x)$. Then, $\{f,h^*,h_1,h_2\}$ shatters $\{x,x'\}$, which contradicts that $d=1$. Therefore the finite set $\{x|(x,y)\in S, h^*(x)\neq f(x)\}\subseteq \{x\in \cX|h^*(x)\neq f(x)\}$ is also an ordered set. If the set is not empty, $x_m$ in the algorithm is well-defined.
  
  Now note that either $\hat{h}=f$ 
  or else $x_m$ is defined and then every 
  $x$ with $\hat{h}(x) \neq f(x)$ has 
  $x \leq^\cH_f x_m$, which implies 
  $h^*(x) \neq f(x)$ as well (since $h^*(x_m) \neq f(x_m)$).
  In particular, if every $(x,y) \in S_\trn$ 
  has $y=f(x)$ then 
  the attackable region $\ATK(h^*, S_\trn,\A,\Adv) \subseteq X$.
  Otherwise let 
  $x_{S_\trn} = \max^{\leq^\cH_f}_{(x,y)\in S_\trn, y\neq f(x)} x$, 
  the maximal element in $\{x|(x,y)\in S_\trn, h^*(x)\neq f(x)\}$, 
  we would have that 
  $\ATK(h^*, S_\trn,\A,\Adv) \subseteq \{x|x_{S_\trn}\po x, h^*(x)\neq f(x)\}$.
  
  In particular, if 
  $\cP_\cD(X)\leq \epsilon$ 
  the above facts imply 
  $\atk_\cD(h^*, S_\trn,\A) \leq \epsilon$.
  Otherwise if $\cP_\cD(X) > \epsilon$, 
  then we let $X_\epsilon \subseteq X$ be any  minimal set such that $\cP_\cD(X_\epsilon)\geq\frac{\epsilon}{2}$ and for every $x'\in X\setminus X_\epsilon$ and every $x\in X_\epsilon$, $x' \po x$. If $\cP_\cD(X_\epsilon)\geq \epsilon$, there exists an element $x\in X_\epsilon$ with probability mass at least $\frac{\epsilon}{2}$. When $m \geq \frac{2\ln(1/\delta)}{\epsilon}$, with probability at least $1-\delta$, $x$ is in $S_\trn$ and therefore 
  $\ATK(h^*, S_\trn,\A,\Adv) \subseteq X_\epsilon \setminus \{x\}$, so that 
  $\atk_\cD(h^*, S_\trn,\A) \leq \frac{\epsilon}{2}$. Otherwise, if $\frac{\epsilon}{2} \leq \cP_\cD(X_\epsilon)< \epsilon$, then as long as $S_\trn$ contains at least one example from $X_\epsilon$, then $\ATK(h^*, S_\trn,\A,\Adv) \subseteq X_\epsilon$, so that $\atk_\cD(h^*, S_\trn,\A) \leq \cP_\cD(X_\epsilon)<\epsilon$. Since $S_\trn$ contains an example from $X_\epsilon$ with probability at least  $1 - (1-\frac{\epsilon}{2})^m$, when $m \geq \frac{2\ln(1/\delta)}{\epsilon}$ we have that with probability at least $1-\delta$, $\atk_\cD(h^*, S_\trn,\A) \leq \epsilon$.
\end{proof}


\vspace{-20pt}
\subsection{Union of intervals}\label{appx:unionintvls}
\begin{thm}
  The algorithm described in Section~\ref{sec:examples} can $(\epsilon,\delta)$-robustly learn union of intervals $\cH_k$ using $m$ samples, where
  \[
    m= O\!\left(\frac{1}{\epsilon}(k\log(1/\epsilon)+\log(1/\delta))\right)\,.
    \]
\end{thm}
\begin{proof}
  We denote the target function by $h^* = \ind{\cup_{i=1}^{k^*}(a_{2i-1},a_{2i})}$ with $0=a_0\leq a_1\leq \ldots\leq a_{2k^*+1}=1$ and $a_{2i-1}\neq a_{2i},\forall i\in [k^*]$ for some $0 \leq k^*\leq k$. In the following, we will construct two classifiers consistent with the training set and then prove that the attackable rate of our algorithm is upper bounded by the sum of the error rates of these two classifiers. For any $i\in[k^*]$, we define $c_i^+$ as the minimum consistent positive interval within $(a_{2i-1},a_{2i})$, i.e.,
  \begin{align*}
      c_i^+=
      \left\{
      \begin{array}{cl}
       \left[\min\limits_{x\in (a_{2i-1},a_{2i}):(x,y)\in S_\trn}x,\max\limits_{x\in (a_{2i-1},a_{2i}):(x,y)\in S_\trn}x\right] 
       & \quad\text{if }S_\trn\cap (a_{2i-1},a_{2i})\times \cY\neq \emptyset\,,\\
     \emptyset      & \quad\text{otherwise}\,.
      \end{array}
      \right.
  \end{align*}
  Similarly, for $i=0,\ldots,k^*$, we define $c_i^-$ as the minimum consistent negative interval within $[a_{2i},a_{2i+1}]$. Since $x=0$ and $x=1$ are labeled as $0$ by every hypothesis in $\cH_k$, we would like $c_0^-$ to include $x=0$ and $c_{k^*}-$ to include $x=1$. Then we denote by $\bar {S_\trn} = S_\trn\cup \{(0,0),(1,0)\}$ and then define $c_i^-$ as
  \begin{align*}
      c_i^-=
      \left\{
      \begin{array}{cl}
       \left[\min\limits_{x\in [a_{2i},a_{2i+1}]:(x,y)\in \bar {S_\trn}}x,\max\limits_{x\in [a_{2i},a_{2i+1}]:(x,y)\in \bar {S_\trn}}x\right] 
       & \quad\text{if }(\bar {S_\trn})\cap [a_{2i},a_{2i+1}]\times \cY\neq \emptyset\,,\\
     \emptyset      & \quad\text{otherwise}\,.
      \end{array}
      \right.
  \end{align*}  
  Let us define two classifiers: $h_c^+ = \ind{\cup_{i=1}^{k^*}c_i^+}$ and $h_c^- = 1- \ind{\cup_{i=0}^{k^*}c_i^-}$. Then we extend $\cH_k$ to $\bar{\cH}_k = \cup_{k'\leq k}\{\ind{\cup_{i=1}^{k'} [a_i,b_i]}|0\leq a_i< b_i\leq 1,\forall i\in [k']\}\cup \cH_k$ by including union of closed intervals. Since both $h_c^+,h_c^-\in \bar{\cH}_k$ are consistent with $S_\trn$ and the VC dimension of $\bar{\cH}_k$ is $2k$, by classic uniform convergence results \citep{vapnik:74,blumer1989learnability}, for any data distribution $\cD$, with probability at least $1-\delta$ over $S_\trn\sim \cD^m$, $\err(h_c^+)\leq \frac{\epsilon}{2}$ and $\err(h_c^-)\leq \frac{\epsilon}{2}$ where $m=O(\frac{1}{\epsilon}(2k\log(1/\epsilon)+\log(1/\delta)))$.


  It is easy to see that the algorithm (even under attack) will always predicts $1$ over $c_i^+$ as the attacker cannot add negative instances into $c_i^+$. Then for any attacker $\Adv$ and any $i\in[k^*]$, for any $x\in \ATK(h^*, S_\trn,\A,\Adv)\cap (a_{2i-1},a_{2i})$, we will have $x\notin c_i^+$, which is classified $0$ by $h_c^+$. Therefore, $\ATK(h^*, S_\trn,\A,\Adv)\cap \{x|h^*(x)=1\}\subseteq \{x|h_c^+(x)=0,h^*(x)=1\}$. We can prove a similar result for $\ATK(h^*, S_\trn,\A,\Adv)\cap \{x|h^*(x)=0\}$. If the algorithm (under attack) predicts $1$ on any $x\in[a_{2i},a_{2i+1}]$ for $i=1,\ldots,k^*-1$, then $c_i^-=\emptyset$ and $h_c^-(x)=1\neq h^*(x)$. Note that the algorithm always correctly labels points in $[0,a_1]$ and $[a_{2k^*},1]$ as there are no positively-labeled points in these two intervals. Therefore, $\ATK(h^*, S_\trn,\A,\Adv)\cap \{x|h^*(x)=0\}\subseteq \{x|h_c^-(x)=1,h^*(x)=0\}$. Then for any point in the attackable region, it is either in the error region of $h_c^+$ or in the error region of $h_c^-$. That is, $\ATK(h^*, S_\trn,\A,\Adv)\subseteq \{x|h_c^+(x)\neq h^*(x)\}\cup \{x|h_c^-(x)\neq h^*(x)\}$. Hence, the attackable rate $\atk(h^*, S_\trn,\A)\leq \err(h_c^+)+\err(h_c^-)\leq \epsilon$.
\end{proof}


\vspace{-20pt}
\subsection{Hypothesis class with infinite hollow star number}\label{appx:hollow}
\begin{thm}\label{thm:hollow}
  For any hypothesis class $\cH$ with hollow star number $k_o\geq 3$, for any consistent proper learner $\A$, there exists a target function $h^*$ and a data distribution $\cD$ such that if $m\leq \floor{(k_o-1)/2}$, then $\EEs{S_\trn\sim \cD^m}{\atk(h^*,S_\trn,\A)}\geq 1/4$ and $\PPs{S_\trn\sim \cD^m}{\atk(h^*,S_\trn,\A)>1/8}\geq 1/7$.
\end{thm}
\begin{proof}
  For $\cH$ with hollow star number $k_o$, there is an unrealizable set $S=\{(x_1,y_1),\ldots, (x_{k_o},y_{k_o})\}$, such that for each $i\in \{1,\ldots,k_o\}$, there exists $h_i\in \cH$ with $\{j\in[k_o]|h_i(x_j)\neq y_j\} = \{i\}$. The target function $h^*=h_{i^*}$. For notation simplicity, we let $i$ represent point $x_i$ and then $h(i)=h(x_i)$ for any classifier $h$ and any $i\in [k_o]$. The marginal data distribution $\cD_\cX$ is a uniform distribution over $[k_o]\setminus\{i^*\}$. We define an attacker which injects all points except $i,i^*$ to attack point $i$ if it is unseen. Specifically, let $S_{i^*,i}=\{(j,h_{i^*}(j))|j\in[k_o]\setminus\{i^*,i\}\}$ denotes the set of all points excluding $i,i^*$ and we define the attacker as
  \begin{align*}
    \Adv(h_{i^*},S_\trn,i)=\begin{cases}
      S_{i^*,i}& \text{if } (i,y_i)\notin S_\trn\,,\\
      \emptyset & \text{else}\,. 
    \end{cases}
  \end{align*}
  Given the input data set $S_\trn\cup \Adv(h_{i^*},S_\trn,i)$, if $(i,y_i)$ is not sampled in the training set, then any algorithm cannot tell whether the true target function is $h_i$ or it is $h_{i^*}$. Since $\{(1,y_1),\ldots,(k_o,y_{k_o})\}$ is unrealizable, any proper classifier consistent with $S_{i,i^*}$ cannot predict $i$ as $h_{i^*}(i)=y_i$ and $i^*$ as $h_i(i^*) =y_{i^*}$ at the same time. Let $S_\trn$ be $m\leq\floor{k_o/2}$ i.i.d. samples from $\cD$ and then we have
  \begin{align*}
    & \sup_{i^*\in [k_o]}\EEs{S_\trn\sim \cD^m}{\atk(h^*,S_\trn,\A)}\\
    \geq &\EEs{i^*\sim \Unif([k_o]), S_\trn\sim \cD^m}{\atk(h^*,S_\trn,\A)}\\
    \geq & \EEs{i^*\sim \Unif([k_o]), S_\trn\sim \cD^m,(i,y_i)\sim \cD,\A}{\ind{\A(S_\trn\cup S_{i^*,i}, i)\neq h_{i^*}(i)\cap i\notin S_{\trn,\cX}}}\\
    \geq &\EEs{i^*, i\sim \Unif([k_o]\setminus\{i^*\} ) }{\EEs{S_\trn\sim\cD^m,\A}{\ind{\A(S_\trn\cup S_{i^*,i}, i)\neq h_{i^*}(i)|i\notin S_{\trn,\cX}}}\cdot \PP{i\notin S_{\trn,\cX}}  }\\
    \geq &\frac{1}{(k_o-1)k_o}\sum_{i^*=1}^{k_o}\sum_{i\neq i^*} \EEs{S_\trn\sim\cD^m,\A}{\ind{\A(S_\trn\cup S_{i^*,i}, i)\neq h_{i^*}(i)|i\notin S_{\trn,\cX}}}\cdot\frac{1}{2}\\
    = &\frac{1}{2(k_o-1)k_o}\sum_{i^*=1}^{k_o}\sum_{i\neq i^*} \EEs{S_\trn\sim\Unif^m(S_{i^*,i}),\A}{\ind{\A(S_\trn\cup S_{i^*,i}, i)\neq h_{i^*}(i)}}\\
    = &\frac{\sum_{i^*<i} \EEs{S_\trn\sim\Unif^m(S_{i^*,i}),\A}{\ind{\A(S_\trn\cup S_{i^*,i}, i)\neq h_{i^*}(i)}+\ind{\A(S_\trn\cup S_{i^*,i}, i^*)\neq h_{i}(i^*)}}}{2(k_o-1)k_o}\\
    \geq &\frac{(k_o-1)k_o}{4(k_o-1)k_o}=\frac{1}{4}\,. 
  \end{align*}
  
  For the second part, by Markov's inequality, we have
  \begin{align*}
      &\PPs{S_\trn\sim \cD^m}{\atk(h^*,S_\trn,\A)>1/8}=
      1-\PPs{S_\trn\sim \cD^m}{\atk(h^*,S_\trn,\A)\leq 1/8}\\
      \geq & 1- \frac{1-\EE{\atk(h^*,S_\trn,\A)}}{7/8}=\frac{1}{7}\,,
  \end{align*}
  which completes the proof.
\end{proof}

\begin{thm}
If $k_o = \infty$, then 
for any consistent \emph{proper} learning algorithm $\A$, for every $m \in \nats$, $\exists h^* \in \cH$ and distribution 
$\cD$ on $D_{h^*}$ such that $\E_{S_\trn \sim \cD^m}[ \atk_{\cD}(h^*,S_\trn,\A) ] \geq 1/4$.
\end{thm}
\begin{proof}
  For any hypothesis class with $k_o=\infty$, there exists a sequence of hollow star set $\{S_i\}_{i=1}^\infty$ with increasing size $\{k_i\}_{i=1}^\infty$ with $k_1\geq 3$. Then following the proof of Theorem~\ref{thm:hollow}, for any $m\leq \floor{(k_i-1)/2}$ for some $i$, there exists a target function and a data distribution such that the expected attackable rate is at least $1/4$ with sample size $m$. Since $k_i\rightarrow \infty$ as $i\rightarrow \infty$, this theorem is proved.
\end{proof}

\vspace{-20pt}
\subsection{Proof of Theorem~\ref{thm:opt-pac}}\label{appx:opt-pac}

Here we present the proof of Theorem~\ref{thm:opt-pac} 
establishing that any deterministic 
robust learner necessarily obtains a 
sample complexity with $O(1/\epsilon)$ 
dependence on $\epsilon$.

\begin{proof}[of Theorem~\ref{thm:opt-pac}]
Without loss of generality, we suppose 
$R(m) \leq 1$ and $R(m)$ is nonincreasing, 
since we can always replace it with 
$\sup_{m' \geq m} \min\{ R(m'), 1 \}$, 
which is monotone and inherits the other assumed 
properties of $R$.
For convenience, 
let us also extend the function $R(m)$ to non-integer values of $m$ by defining $R(\alpha) = R(\lfloor \alpha \rfloor)$, 
and defining $R(0) = 1$.
Also define $\Log(x) = \lceil \log_{2}(x) \rceil$ for any $x \geq 1$.

Fix any $h^* \in \cH$.
Since $\A$ is deterministic, note that for any finite multiset 
$S \subseteq D_{h^*}$ there is a set $\ATK_{S} \subseteq \cX$ corresponding to the points that would be attackable for $\A$ if $S_\trn = S$.
Moreover, we may note that the set $\ATK_{S}$ is \emph{non-increasing} in $S$ 
(subject to $S \subseteq D_{h^*}$), 
since adding any $(x,y) \in D_{h^*}$ to $S$ is equivalent to constraining the adversary to include these points in its attack set. 

Now we argue that $R\!\left(\frac{m}{\Log(1/\delta)}\right)$ 
is a $1-\delta$ confidence bound on 
$\atk_{\cD}(h^*,S_\trn,\A)$.
For any distribution $\cD$ on $D_{h^*}$, and any $\delta \in (0,1)$, 
if $m < \Log(1/\delta)$ then we trivially have $\atk_{\cD}(h^*,S_\trn,\A) \leq R\!\left(\frac{m}{\Log(1/\delta)} \right)$.
Otherwise, if $m \geq \Log(1/\delta)$, then letting $S_\trn \sim \cD^m$, 
letting $S_1$ be the first $\left\lfloor \frac{m}{\Log(1/\delta)} \right\rfloor$ elements of $S_\trn$, 
$S_2$ the next $\left\lfloor \frac{m}{\Log(1/\delta)} \right\rfloor$ elements of $S_\trn$, and so on up to $S_{\Log(1/\delta)}$, 
each $i \leq \Log(1/\delta)$ has, independently, probability at least $\frac{1}{2}$ of 
$\atk_{\cD}(h^*,S_i,\A) \leq R\!\left(\frac{m}{\Log(1/\delta)}\right)$.
In particular, this implies that, with probability at least $1 - (1/2)^{\Log(1/\delta)} \geq 1-\delta$, 
at least one $i \leq \Log(1/\delta)$ will satisfy this inequality.
Moreover, by the monotonicity property of $\ATK$, we know that 
$\ATK_{S_{\trn}} \subseteq \bigcap_{i \leq \Log(1/\delta)} \ATK_{S_{i}}$.
Thus, with probability at least $1-\delta$, 
\begin{align*} 
\atk_{\cD}(h^*,S_\trn,\A) 
& = \P_{(x,y) \sim \cD}( x \in \ATK_{S_{\trn}} ) 
\leq \min_{i \leq \Log(1/\delta)} \P_{(x,y) \sim \cD}( x \in \ATK_{S_i} ) 
\\ & = \min_{i \leq \Log(1/\delta)} \atk_{\cD}(h^*,S_i,\A) 
\leq R\!\left(\frac{m}{\Log(1/\delta)}\right).
\end{align*}

The remainder of the proof follows a familiar ``conditioning'' argument from the literature on log factors in the sample complexity of PAC learning 
\citep*[e.g.,][]{hanneke:thesis,hanneke2016refined}.
Fix any distribution $\cD$ over $D_{h^*}$.
We proceed by induction on $m$, establishing for each $m$ 
that $\forall \delta \in (0,1)$, for $S_{\trn} \sim \cD^m$, with probability at least $1-\delta$, 
$\atk_{\cD}(h^*,S_{\trn},\A) \leq \frac{c}{m}\Log\!\left(\frac{1}{\delta}\right)$, where $c$ is a finite $R$-dependent constant.
Note that this suffices to establish the theorem by taking $c_{R} = c$ 
(assuming base $2$ in the $\log$). 
The claim is trivially satisfied for $m < 3\Log\!\left(\frac{1}{\delta}\right)$, as the claimed bound is vacuous (taking any $c \geq 3$).
Now as an inductive hypothesis suppose $m \geq 3\Log\!\left(\frac{1}{\delta}\right)$ is such that, 
for every $m' < m$, for $S_\trn \sim \cD^{m'}$, 
for any $\delta \in (0,1)$,  
with probability at least $1-\delta$, 
$\atk_{\cD}(h^*,S_\trn,\A) \leq \frac{c}{m'}\Log\!\left(\frac{1}{\delta}\right)$.

Fix any $\delta \in (0,1)$ and let $S_\trn \sim \cD^m$.
Note that $\atk_{\cD}(h^*,S_\trn,\A) = \P_{(x,y) \sim \cD}( x \in \ATK_{S_{\trn}} )$.
Let $S_{\lfloor m/2 \rfloor}$ be the first $\lfloor m/2 \rfloor$ of the data points in $S_\trn$, 
and let $T = (S \setminus S_{\lfloor m/2 \rfloor}) \cap (\ATK_{S_{\lfloor m/2 \rfloor}} \times \cY)$: 
that is, $T$ are the samples in the last $\lceil m/2 \rceil$ points in $S_\trn$ 
that are in the attackable region when $\A$ has training set $S_{\lfloor m/2 \rfloor}$. 

Since, conditioned on $S_{\lfloor m/2 \rfloor}$ and $|T|$, the examples in $T$ 
are conditionally i.i.d.\ 
with each sample having distribution $\cD( \cdot | \ATK_{S_{\lfloor m/2 \rfloor}} \times \cY )$ on $D_{h^*}$, 
the property of $R(\cdot)$ established above 
implies that with conditional (given $S_{\lfloor m/2 \rfloor}$ and $|T|$) 
probability at least $1-\frac{\delta}{3}$, we have 
$\P_{(x,y) \sim \cD}( x \in \ATK_{T} | x \in \ATK_{S_{\lfloor m/2 \rfloor}} ) \leq R\!\left(\frac{|T|}{\Log(3/\delta)}\right)$. 
By the law of total probability, this inequality holds with (unconditional) probability at least $1-\frac{\delta}{3}$.

Furthermore, a Chernoff bound (applied under the conditional distribution given $S_{\lfloor m/2 \rfloor}$) 
and the law of total probability imply that, with probability at least $1-\frac{\delta}{3}$, 
if $\P_{(x,y) \sim \cD}( x \in \ATK_{S_{\lfloor m/2 \rfloor}} ) \geq \frac{16}{m} \ln\frac{3}{\delta}$, 
then $|T| \geq \P_{(x,y) \sim \cD}( x \in \ATK_{S_{\lfloor m/2 \rfloor}} ) \frac{m}{4}$.
Combining these two events with monotonicity of $R$, 
by the union bound, with probability at least $1-\frac{2}{3}\delta$, 
either 
$\P_{(x,y) \sim \cD}( x \in \ATK_{S_{\lfloor m/2 \rfloor}} )  < \frac{16}{m} \ln\frac{3}{\delta}$
or 
$\P_{(x,y) \sim \cD}( x \in \ATK_{T} | x \in \ATK_{S_{\lfloor m/2 \rfloor}} ) 
\leq R\!\left( \P_{(x,y) \sim \cD}( x \in \ATK_{S_{\lfloor m/2 \rfloor}} ) \frac{m}{4 \Log(3/\delta)} \right)$. 

Next, by monotonicity of $\ATK_S$, we have 
$\ATK_{S_{\trn}} \subseteq \ATK_{S_{\lfloor m/2 \rfloor}} \cap \ATK_{T}$.
Therefore, 
$\atk_{\cD}(h^*,S_\trn,\A)$ $\leq \P_{(x,y) \sim \cD}( x \in \ATK_{S_{\lfloor m/2 \rfloor}} ) \P_{(x,y) \sim \cD}( x \in \ATK_{T} | x \in \ATK_{S_{\lfloor m/2 \rfloor}} )$.
Thus, on the above event of probability at least $1-\frac{2}{3}\delta$,
either 
$\atk_{\cD}(h^*,S_\trn,\A)  < \frac{16}{m} \ln\frac{3}{\delta}$
or 
\begin{align*} 
\atk_{\cD}(h^*,S_\trn,\A) 
& \leq \P_{(x,y) \sim \cD}( x \in \ATK_{S_{\lfloor m/2 \rfloor}} ) R\!\left( \P_{(x,y) \sim \cD}( x \in \ATK_{S_{\lfloor m/2 \rfloor}} ) \frac{m}{4\Log(3/\delta)} \right) 
\\ & = \atk_{\cD}(h^*,S_{\lfloor m/2 \rfloor},\A) R\!\left( \atk_{\cD}(h^*,S_{\lfloor m/2 \rfloor},\A) \frac{m}{4\Log(3/\delta)} \right).
\end{align*}

By the inductive hypothesis, with probability at least $1-\frac{\delta}{3}$, we have that 
$\atk_{\cD}(h^*,S_{\lfloor m/2 \rfloor},\A) \leq \frac{c}{\lfloor m/2 \rfloor}\Log\!\left(\frac{3}{\delta}\right) \leq \frac{3 c}{m} \Log\!\left(\frac{3}{\delta}\right)$.
For any $\alpha \geq 1$, define $R'(\alpha) = \frac{1}{\alpha} \sup_{1 \leq \alpha' \leq \alpha} \alpha' R(\alpha')$,  
and note that $R(\alpha) \leq R'(\alpha)$ for all $\alpha \geq 1$, 
and $\alpha R'(\alpha)$ is nondecreasing in $\alpha \geq 1$.
Therefore, on the above event, 
\begin{align*}
& \atk_{\cD}(h^*,S_{\lfloor m/2 \rfloor},\A) R\!\left( \atk_{\cD}(h^*,S_{\lfloor m/2 \rfloor},\A) \frac{m}{4 \Log(3/\delta)} \right)
\\ & \leq \frac{3 c}{m} \Log\!\left(\frac{3}{\delta}\right) R'\!\left( \frac{3 c}{4} \right)
\leq \frac{9 c}{m} \Log\!\left(\frac{1}{\delta}\right) R'\!\left( \frac{3 c}{4} \right).
\end{align*}
Now note that $\lim_{\alpha \to \infty} R'(\alpha) = 0$. 
To see this, for the sake of contradiction, suppose $\exists \epsilon > 0$ and 
a strictly increasing sequence $\alpha_{t} \geq 1$ with $\alpha_{t} \to \infty$ such that $R'(\alpha_{t}) \geq \epsilon$, 
and let $\alpha'_{t}$ be any sequence with $1 \leq \alpha'_{t} \leq \alpha_{t}$ 
and $\frac{1}{\alpha_{t}} \alpha'_{t} R(\alpha'_{t}) \geq R'(\alpha_{t})/2 \geq \epsilon/2$.
If there exists an infinite subsequence $t_{i}$ with $\alpha'_{t_{i}}$ bounded above by some finite $\bar{\alpha}$, 
then $\lim_{i \to \infty} \frac{1}{\alpha_{t_{i}}} \alpha'_{t_{i}} R\!\left(\alpha'_{t_{i}}\right) \leq \lim_{i \to \infty} \frac{\bar{\alpha}}{\alpha_{t_{i}}}=0$: 
a contradiction.  Otherwise, we have $\alpha'_{t} \to \infty$, so that 
$\lim_{t \to \infty} \frac{1}{\alpha_{t}} \alpha'_{t} R(\alpha'_{t}) \leq \lim_{t \to \infty} R(\alpha'_{t}) = 0$: 
again, a contradiction.
Thus, since we have just established that $\lim_{\alpha \to \infty} R'(\alpha) = 0$, 
there exists a sufficiently large choice of $c$ 
for which $R'\!\left( \frac{3 c}{4} \right) \leq \frac{1}{9}$, 
so that  $\frac{9 c}{m} \Log\!\left(\frac{1}{\delta}\right) R'\!\left( \frac{3 c}{4} \right) \leq \frac{c}{m} \Log\!\left(\frac{1}{\delta}\right)$.

Altogether, by the union bound, we have 
established that 
with probability at least $1-\delta$, 
either \linebreak
$\atk_{\cD}(h^*,S_\trn,\A) < \frac{16}{m} \ln\frac{3}{\delta}$
or 
$\atk_{\cD}(h^*,S_\trn,\A) \leq \frac{c}{m} \Log\!\left(\frac{1}{\delta}\right)$.
Taking $c$ sufficiently large so that 
$c \geq 16 \ln(3e)$, 
both cases imply that
$\atk_{\cD}(h^*,S_\trn,\A) \leq \frac{c}{m} \Log\!\left(\frac{1}{\delta}\right)$.
The theorem now follows by the principle of induction.
\end{proof}

\vspace{-20pt}
\section{Proof of Theorem~\ref{thm:lblinear}}\label{appx:lb_linear}
In this section, we first formally prove the statement in the case of $n=3$, for which we already provided a proof sketch in Section~\ref{sec:linear}. In the proof sketch, we relax the definition of linear hypothesis class by allowing the decision boundary to be either positive or negative. Here, we adopt the convention that the boundary is only allowed to be positive. Then we prove the statement in the case of $n=2$, which requires a more delicate construction. Before proving Theorem~\ref{thm:lblinear}, we first introduce a lemma.

\begin{lm}\label{lmm:exp2hp}
For any hypothesis class $\cH$, any algorithm $\A$ and any $m>0$, if there exists a  universal constant $c>0$,
a distribution $\mu$ over $\cH$,
and a set of distributions $\cD(h)$ over $D_{h}$ for every $h\in \supp(\mu)$,
such that $\EEs{h\sim \mu, S_\trn \sim \cD(h)^m}{\atk_{\cD(h)}(h, S_\trn,\A)}\geq 2c$, then $\cH$ is not $(\epsilon,\delta)$-robust learnable.
\end{lm}
\begin{proof}
First, take $\eps = c$, we have by the definition of $\sup$, there exists an $h^*\in \cH$ such that,
\begin{align*}
    &\EEs{S_\trn \sim \cD(h^*)^m}{\atk_{\cD(h^*)}(h^*, S_\trn,\A)} \\
    \geq& \sup_{h\in \cH} \EEs{S_\trn \sim \cD(h)^m}{\atk_{\cD(h)}(h, S_\trn,\A)} - \eps \\
    \geq \textbf{}& \EEs{h\sim \mu, S_\trn \sim \cD(h)^m}{\atk_{\cD(h)}(h, S_\trn,\A)} - \eps\\
    \geq& c\,.
\end{align*}

Then by Markov's inequality,
\begin{align*}
    &\PP{\atk_{\cD(h^*)}(h^*, S_\trn,\A)>c/2}=1-\PP{\atk_{\cD(h^*)}(h^*, S_\trn,\A)\leq c/2}\\
    \geq& 1-\frac{1-\EE{\atk_{\cD(h^*)}(h^*, S_\trn,\A)}}{1-c/2}\geq \frac{c}{2-c}\,.
\end{align*}
Hence, $\cH$ is not $(\epsilon,\delta)$-robust learnable.
\end{proof}

\begin{proof}[of Theorem~\ref{thm:lblinear} in $n=3$] We divide the proof into three parts: a) the construction of the target function and the data distribution, b) the construction of the attacker and c) the analysis of the attackable rate.

\paragraph{The target function and the data distribution. }
We denote by $\sph = \sph^3(\bZero,1)$ the sphere of the $3$-dimensional unit ball centered at the origin. For some small $0<\eta<1/6$, let $\cH_\eta = \{h(x) = \ind{\inner{w}{x}-\frac{1}{2}\geq 0}|\norm{w}=1\}\cup \{h(x) = \ind{\inner{w}{x}-\frac{1-\eta}{2}\leq 0}|\norm{w}=1\}$ denote a set of linear classifiers with boundary $1/2$ or $\frac{1-\eta}{2}$ away from the origin. Let $K_w= \{x|\inner{w}{x}-\frac{1}{2}=0\}$ denote the hyperplane of the boundary of $h=\ind{\inner{w}{x}-\frac{1}{2}\geq 0}$ and $C_w = K_w\cap \sph$ denote the intersection of $K_w$ and $\sph$, which is a circle with radius $\sqrt{3}/2$ centered at $w/2$. 


We consider the target function $h^*$ selected uniformly at random from $\cH_\eta$, which is equivalent to: randomly picking $w\sim \Unif(\sph)$ and randomly picking $j\sim \Ber(1/2)$; if $j=1$, letting $h^* = h_{w,j}^* = \ind{\inner{w}{x}-\frac{1}{2}\geq 0}$; otherwise letting $h^* = h^*_{w,j} = \ind{\inner{w}{x}-\frac{1-\eta}{2}\leq 0}$. If the target function $h^*=h^*_{w,j}$, the data distribution $\cD = \cD_{w,j}$ is the uniform distribution over $C_w\times \{j\}$. Note that all instances on the circle $C_w$ are labeled as $j$ by $h^*_{w,j}$. We will show that the expected attackable rate $\EEs{h^*\sim \Unif(\cH_\eta), S_\trn\sim \cD^m}{\atk_\cD(h^*, S_\trn, \A)}\geq 1/2$. Combining with Lemma~\ref{lmm:exp2hp}, we prove Theorem~\ref{thm:lblinear} in $n=3$.

\paragraph{The attacker. }
Then we define the attacker $\Adv$ in the following way. We first define a map $m_{x_0}: \sph\mapsto \sph$ for some $x_0\in \sph$ such that $m_{x_0}(x) = {2\inner{x_0}{x}}x_0-x$. Here, $m_{x_0}(x)$ is the reflection of $x$ through the line passing the origin and $x_0$. Note that $m_{x_0}(m_{x_0}(x))=x$. This symmetric property will help to confuse algorithms such that no algorithm can distinguish the training data and the poisoning data. For $S_\trn\sim \cD_{w,j}^m$, we define $m_{x_0}(S_\trn) = \{ (m_{x_0}(x),1-y)|(x,y)\in S_\trn \}$, and let
  \begin{align*}
    \Adv(h^*_{w,j},S_\trn,x_0)=\begin{cases}
      m_{x_0}(S_\trn)& \text{if }S_{\trn,\cX}\cap \cB(x_0,\sqrt{3\eta/2})=\emptyset\,,\\
      \emptyset & \text{else.}
    \end{cases}
  \end{align*}
  Now we show that $\Adv$ is a clean-label attacker. In the second case of $\Adv(h^*_{w,j},S_\trn,x_0)=\emptyset$, it is clean-labeled trivially. In the first case of $S_{\trn,\cX}\cap \cB(x_0,\sqrt{3\eta/2})=\emptyset$, we discuss two cases:
  \begin{itemize}
    \item The target function $h^*=h_{w,j}^* = \ind{\inner{w}{x}-\frac{1}{2}\geq 0}$ has its decision boundary $\frac{1}{2}$ away from the origin, i.e., $j=1$. Then every training instance is labeled by $1$ and for any training instance $x$, $\inner{w}{m_{x_0}(x)}-\frac{1}{2} = \inner{x_0}{x}-1<0$. Hence $\Adv(h^*_{w,j},S_\trn,x_0)$ is clean-labeled.
    
    \item The target function $h^*=h_{w,j}^* = \ind{\inner{w}{x}-\frac{1-\eta}{2}\leq 0}$ has its decision boundary $\frac{1-\eta}{2}$ away from the origin, i.e., $j=0$. For each training instance $x$, since $x\notin \cB(x_0,\sqrt{3\eta/2})$, we have $\norm{x-x_0}_2^2\geq \frac{3\eta}{2}$ and thus, $\inner{x}{x_0}\leq 1-\frac{3\eta}{4}$. Then $\inner{w}{m_{x_0}(x)}-\frac{1-\eta}{2} = \inner{x_0}{x}-(1-\frac{\eta}{2})\leq 1-\frac{3\eta}{4}-(1-\frac{\eta}{2})<0$. Hence $\Adv(h^*_{w,j},S_\trn,x_0)$ is clean-labeled.
  \end{itemize}
  
  \paragraph{Analysis. }
  Let $\cE_1(h^*_{w,j},S_\trn,x_0)$ denote the event of $\{\A(S_\trn\cup \Adv(h^*_{w,j},S_\trn,x_0), x_0)\neq h^*_{w,j}(x_0)\}$ 
  and $\cE_2(S_\trn,x_0)$ denote the event of $S_{\trn,\cX}\cap \cB(x_0,\sqrt{3\eta/2})=\emptyset$. 
  It is not hard to check that $\cE_2(S_\trn,x_0) =\cE_2(m_{x_0}(S_\trn),x_0)$ due to the symmetrical property of the reflection. 
  Besides, conditional on $\cE_2(S_\trn,x_0)$, 
  the poisoned data set $S_\trn\cup \Adv(h^*_{w,j},S_\trn,x_0) = m_{x_0}(S_\trn)\cup \Adv(h^*_{m_{x_0}(w),1-j},m_{x_0}(S_\trn),x_0)$ 
  and thus, any algorithm $\A$ will behave the same (under attack) at test instance $x_0$ given training set $S_\trn$ or $m_{x_0}(S_\trn)$. 
  Since $h^*_{w,j}(x_0)\neq h^*_{m_{x_0}(w),1-j}(x_0)$, 
  we know that $\ind{\cE_1(h^*_{w,j},S_\trn,x_0)} = \ind{\neg\cE_1(h^*_{m_{x_0}(w),1-j},m_{x_0}(S_\trn),x_0)}$ conditional on $\cE_2(S_\trn,x_0)$. 
  Let $f_{w,j}(x)$ denote the probability density function of the marginal distribution of $\cD_{w,j}$ 
  (i.e., the uniform distribution over $C_w$) 
  and then we have $f_{w,j}(x) = f_{m_{x_0}(w),1-j}(m_{x_0}(x))$. 
  For any fixed $x_0$, the distributions of $w$ and $m_{x_0}(w)$ and the distributions of $j$ and $1-j$ are the same respectively. 
  The training set $S_\trn$ are samples drawn from $\cD_{w,j}$, and hence we can view $m_{x_0}(S_\trn)$ as samples drawn from $\cD_{m_{x_0}(w),1-j}$. Then for any algorithm $\A$, we have
  \begingroup
  \allowdisplaybreaks
  \begin{align}
    &\EEs{h^*_{w,j}\sim \Unif(\cH_\eta),S_\trn\sim \cD_{w,j}^m}{\atk_\cD(h^*_{w,j}, S_\trn, \A)}\nonumber\\
    \geq &\EEs{w\sim \Unif(\sph), j\sim\Ber(1/2),S_\trn\sim \cD_{w,j}^m,(x,y)\sim \cD_{w,j},\cA}{\ind{\cE_1(h^*_{w,j},S_\trn,x)\cap \cE_2(S_\trn,x)}}\nonumber\\
    =&\int_{x\in \sph}\EEs{w\sim \Unif(\sph), j\sim\Ber(1/2),S_\trn\sim \cD_{w,j}^m,\cA}{f_{w,j}(x)\ind{\cE_1(h^*_{w,j},S_\trn,x)\cap  \cE_2(S_\trn,x)}}dx\label{eq:toadd}\\
    =&\int\limits_{x\in \sph}\underset{{w, j,S_\trn\sim \cD_{w,j}^m,\cA}}{\E}{\left[f_{m_{x}(w),1-j}(x)\ind{\neg\cE_1(h^*_{m_x(w),1-j},m_x(S_\trn),x)\cap \cE_2(m_x(S_\trn),x)}\right]}dx\label{eq:sameevent}\\ 
    =&\int_{x\in \sph}\EEs{w, j,S_\trn\sim \cD_{w,j}^m,\cA}{f_{w,j}(x)\ind{\neg\cE_1(h^*_{w,j},S_\trn,x)\cap  \cE_2(S_\trn,x)}}dx\label{eq:samedist}\\
    =&\frac{1}{2}\int_{x\in \sph}\EEs{w\sim \Unif(\sph), j\sim \Ber(1/2),S_\trn\sim \cD^m,\cA}{f_{w,j}(x)\ind{\cE_2(S_\trn,x)}}dx
    \xrightarrow{\eta\rightarrow 0^+} \frac{1}{2}\,,\label{eq:addition}
  \end{align}
  \endgroup
  where Eq.~\eqref{eq:sameevent} uses the fact $\cE_2(S_\trn,x_0) =\cE_2(m_{x_0}(S_\trn),x_0)$ and that conditional on $\cE_2(S_\trn,x_0)$, $\ind{\cE_1(h^*_{w,j},S_\trn,x_0)} = \ind{\neg\cE_1(h^*_{m_{x_0}(w),1-j},m_{x_0}(S_\trn),x_0)}$; Eq.~\eqref{eq:samedist} uses the fact that for any fixed $x_0$, the distributions of $w$ and $m_{x_0}(w)$ and the distributions of $j$ and $1-j$ are the same respectively; and Eq.~\eqref{eq:addition} is the average of Eq.~\eqref{eq:toadd} and Eq.~\eqref{eq:samedist}.
\end{proof}

\begin{proof}[of Theorem~\ref{thm:lblinear} in $n=2$]
Again, we divide the proof into three parts.

\paragraph{The target function and the data distribution.}
In $2$-dimensional space, we denote by $w = (\cos\theta,\sin\theta)$ and represent the target function $h^* = \ind{\inner{(\cos\theta^*,\sin\theta^*)}{x}+b^*\geq 0}$ by $(\theta^*,b^*)$. Then the target function is selected in the following way: uniformly at random selecting a point $o$ from a $2$-dimensional ball centered at $\bZero$ with some large enough radius $r\!\geq\! 4$, i.e., $o\sim\! \Unif(\cB^2(\bZero,\!r))$, then randomly selecting a direction $\theta^* \sim \Unif([0,2\pi))$, and letting the target function be $h^* = \ind{\inner{(\cos \theta^*,\sin \theta^*)}{x-o}\geq 0}$. Then for any $m\in \NN$, we construct the data distribution over $2m$ discrete points, where all points are labeled the same and the distance between every two instances is independent of $h^*$. Specifically, the data distribution $\cD$ is described as follows. 
  \begin{itemize}
    \item We randomly draw $s \sim \Ber(1/2)$. We define two unit vectors $v_1 = (\sin \theta^*, -\cos \theta^*)$ and $v_2 =(2s-1)\cdot (\cos \theta^*,\sin \theta^*)$. Here $v_1$ is perpendicular to $w^*$ and $v_2$ is in the same direction as $w^*$ if $s=1$ and in the opposite direction of $w^*$ if $s=0$.
    \item Let $\cX_m = \{x_1,\ldots,x_{2m}\}$ be a set of $2m$ points. For notation simplicity, we also define $x_0$ and $x_{2m+1}$. Let $x_0 = o$ and for all $i\in [2m+1]$, let $x_{i} = x_{i-1}+l\cos(\beta_{i-1})v_1+ l\sin(\beta_{i-1})v_2$, where $\beta_{i}=7\beta_{i-1}$, $\beta_0 = 7^{-2m}\cdot \frac{\pi}{6}$ and $l=\frac{1}{2m}$.
    \item Let the marginal data distribution be a uniform distribution over $\cX_m$. Note that if $s=1$, all training points lie on the positive side of the decision boundary and are labeled by $1$; if $s=0$, all training points lie on the negative side and are labeled by $0$.
  \end{itemize}
  Here $(o,v_1,v_2)$ constructs a new coordinate system. For any $x\in \R^2$, we use $\tilde{x}= ((x-o)^\top v_1, (x-o)^\top v_2)$ to represent $x$ in this new coordinate system. Then the decision boundary of the target function is represented as $\inner{\tilde{x}}{\tilde{v_2}}=0$ and for any $x\in \cX_m$ we have $\inner{\tilde{x}}{\tilde{v_2}}>0$. It is worth noting that for any $i\in[2m]$, if the positions of three points $x_{i-1},x_i,x_{i+1}$ are fixed, then $o,s,v_1,v_2$ are all fixed.
  
  \paragraph{The attacker.}
  For any ${x_i}\in \cX_m$, we let $b_i=l\sum_{j=0}^i \cos \beta_j -\frac{l(\cos \beta_{i-1} +\cos \beta_i)\sum_{j=0}^i \sin \beta_j}{\sin \beta_{i-1}+\sin \beta_i}$ such that $z_i = {b_iv_1+o}$, ${x_{i-1}}$ and ${x_{i+1}}$ are collinear. We denote by $L_{x_i}$ the line passing $z_i$ and $x_{i}$. Then for any $i\in [2m]$, let
  \[
    \tilde{\Rfl_{L_{x_i}}(x)}=\frac{2\inner{\tilde{x}-\tilde{z_i}}{\tilde{x_i}-\tilde{z_i}}}{\norm{\tilde{x_i}-\tilde{z_i}}_2^2}(\tilde{x_i}-\tilde{z_i}) -\tilde{x}+2\tilde{z_i}\,,\]
  be the reflection of $x$ across $L_{x_i}$ as illustrated in Fig.~\ref{fig:lblin2d}. 
  
  \begin{figure}[H]
      \centering
      \includegraphics[width= 0.5\textwidth]{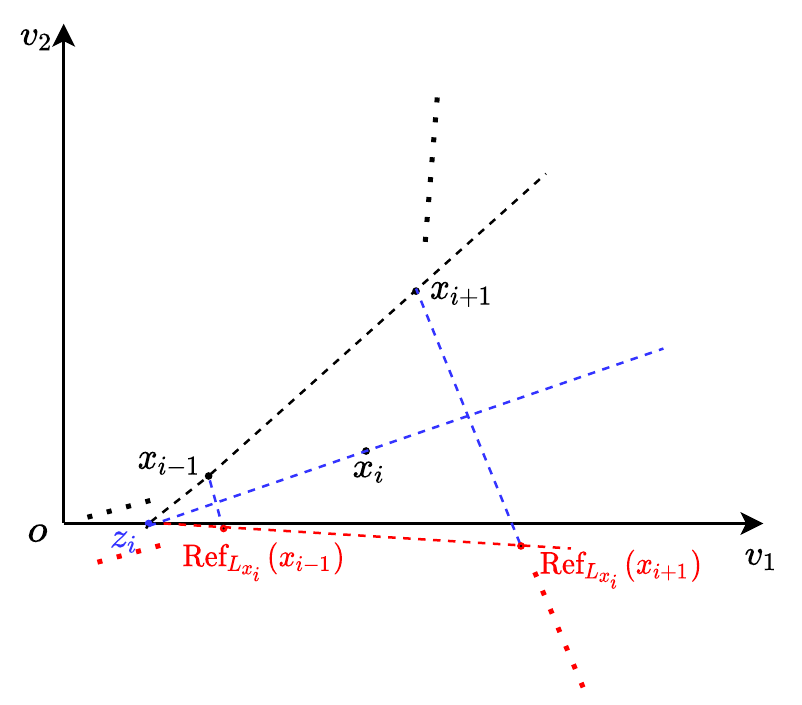}
      \caption{Illustration of $\Rfl_{L_{x_i}}(\cdot)$.}
      \label{fig:lblin2d}
\end{figure}

For $S_\trn\sim \cD^m$, we let $U=\cX_m\setminus S_{\trn,\cX}$ denote the set of points not sampled in $\cX_m$, where $\abs{U}\geq m$. Then we define $\Adv$ as
  \begin{align*}
      \Adv(h^*,S_\trn, x_i)=\begin{cases}
      \{(\Rfl_{L_{x_i}}(x),1-y)|(x,y)\in S_\trn\}&\text{if } x_i\in U\,,\\
      \emptyset & \text{else.}
      \end{cases}
  \end{align*}
  Now we need to show that $\Adv$ is a clean-label attacker. We will show that $\inner{\tilde{\Rfl_{L_{x_i}}(x_j)}}{\tilde{v_2}}<0$ for all $j\neq i\in [2m]$, which implies that the poison data is correctly labeled when $x_i\in U$. First, we claim that for any $j\neq i$, $x_j$ lie in the polytope above the line passing $x_{i-1},x_{i+1}$, the line passing $x_{i-1},x_i$ and the line passing $x_{i},x_{i+1}$. Formally speaking, for any $j\neq i$, $x_j$ satisfies
  \begin{align*}
      \begin{array}{l}
      \inner{(-\sin \beta_{i-1}-\sin \beta_i,\cos \beta_{i-1}+\cos \beta_i)}{\tilde{x_j}-\tilde{x_{i+1}}}\geq 0\,,\\
      \inner{(-\sin \beta_{i-1},\cos \beta_{i-1})}{\tilde{x_j}-\tilde{x_{i}}}\geq 0\,,\\
      \inner{(-\sin \beta_i,\cos \beta_i)}{\tilde{x_j}-\tilde{x_{i}}}\geq 0\,.
      \end{array}
  \end{align*}
  This claim is not hard to prove. At a high level, we prove this claim by that $\{\beta_i\}_{i=0}^{2m}$ is monotonically increasing and that the polygon defined by connecting every pair of neighboring points in $\cX_m\cup \{x_0,x_{2m+1}\}$ is convex. For the first constraint, it is satisfied trivially when $j=i-1,i+1$. If $j\geq i+2$, by direct calculation, we have 
  \begingroup
  \allowdisplaybreaks
  \begin{align*}
      &\inner{(-\sin \beta_{i-1}-\sin \beta_i,\cos \beta_{i-1}+\cos \beta_i)}{\tilde{x_j}-\tilde{x_{i+1}}}\\
      =&\inner{(-\sin \beta_{i-1}-\sin \beta_i,\cos \beta_{i-1}+\cos \beta_i)}{\left(\sum_{k=i+1}^{j-1}\cos \beta_k, \sum_{k=i+1}^{j-1}\sin \beta_k\right)}\\
      = &\sum_{k=i+1}^{j-1}\cos \beta_k\inner{(-\sin \beta_{i-1}-\sin \beta_i,\cos \beta_{i-1}+\cos \beta_i)}{\left(1, \frac{\sum_{k=i+1}^{j-1}\sin \beta_k}{\sum_{k=i+1}^{j-1}\cos \beta_k}\right)}\\
      \geq &\sum_{k=i+1}^{j-1}\cos \beta_k\inner{(-\sin \beta_{i-1}-\sin \beta_i,\cos \beta_{i-1}+\cos \beta_i)}{\left(1, \tan \beta_{i+1} \right)}\\
      \geq &\frac{\sum_{k=i+1}^{j-1}\cos \beta_k}{\cos \beta_{i+1}}(\sin(\beta_{i+1}-\beta_{i-1}) +\sin(\beta_{i+1}-\beta_{i}))\\
      \geq & 0\,.
  \end{align*}
  \endgroup
  Similarly, if $j\leq i-2$, then
  \begin{align*}
      &\inner{(-\sin \beta_{i-1}-\sin \beta_i,\cos \beta_{i-1}+\cos \beta_i)}{\tilde{x_j}-\tilde{x_{i+1}}}\\
      =&\inner{(-\sin \beta_{i-1}-\sin \beta_i,\cos \beta_{i-1}+\cos \beta_i)}{\tilde{x_j}-\tilde{x_{i-1}}}\\
      =&-\inner{(-\sin \beta_{i-1}-\sin \beta_i,\cos \beta_{i-1}+\cos \beta_i)}{\left(\sum_{k=j}^{i-2}\cos \beta_k, \sum_{k=j}^{i-2}\sin \beta_k\right)}\\
      = &\sum_{k=j}^{i-2}\cos \beta_k\inner{(\sin \beta_{i-1}+\sin \beta_i,-\cos \beta_{i-1}-\cos \beta_i)}{\left(1, \frac{\sum_{k=j}^{i-2}\sin \beta_k}{\sum_{k=j}^{i-2}\cos \beta_k}\right)}\\
      \geq &\sum_{k=j}^{i-2}\cos \beta_k\inner{(\sin \beta_{i-1}+\sin \beta_i,-\cos \beta_{i-1}-\cos \beta_i)}{\left(1, \tan \beta_{i-2} \right)}\\
      \geq &\frac{\sum_{k=j}^{i-2}\cos \beta_k}{\cos \beta_{i-2}}(\sin(\beta_{i-1}-\beta_{i-2}) +\sin(\beta_{i}-\beta_{i-2}))\\
      \geq & 0\,.
  \end{align*}
  It is easy to check that $x_j$ satisfies the second and the third constraints using the same way of computation, which is omitted here. Based on that $x_j$ lie in the polytope for all $j\neq i$, then we only need to prove that 
  $\inner{\tilde{\Rfl_{L_{x_i}}(x)}}{\tilde{v_2}}<0$ for the points lying on the faces of the polytope, which are $\tilde{x}\in \{\tilde{x_{i+1}} +\eta (\cos \beta_{i},\sin \beta_{i})|\eta\geq 0\}$, $\tilde{x}\in \{\tilde{x_{i-1}} -\eta (\cos \beta_{i-1},\sin \beta_{i-1})|\eta\geq 0\}$ and $\tilde{x}\in \{\eta \tilde{x_{i-1}} +(1-\eta)\tilde{x_{i+1}}|\eta\in [0,1]\}$. 
  Since $\tilde{\Rfl_{L_{x_i}}(\cdot)}$ is a linear transform, if we can show $\inner{\tilde{\Rfl_{L_{x_i}}(x_{i-1})}}{\tilde{v_2}}<0$ and 
  $\inner{\tilde{\Rfl_{L_{x_i}}(x_{i+1})}}{\tilde{v_2}}<0$, 
  then we have $\inner{\tilde{\Rfl_{L_{x_i}}(x)}}{\tilde{v_2}}<0$ for all points on the third face 
  $\{\eta \tilde{x_{i-1}} +(1-\eta)\tilde{x_{i+1}}|\eta\in [0,1]\}$. 
  Hence, we only need to prove the statement for points lying on the first two faces.
  
  For any two vectors $u,v$, we denote by $\theta(u,v)$ the angle between $u$ and $v$. Then let us denote by $\theta_1 = \theta(\tilde{x_{i+1}}-\tilde{z_i},\tilde{x_{i}}-\tilde{z_i})$ the angle between $\tilde{x_{i+1}}-\tilde{z_i}$ and  and $\tilde{x_{i}}-\tilde{z_i}$ and $\theta_2 = \theta(\tilde{x_{i}}-\tilde{z_i},\tilde{v_1})$ the angle between $\tilde{x_{i}}-\tilde{z_i}$ and $\tilde{v_1}$. Then we have both $\theta_1\leq \beta_{i}\leq \frac{\pi}{6}$ and $\theta_2\leq \beta_{i}\leq \frac{\pi}{6}$. Then since $\norm{\tilde{x_i}-\tilde{x_{i-1}}} = \norm{\tilde{x_{i+1}}-\tilde{x_{i}}} = l$ and $\theta(\tilde{x_{i+1}}-\tilde{x_{i-1}},\tilde{x_i}-\tilde{x_{i-1}}) = (\beta_i-\beta_{i-1})/2$ due to the construction, we have
  \begin{align*}
    &\sin \theta_1 = \frac{\norm{\tilde{x_{i}}-(\tilde{x_{i+1}}+\tilde{x_{i-1}})/2}}{\norm{\tilde{x_{i}}-\tilde{z_i}}}=\frac{l\sin((\beta_i-\beta_{i-1})/2)}{\norm{\tilde{x_{i}}-\tilde{z_i}}}
    =\frac{l\sin(3\beta_{i-1})}{\norm{\tilde{x_{i}}-\tilde{z_i}}}\geq \frac{3\beta_{i-1}l}{2\norm{\tilde{x_{i}}-\tilde{z_i}}}\,,
  \end{align*}
  where the last inequality is by $3\beta_{i-1}\leq \frac{\pi}{6}$. On the other hand, we have
  \begin{align*}
    \sin \theta_2 = \frac{l\sum_{k=0}^{i-1}\sin \beta_k}{\norm{\tilde{x_{i}}-\tilde{z_i}}}\leq \frac{l\sum_{k=0}^{i-1}\beta_k}{\norm{\tilde{x_{i}}-\tilde{z_i}}}\leq \frac{7\beta_{i-1}l}{6\norm{\tilde{x_{i}}-\tilde{z_i}}}\,.
  \end{align*} 
  Combining these two equations, we have $\sin\theta_1>\sin \theta_2$, which indicates that $\theta_1>\theta_2$. Then since $\beta_i-\theta_2 = \theta(\tilde{x_{i+1}}-\tilde{x_i},\tilde{x_{i}}-\tilde{z_i}) \geq \theta_1$, we have $\beta_i> 2\theta_2$. Let $\tilde{w_i} = \frac{\tilde{x_{i}}-\tilde{z_i}}{\norm{\tilde{x_{i}}-\tilde{z_i}}}$ denote the unit vector in the direction of $\tilde{x_{i}}-\tilde{z_i}$. For $\tilde{x}=\tilde{x_{i+1}} +\eta (\cos \beta_{i},\sin \beta_{i})$ with $\eta\geq 0$,
  \begin{align*}
      &\inner{\tilde{\Rfl_{L_{x_i}}(x)}}{\tilde{v_2}}\\
      =&{2\inner{\tilde{x}-\tilde{z_i}}{\tilde{w_i}}}\inner{\tilde{w_i}}{\tilde{v_2}} -\inner{\tilde{x}-\tilde{z_i}}{\tilde{v_2}}\\
      =&{2\inner{\tilde{x_{i+1}}-\tilde{z_i}+\eta (\cos \beta_{i},\sin \beta_{i})}{\tilde{w_i}}}\inner{\tilde{w_i}}{\tilde{v_2}} -\inner{\tilde{x_{i+1}}-\tilde{z_i}+\eta (\cos \beta_{i},\sin \beta_{i})}{\tilde{v_2}}\\
      =&2\norm{\tilde{x_{i+1}}-\tilde{z_i}}\cos \theta_1\sin \theta_2-\norm{\tilde{x_{i+1}}-\tilde{z_i}}\sin(\theta_1+\theta_2) +2\eta\cos(\beta_i-\theta_2)\sin \theta_2-\eta \sin\beta_i\\
      = &\norm{\tilde{x_{i+1}}-\tilde{z_i}}\sin(\theta_2-\theta_1) +\eta \sin (2\theta_2-\beta_i)\\
      < & 0\,.
  \end{align*}
  It is easy to check that $\beta_{i-1}\leq \theta_2$ (let $\tilde{p}$ denote the intersection of the line passing $\tilde{x_{i-1}}$ and $\tilde{x_i}$ and the line $\inner{\tilde{x}}{\tilde{v_2}}=0$, $\beta_{i-1} = \theta(\tilde{x_i}-\tilde{p}, \tilde{v_1})$ and $\theta_2$ is the external angle of triangle with vertices $\tilde{p}, \tilde{z_i}$ and $\tilde{x_i}$). Then for $\tilde{x}=\tilde{x_{i-1}} -\eta (\cos \beta_{i-1},\sin \beta_{i-1})$ with $\eta\geq 0$,
  \begin{align*}
      &\inner{\tilde{\Rfl_{L_{x_i}}(x)}}{\tilde{v_2}}\\
      =&{2\inner{\tilde{x}-\tilde{z_i}}{\tilde{w_i}}}\inner{\tilde{w_i}}{\tilde{v_2}} -\inner{\tilde{x}-\tilde{z_i}}{\tilde{v_2}}\\
      =&{2\inner{\tilde{x_{i-1}}-\tilde{z_i}-\eta (\cos \beta_{i-1},\sin \beta_{i-1})}{w_i}}\inner{\tilde{w_i}}{\tilde{v_2}} -\inner{\tilde{x_{i-1}}-\tilde{z_i}-\eta (\cos \beta_{i-1},\sin \beta_{i-1})}{\tilde{v_2}}\\
      =&2\norm{\tilde{x_{i-1}}-\tilde{z_i}}\cos \theta_1\sin \theta_2-\norm{\tilde{x_{i-1}}-\tilde{z_i}}\sin(\theta_1+\theta_2) -2\eta\cos(\beta_{i-1}-\theta_2)\sin \theta_2+\eta \sin\beta_{i-1}\\
      = &\norm{\tilde{x_{i-1}}-\tilde{z_i}}\sin(\theta_2-\theta_1) +\eta \sin (\beta_{i-1}-2\theta_2)\\
      < & 0\,.
  \end{align*}
  Now we complete the proof of $\inner{\tilde{\Rfl_{L_{x_i}}(x_j)}}{\tilde{v_2}}<0$ for all $j\neq i$ and that $\Adv$ is a clean-label attacker. It is worth noting that $L_{x_i'}=L_{x_i}$, where $L_{x_i'}$ is defined over $\{x_j'|j\in [2m]\}$ in the same way as $L_{x_i}$ defined over $\{x_j|j\in [2m]\}$. This is because reflections of $z_i$ and $x_i$ over $L_{x_i}$ are themselves. This symmetric property plays an important role in the analysis.

\paragraph{Analysis.}
Our probabilistic construction of the target function $h^*$ and the data distribution $\cD$ and the random sampling process of drawing $m$ i.i.d. samples from $\cD$ is equivalent to: sampling a multiset of indexes $I_\trn\sim \Unif([2m])$ first; then selecting the target function and the data distribution to determine the positions of the $m$ training points; mapping $I_\trn$ to $S_\trn$ by adding instance-label pair $(x_i,h^*(x_i))$ to $S_\trn$ for each $i$ in $I_\trn$. We let $I_u = [2m]\setminus I_\trn$ denote the indexes not sampled. As we know from the construction, for any $i\in[2m]$, once $s$ and the positions of $o$ and $x_i$ is determined, the positions of other points in $\cX_m$ and $h^*$ are determined. Then we consider an equivalent way of determining the target function and the data distribution. That is, randomly selecting the position of $x_i$ (dependent on the randomness of $o,\theta^*,s$) and then considering the following two different processes of selecting $s$ and $o$.
  \begin{itemize}
      \item Given a fixed $x_i$, randomly select $s\sim \cD(s|x_i)$ and select $o\sim \cD(o|s,x_i)$, where $\cD(s|x_i)$ and $\cD(o|s,x_i)$ denote the conditional distributions of $s$ and $o$ respectively. Note that when $x_i$ satisfies $\norm{x_i}\leq r-2$, $\cD(s|x_i) = \Ber(1/2)$ and $\cD(o|s,x_i)=\Unif(\sph^2(x_i,r_i))$ is a uniform distribution over the circle with radius $r_i$ centered at $x_i$, where $r_i$ is the distance between $o$ and $x_i$ and is a constant according to the definition.
      
      \item Given $(x_i,s,o)$ selected in the above process, if $\norm{x_i}\leq r-2$, we let $s'=1-s$ and $o'=\Rfl_{L_{x_i}}(o)$ (where $x_{i+1}$ and $x_{i-1}$ is determined by $(x_i,s,o)$); otherwise we let $s'=s$ and $o'=o$. It is easy to check that the distribution of $s'$ conditional on $x_i$ is $\Ber(1/2)$ and the distribution of $o'$ conditional on $x_i$ is $\Unif(\sph^2(x_i,r_i))$ if $\norm{x_i}\leq r-2$.
  \end{itemize}
  Therefore, the distributions of $s$ and $s'$ and the distributions of $o$ and $o'$ are the same given $x_i$ respectively. 
  Our following analysis depends on the event of $\norm{x_i}\leq r-2$, 
  the probability of which is $\PP{\norm{x_i}\leq r-2}\geq \PP{\norm{o}\leq r-3} = \frac{(r-3)^2}{r^2}$. 
  We let $S_\trn(x_i,s,o)$ denote the training set by mapping $I_\trn$ to the positions determined by $(x_i,s,o)$ 
  and let $h^*(x_i,s,o)$ denote the target function determined by $(x_i,s,o)$. 
  Note that when $\norm{x_i}\leq r-2$ and $x_i$ is not in the training set, the poisoned data sets with the training sets generated in the above two different processes are the same, i.e., 
  $S_\trn(x_i,s,o)\cup \Adv(h^*(x_i,s,o),S_\trn(x_i,s,o), x_i) = S_\trn(x_i,s',o')\cup \Adv(h^*(x_i,s',o'),S_\trn(x_i,s',o'), x_i)$.
  This is due to the symmetric property of the attacker. 
  Hence, any algorithm will behave the same at point $x_i$ no matter whether the training set is $S_\trn(x_i,s,o)$ or $S_\trn(x_i,s',o')$. 
  In addition, the target functions produced in the two different processes classify $x_i$ differently when $\norm{x_i}\leq r-2$. Let $\cE_2(x_i)$ denote the event of $\{\norm{x_i}\leq r-2\}$ and $\cE(h^*,\A,S_\trn,i)$ denote the event of $\A(S_\trn\cup \Adv(h^*,S_\trn,x_i), x_i)\neq h^*(x_i)$. Then for any $i\in I_u$, conditional on $\cE_2(x_i)$, 
  for any algorithm $\A$, we have 
  \begingroup
  \allowdisplaybreaks
  \begin{align}
      \ind{\cE(h^*(x_i,s,o),\A,S_\trn(x_i,s,o),i)} = \ind{\neg \cE(h^*(x_i,s',o'),\A,S_\trn(x_i,s',o'),i)}\,.\label{eq:s0}
  \end{align}
  Similar to the proof in the case of $n=3$, we have the expected attackable rate
    \begin{align}
      &\EEs{h^*,s,S_\trn\sim \cD^m}{\atk(h^*, S_\trn,\A)}\nonumber\\
      \geq &\frac{1}{2m}\EEs{o\sim \Unif(\{x:\norm{x}\leq r\}),\theta^* \sim \Unif(2\pi),s\sim \Ber(1/2),S_\trn\sim \cD^m,\A}{\sum_{x_i\in U}\ind{\cE(h^*,\A,S_\trn,i)}}\nonumber\\
      = &\frac{1}{2m}\EEs{I_\trn\sim \Unif([2m])}{\sum_{i\in I_u}\EEs{x_i,s,o,\A}{\ind{\cE(h^*(x_i,s,o),\A,S_\trn(x_i,s,o),i)}}}\nonumber\\
      \geq &\frac{1}{2m}\EEs{I_\trn}{\sum_{i\in I_u}\EEs{x_i}{\EEs{s,o,\A}{\ind{\cE(h^*(x_i,s,o),\A,S_\trn(x_i,s,o),i)}|x_i}\ind{\cE_2(x_i)}}}\nonumber\\
      = &\frac{1}{4m}\left(\EEs{I_\trn}{\sum_{i\in I_u}\EEs{x_i}{\EEs{s,o,\A}{\ind{\cE(h^*(x_i,s,o),\A,S_\trn(x_i,s,o),i)}|x_i}\ind{\cE_2(x_i)}}}\right.\nonumber\\
      +& \left.\EEs{I_\trn}{\sum_{i\in I_u}\EEs{x_i}{\EEs{s',o',\A}{\ind{\neg\cE(h^*(x_i,s',o'),\A,S_\trn(x_i,s',o'),i)}|x_i}\ind{\cE_2(x_i)}}}\right)\label{eq:s1}\\
      = &\frac{1}{4m}\EEs{I_\trn}{\sum_{i\in I_u}\EEs{x_i}{\ind{\norm{x_i}\leq r-2}}}\label{eq:s2}\\
      \geq&\frac{m(r-3)^2}{4mr^2}\nonumber\\
      \geq& \frac{1}{64},\nonumber
    \end{align}
    \endgroup
    when $r\geq 4$. Here Eq.~\eqref{eq:s1} holds due to Eq.~\eqref{eq:s0} and Eq.~\eqref{eq:s2} holds since the distributions of $s$ and $s'$ and the distributions of $o$ and $o'$ are the same given $x_i$ respectively. Combining with Lemma~\ref{lmm:exp2hp}, we complete the proof.
\end{proof}
\vspace{-20pt}
\section{Proof of Theorem~\ref{thm:lin2d}}\label{appx:lin2d}
\begin{proof} The proof contains three steps. For notation simplicity, we sometimes use $(\beta,b)$ to represent the linear classifier $h_{\beta,b}$. We say an angle $\beta$ is consistent with a data set $S$ if there exists $b\in [-2,2]$ such that $(\beta,b)$ is consistent with $S$. Also, for an fixed $\beta$, we say an offset $b$ is consistent with $S$ if $(\beta,b)$ is consistent with $S$.
\paragraph{Step 1: For any fixed $\beta$, the probability mass of union of error region is bounded.} For any $\beta$, if there exists any $b\in[-2,2]$ such that $(\beta,b)$ is consistent with $S_\trn$, from this set of consistent $b$'s, 
we denote by $b_{\sup}(\beta)$ the superior value of this set and $b_{\inf}(\beta)$ the inferior value of this set. 
By uniform convergence bound in PAC learning~\citep{blumer1989learnability}, when $m\geq \frac{4}{\epsilon'} \log \frac{2}{\delta} +  \frac{24}{\epsilon'} \log \frac{13}{\epsilon'}$,
we have that with probability at least $1-\delta$, every linear classifier $(\beta,b)$ consistent with $S_\trn$ has $\err(h_{\beta,b})\leq \epsilon'$.
Then the probability mass of the union of error region of all $(\beta,b)$ consistent with $S_\trn$ for a fixed $\beta$ is
\begin{align*}
    &\cP(\{x| \exists b\in [-2,2], h_{\beta,b}(x)\neq h^*(x), \forall (x',y')\in S_\trn, y'=h_{\beta,b}(x') \})\\
    = &\cP( \bigcup\limits_{b_{\inf}(\beta) \leq b\leq b_{\sup}(\beta) } \{x| b \text{~is consistent}~\&~ h_{\beta,b}(x)\neq h^*(x) \}  )\\
    =& \lim\limits_{\delta \to 0^+}  \cP( \{x| h_{\beta,b_{\inf}(\beta)+ \delta}(x)\neq h^*(x) \}\cup\{x|b_{\inf}(\beta)\text{~is consistent}~\&~ h_{\beta,b_{\inf}(\beta)}(x)\neq h^*(x) \} \\
    &\quad\cup  \{x|  h_{\beta,b_{\sup}(\beta)- \delta}(x)\neq h^*(x) \} \cup  \{x|b_{\sup}(\beta)\text{~is consistent} ~\&~ h_{\beta,b_{\sup}(\beta)}(x)\neq h^*(x) \} ) \\
    \leq &2\epsilon'\,.
\end{align*}
If there does not exist any consistent $b\in[-2,2]$ for $\beta$, then $\cP(\{x| \exists b\in [-2,2], h_{\beta,b}(x)\neq h^*(x), \forall (x',y')\in S_\trn, y'=h_{\beta,b}(x') \})=0$.

\paragraph{Step 2: The binary-search path of $\beta$ is unique and adding clean-label points can only change the depth of the search.} That is, for any fixed target function $(\beta^*,b^*)$, for $h-l=2\pi,\pi,\frac{\pi}{2},\ldots$, 
if $\frac{l+h}{2}$ is not consistent with the input (poisoned or not) data set $S$,
then there cannot exist $\beta$ consistent with the input data set $S$ in both two intervals $(l,\frac{l+h}{2})$ and $(\frac{l+h}{2},h)$. 
Since $\beta^*$ is always consistent with $S$, only the interval containing $\beta^*$ will contain $\beta$ consistent with $S$. 
To prove this statement, assume that any $\beta \in\{ l,h, \frac{l+h}{2}\}$ is not consistent with $S$ and there exists $(\beta_1,b_1)$ with $\beta_1\in (l,\frac{l+h}{2})$ and $(\beta_2,b_2)$ with $\beta_2\in (\frac{l+h}{2},h)$ consistent with $S$. If $\beta_2-\beta_1\leq \pi$, let $\beta_3 =\frac{l+h}{2}$; otherwise, let $\beta_3=l$. 
Since $(\beta_1,b_1)$ and $(\beta_2,b_2)$ are consistent classifiers,
for any $\alpha_1,\alpha_2\geq 0$, 
$\ind{(\alpha_1(\cos\beta_1,\sin \beta_1)+\alpha_2(\cos\beta_2,\sin \beta_2))\cdot x+ \alpha_1 b_1+\alpha_2 b_2\geq 0}$ is also a consistent classifier.
By setting $\alpha_1 = \frac{\sin (\beta_2-\beta_3)}{\sin (\beta_2-\beta_1)}$ 
and $\alpha_2 = \frac{\sin (\beta_3-\beta_1)}{\sin (\beta_2-\beta_1)}$, 
we have $(\beta_3, \frac{\sin (\beta_2-\beta_3)b_1+\sin (\beta_3-\beta_1)b_2}{\sin(\beta_2-\beta_1)})$ is consistent. If $\frac{\sin (\beta_2-\beta_3)b_1+\sin (\beta_3-\beta_1)b_2}{\sin(\beta_2-\beta_1)}\in [-2,2]$, this contradicts that any $\beta_3$ is not consistent with $S$; else, since $\cX\subseteq\cB^n(\bZero,1)$, there must exist $b\in[-2,2]$ such that $(\beta_3, b)$ is consistent, which is a contradiction.
  
\paragraph{Step 3: When $h-l< \arctan(f(\epsilon'')/2)$, the attackable rate caused by deeper search is at most $2\epsilon''$.} We consider two cases: $\abs{b^*}> 1$ and $\abs{b^*}\leq 1$. In the case of $\abs{b^*}> 1$, the target function classifies $\cX$ all positive or all negative and thus, there always exists a consistent $b$ for $\beta=0$. The binary-search for $\beta$ will not search in depth and output $(0,b)$ for some consistent $b$. Therefore, $\atk(h^*, S_\trn, \cA)\leq \cP(\{x| \exists b\in [-2,2], h_{0,b}(x)\neq h^*(x), \forall (x',y')\in S_\trn, y'=h_{0,b}(x') \})\leq 2\epsilon'$. 

In the case of $\abs{b^*}\leq 1$, let $A_{z}\eqdef \{x|(\cos \beta^*,\sin \beta^*)\cdot x +b^*\in [-z,z]\}$ for $z\geq 0$. We will show that when $h-l\leq {\arctan(z/2)}$ for $z\leq 2$, the classifier boundary $K_{\beta,b}$ must lie in $A_z$, i.e., $K_{\beta,b}\eqdef \{x|(\cos \beta,\sin \beta)\cdot x+b=0, \norm{x}\leq 1\}\subseteq A_z$. This indicates that when $h-l\leq {\arctan(z/2)}$, the algorithm can only make mistakes at points inside $A_z$. 
If it is false, let $\{q_1,q_2\}=\{x|(\cos \beta^*,\sin \beta^*)\cdot x+b^*=0\}\cap \sph(\bZero,1)$ with 
$\beta^*-\theta((q_2-q_1), e_1) = \frac{\pi}{2}\pmod{2\pi}$ 
if the intersection of the target function boundary and the unit circle has two different points; or let $q_1=q_2=q$ where $\{q\}=\{x|(\cos \beta^*,\sin \beta^*)\cdot x+b^*=0\}\cap \sph(\bZero,1)$ if intersection of the target function boundary and the unit circle has only one point. 
Similarly, we denote by $\{p_1,p_2\}$ the intersection of of $K_{\beta,b}$ and the unit circle. Since $h-l\leq\arctan(1)=\frac{\pi}{4}$, the input data set must contain both positive points and negative points and hence, $K_{\beta,b}\cap \sph(\bZero,1)$ is not empty. We let $\{p_1,p_2\}=\{x|(\cos \beta,\sin \beta)\cdot x+b=0\}\cap \sph(\bZero,1)$ with $\beta-\theta((p_2-p_1), e_1)  = \frac{\pi}{2} \pmod{2\pi}$
if the intersection of the boundary of classifier $(\beta,b)$ and the unit circle has two different points; or let $p_1=p_2=p$ where $\{p\}=\{x|(\cos \beta,\sin \beta)\cdot x+b=0\}\cap \sph(\bZero,1)$ if intersection of the boundary of classifier $(\beta,b)$ and the unit circle has only one point. The definitions of $q_1,q_2,p_1,p_2$ are illustrated in Fig.~\ref{fig:lin2d}. Then at least one of $\{p_1,p_2\}$ is not in $A_z$ and w.l.o.g., assume that $p_2\notin A_z$. 
It is easy to check that $p_1,p_2$ must lie on the same side of the boundary of the target function, otherwise $\abs{\beta-\beta^*}=\theta((p_2-p_1),(q_2-q_1))>\arctan(z/2)$, which contradicts $h-l\leq {\arctan(z/2)}$. Then there exists a consistent classifier $h'=\ind{\inner{w}{x-p_2}\geq 0}$ with $w = (q_{12}-p_{22}, p_{21}-q_{11})$ (whose boundary is the line passing $q_1$ and $p_2$). Then let $\beta'$ denote the direction of $h'$, i.e., $\cos\beta' =\frac{q_{12}-p_{22}}{\norm{w}}$ and 
$\sin\beta' = \frac{p_{21}-q_{11}}{\norm{w}}$, 
and we have $\abs{\beta'-\beta^*}= \theta(p_2-q_1,q_2-q_1)> \arctan(z/2)$, which contradicts that $h-l\leq {\arctan(z/2)}$.
\begin{figure}[H]
    \centering
    \includegraphics[width = 0.5\textwidth]{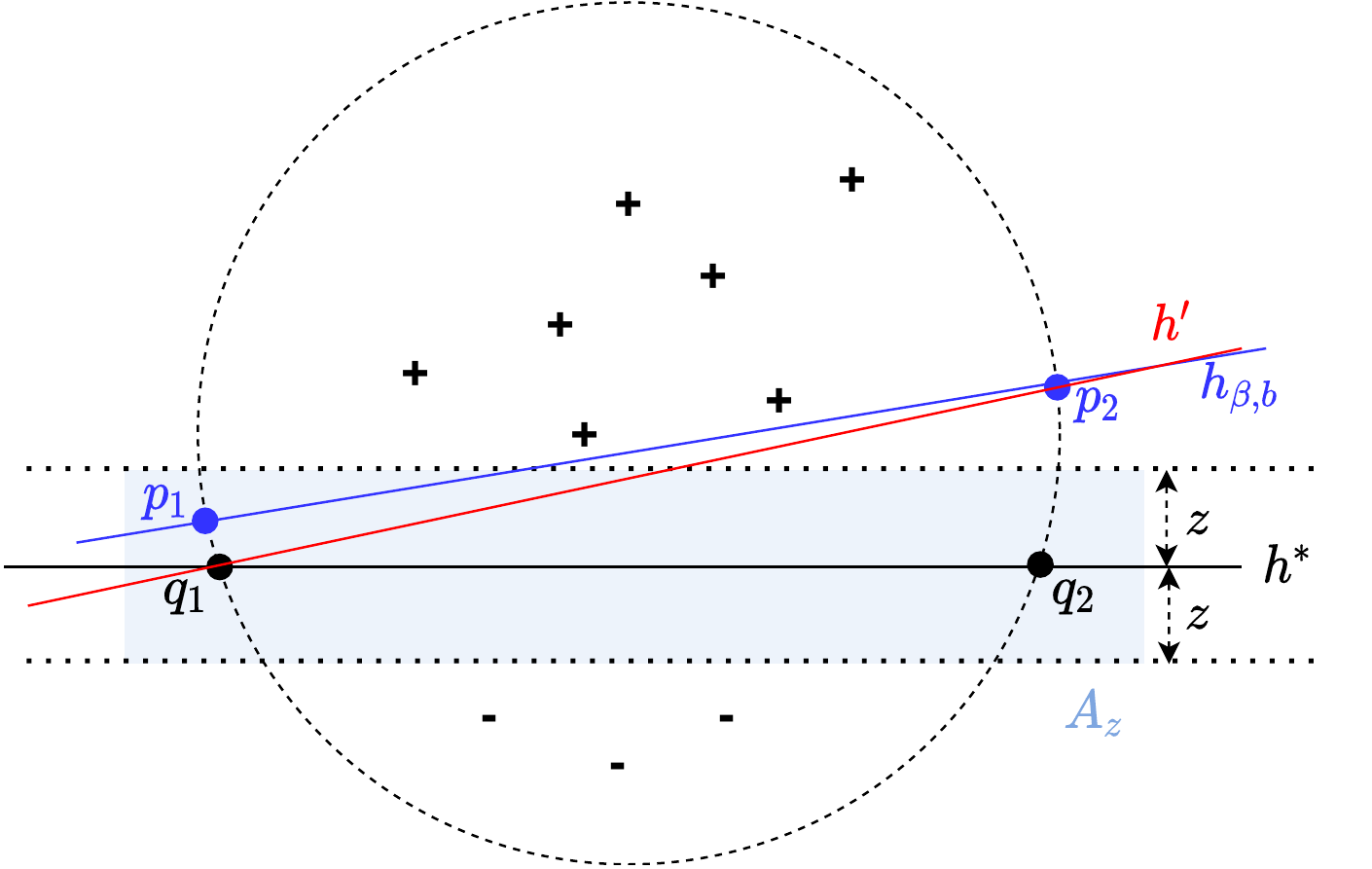}
    \caption{Illustration of $p_1,p_2,q_1,q_2,h'$.}
    \label{fig:lin2d}
\end{figure}
We let $z = f(\epsilon'')$. When $h-l\leq\arctan(f(\epsilon'')/2)$, the classifier can only make mistakes inside $A_{f(\epsilon'')}$ and thus the attackable rate is upper bounded by $2\epsilon''$. 
In addition, when $h-l>\arctan(f(\epsilon'')/2)$, the binary-search has searched to at most the $\floor{\log_2(\frac{2\pi}{\arctan (f(\epsilon'')/2)})}+1$-th depth, which leads to at most $2\epsilon'(\floor{\log_2(\frac{2\pi}{\arctan (f(\epsilon'')/2)})}+1)$ attackable rate. Combining these results together, when $m\geq \frac{4}{\epsilon'} \log \frac{2}{\delta} +  \frac{24}{\epsilon'} \log \frac{13}{\epsilon'}$, with probability at least $1-\delta$ over $S_\trn\sim \cD^m$, we have
  \begin{align}
    \atk(h^*, S_\trn, \cA) \leq &2\epsilon'\left(\floor{\log_2\left(\frac{2\pi}{\arctan (f(\epsilon'')/2)}\right)}+1\right) + 2\epsilon''\nonumber\\
    \leq& 2\epsilon'\log_2\left(\frac{4\pi}{\arctan (f(\epsilon'')/2)}\right) + 2\epsilon''\nonumber\\
    \leq &\log_2(\frac{4\pi}{\pi/8 (f(\epsilon'')\wedge 2)}) 2\epsilon' + 2\epsilon'' \label{eq:ss0}\\
    = &\log_2(\frac{32}{f(\epsilon'')\wedge 2})2\epsilon' + 2\epsilon'' \nonumber\,,
  \end{align}
  where Eq.~\eqref{eq:ss0} holds because $\arctan (x/2)\geq {\pi}/{4}$ when $x\geq 2$ and $\arctan (x/2)\geq {\pi x}/{8}$ when $x\in [0,2]$.
\end{proof}
\vspace{-20pt}
\section{Proof of Theorem~\ref{thm:linSVM}}\label{appx:linSVM}

To prove Theorem~\ref{thm:linSVM}, we first introduce a lemma on the behavior of uniform distribution on a unit sphere. 
\begin{lm}[Lemma~2.2 by \cite{ball1997elementary}]\label{lmm:ballconc}
    For any $a\in [0,1)$, for $x\sim \Unif(\sph^n(\bZero,1))$, with probability at least $1-e^{-na^2/2}$, we have $\inner{x}{e_1}\leq a$.
\end{lm}

\begin{proof}[of Theorem~\ref{thm:linSVM}] One essential hardness for robust learning of linear classifiers in high dimension is that for a fixed test instance $x_0$ on the sphere of a unit ball, with high probability over the selection of a set of training points from the uniform distribution over the sphere, every training instance has small component in the direction of $x_0$ as shown by Lemma~\ref{lmm:ballconc}. Taking advantage of this, the attacker can add a point (labeled differently from $x_0$) which is considerably closer to $x_0$ than all of the training instances, thus altering the behavior of SVM at $x_0$ as he wishes. In the following, we prove the theorem based on this idea. We divide the proof into three parts: a) the construction of the target function and the data distribution, b) the construction of the attacker and c) the analysis of the attackable rate. 
\paragraph{The target function and the data distribution.} 
The target function is $h^* = \ind{\inner{w^*}{x}\geq -\gamma/2}$ with $w^* = e_1$ and margin $\gamma= 1/8$. We define the marginal data distribution $\cD_{\cX}$ by putting probability mass $1-8\epsilon$ on $-e_1$ and putting probability mass $8\epsilon$ uniformly on the half sphere of a unit ball $\{x|\norm{x}=1,\inner{x}{e_1}\geq 0\}$. 
We let $\cD^+$ denote the uniform distribution over this positive half sphere. We draw $m$ i.i.d. training samples $S_\trn$ from $\cD$ and then let $S_\trn^+$ denote the positive training samples. Let $m^+ = \abs{S_\trn^+}$ denote the number of positive training samples.

\paragraph{The attacker.} For a given test instance $x_0\in \{x|\norm{x}=1,\inner{x}{e_1}\geq 0\}$, we define two base vectors $v_1 = e_1$ and $v_2 = \frac{x_0-\inner{x_0}{e_1}e_1}{\norm{x_0-\inner{x_0}{e_1}e_1}_2}$. Note that $v_2$ is well-defined almost surely. Then we define an attacker $\Adv$ as
\begin{align*}
    \Adv(h^*,S_\trn,x_0)=\begin{cases}
      \{(-v_2,1)\}&\text{if } m^+=0,\\
      \{(-\gamma v_1+\sqrt{1-\gamma^2}v_2,0)\}& \text{else}\,.
    \end{cases}
\end{align*}
Since $\inner{w^*}{v_2} = 0>-\frac{\gamma}{2}$ and $ \inner{w^*}{-\gamma v_1 + \sqrt{1-\gamma^2}v_2}=-\gamma <-\frac{\gamma}{2}$, $\Adv$ is a clean-label attacker.

\paragraph{Analysis.}
For $n\leq 128$, if $m<\frac{1}{8\epsilon}\vee \frac{e^{n/128}}{768\epsilon} = \frac{1}{8\epsilon} $ and $\eps < 1/16$, then $\PP{m^+=0}=(1-8\epsilon)^m\geq \frac{1}{4}$. Furthermore, if $m^+=0$, SVM can only observe instance-label pairs of $(-v_1,0)$ and $(-v_2,1)$ and then output $\hat{h}(x) =\ind{\inner{v_1-v_2}{x}\geq 0}$. Therefore, if $\inner{e_1}{x_0}<\frac{1}{\sqrt{2}}$, then $x_0$ is attackable. Therefore, for $m<\frac{1}{8\epsilon}$, we have 
  \begin{align*}
    \EEs{S_\trn\sim \cD^m}{\atk(h^*,S_\trn,\SVM)}\geq 8\epsilon\PPs{x\sim \cD^+}{\inner{x}{e_1}<\frac{1}{\sqrt{2}}} \PP{m^+=0} \geq \epsilon\,.
  \end{align*}
  For $n> 128$, if $m < \frac{1}{8\epsilon}$, the analysis above works as well. Else suppose $\frac{1}{8\epsilon} \leq m \leq \frac{e^{n/128}}{768\epsilon} $, we know that $\EE{m^+} = 8m\epsilon$, thus by Chernoff bounds, we have $\PP{m^+>32m\epsilon} \leq e^{-24m\epsilon}\leq e^{-3}$. Furthermore, by Lemma~\ref{lmm:ballconc} and the union bound, drawing $m_0$ i.i.d. samples $S_0\sim (\cD^{+})^{m_0}$, with probability at least $1-3m_0e^{-n/128}$, every instance $x\in S_0$ satisfies $\inner{x}{v_1}\leq \frac{1}{8}$ and $\inner{x}{v_2}\leq \frac{1}{8}$. Let $\cE$ denote the event of $\{\forall (x,y)\in S_\trn^+,\inner{x}{v_1}\leq \frac{1}{8},\inner{x}{v_2}\leq \frac{1}{8},1\leq m^+ \leq 32m\epsilon\}$. If $\cE$ holds, then there is a linear separator
  \[\left(\sqrt{\frac{1+\gamma}{2}}v_1 -\sqrt{\frac{1-\gamma}{2}}v_2\right)^\top x+\frac{1}{2}\sqrt{\frac{1+\gamma}{2}}+\frac{1}{16}\sqrt{\frac{1-\gamma}{2}}\geq 0\,,\]
  such that the distance between any point in $S_\trn \cup \Adv(h^*,S_\trn,x_0)$ and the linear separator is no smaller than $\frac{3}{8}-\frac{\sqrt{7}}{64}\geq \frac{1}{4}$. Hence the distance between the points and the seperator output by SVM is also no smaller than $\frac{1}{4}$. When the test instance $x_0$ satisfies $\inner{x_0}{v_1}\leq \frac{1}{8}$, we have $\norm{x_0-(-\gamma v_1+\sqrt{1-\gamma^2}v_2)}\leq \frac{1}{4}$ and then $x_0$ is misclassified as negative by SVM. Hence, we have
  \begin{align*}
    &\EEs{S_\trn\sim \cD^m}{\atk(h^*,S_\trn,\SVM)}\\
    \geq & 8\epsilon \EEs{S_\trn\sim \cD^m}{\PPs{x\sim \cD^+}{\SVM(S_\trn \cup \Adv(h^*,S_\trn,x),x)\neq h^*(x)}}\\
    \geq & 8\epsilon \EEs{x\sim \cD^+, S_\trn\sim \cD^m}{\ind{\forall (x',y')\in S_\trn^+, \inner{x'}{v_1}\leq \frac{1}{8},\inner{x'}{v_2}\leq \frac{1}{8},\inner{x}{v_1}\leq \frac{1}{8}}}\\
    \geq & 8\epsilon \EEs{x\sim \cD^+}{\EEs{S_\trn}{\ind{\cE}|x}\ind{\inner{x}{v_1}\leq \frac{1}{8}}}\\
    \geq & 8\epsilon (1-2e^{-n/128})(1-e^{-3}-e^{-8m\epsilon})(1-96m\epsilon e^{-n/128})\\
    \geq &\epsilon\,,
  \end{align*} 
  when $\frac{1}{8\epsilon} \leq m\leq \frac{e^{n/128}}{768\epsilon}$ and $n> 128$.  Thus in all we have shown that if $m< \frac{1}{8\epsilon} \vee \frac{e^{n/128}}{768\epsilon} $ then $\EEs{S_\trn\sim \cD^m}{\atk(h^*,S_\trn,\SVM)} > \epsilon$.
\end{proof}
\section{Proof of Theorem~\ref{thm:lblinmrg}}\label{appx:lb_linmrg}
\begin{proof}[of Theorem~\ref{thm:lblinmrg}]
The proof combines the idea of constructing a set of symmetrical poisoning instances in the proof of Theorem~\ref{thm:lblinear} and the idea that the training instances are far away from the test point in the proof of Theorem~\ref{thm:linSVM}. Again, we divide the proof into three parts as we did in the previous proofs.
\paragraph{The target function and the data distribution.}

We denote every point in $\R^n$ by $(x,z)$ for $x\in \R^{n-1}$ and $z\in \R$. The target function $h^*$ is selected uniformly at random from $\cH^*$, where $\cH^* = \{\ind{\inner{(jw^*,1)}{(x,z)}\geq j\gamma/2}|j\in\{\pm 1\}, w^*\in \sph^{n-1}(\bZero,1)\}$. Let $\gamma=\frac{1}{8}$. For target function $h^* = h_{w^*,j} = \ind{\inner{(jw^*,1)}{(x,z)}\geq j\gamma/2}$, the marginal data distribution $\cD_{w^*,j,\cX}$ puts probability mass $1-8\epsilon$ on the point $e_{n}$, then put the remaining $8\epsilon$ probability uniformly over the half sphere of a $(n-1)$-dimensional unit ball $\sph_{w^*,\gamma}\times\{0\}$, where $\sph_{w^*,\gamma}=\sph^{n-1}(\gamma w^*,1)\cap \{x|\inner{w^*}{x}\geq \gamma\}$. Then since every hypothesis in $\cH^*$ predicts $e_n$ positively, we only need to focus on the half sphere in the lower dimension. Note that the label of every point on the half sphere is determined by $j$. We sample $S_{\trn}\sim \cD^m_{w^*,j}$ and $S_{\trn,w^*,j}$ denote the samples on $\sph_{w^*,\gamma}\times\{0\}\times \cY$. Let $m_{w^*,j} = \abs{S_{\trn,w^*,j}}$.

\paragraph{The attacker.}

For any $u_1,u_2\in \R^{n-1}\setminus \{\bZero\}$, let $K_{u_1}(u_2) = \{x| \norm{u_1}^2 \inner{u_2}{x} - \inner{u_1}{u_2}\inner{u_1}{x} = 0 \}$ denote the homogeneous (passing through the origin) hyperplane perpendicular to the vector $u_2 - \inner{u_1}{u_2} \frac{u_1}{\norm{u_1}^2}$. For any given test instance $(x_0,0)$ with $x_0\in \sph_{w^*,\gamma}$, we define two base vectors $v_1=w^*$ and $v_2 = \frac{x_0-\inner{x_0}{w^*}w^*}{\norm{x_0-\inner{x_0}{w^*}w^*}}$. Note that $v_2$ is well-defined almost surely. Denote $K_{x_0} = K_{x_0}(v_1)$.
Let $x_\parallel= \inner{x}{v_1}v_1 + \inner{x}{v_2}v_2$ denote $x$'s component on the hyperplane defined by $v_1,v_2$ and $x_\perp = x-x_\parallel$ denote the component perpendicular to $v_1,v_2$ and then we define $\Rfl_{K_{x_0}}(x)\eqdef x_\perp + \frac{2\inner{x_\parallel}{x_0}}{\norm{x_0}^2}x_0  - x_\parallel$ as the reflection of $x$ through $K_{x_0}$.
Then we define an attacker $\Adv$ as
\begin{align*}
  &\Adv(h_{w^*,j},S_\trn,(x_0,0)) \\
  =&\begin{cases}
    \{((\Rfl_{K_{x_0}}(x),0),1-y)|((x,0),y)\in S_\trn\} & \text{if } \cE_1(w^*,S_\trn,x_0, m_{w^*,j})\,,\\
    \emptyset & \text{else}\,,
  \end{cases}
\end{align*}
where
\begin{align*}
    \cE_1(w^*,S_\trn,x_0, m_{w^*,j} ) = &\left\{\forall (x,0)\in S_{\trn,\cX}, \inner{x}{w^*}\leq \frac{1}{8} + \gamma, \inner{x}{\frac{x_0-\inner{x_0}{w^*}w^*}{\norm{x_0-\inner{x_0}{w^*}w^*}}}\leq \frac{1}{8}\right\} \\
    &\quad \cap \{\inner{x_0}{w^*}\leq \frac{1}{8}+\gamma\}\cap \{m_{w^*,j} \leq 32m\epsilon\}\,.
\end{align*}
Here $\cE_1(w^*,S_\trn,x_0, m_{w^*,j} )$ is thought as a condition to attack $x_0$. Then we show that $\Adv$ is a clean-label attacker.
If $\cE_1(w^*,S_\trn,x_0, m_{w^*,j})$ holds, we have
\begin{align*}
  &\inner{\Rfl_{K_{x_0}}(x)}{w^*}\\
  =& \inner{x_\perp + \frac{2\inner{x_\parallel}{x_0}}{\norm{x_0}^2}x_0  - x_\parallel}{v_1} \\
  =& \inner{x  - 2 \inner{x}{v_1}v_1 - 2\inner{x}{v_2}v_2 + 2 (\inner{x}{v_1}\inner{x_0}{v_1} + \inner{x}{v_2}\inner{x_0}{v_2}) \frac{x_0}{\norm{x_0}^2}}{v_1}\\
  =& -\inner{x}{v_1} + 2\inner{x}{v_1}\inner{x_0}{v_1}^2 \frac{1}{\norm{x_0}^2} + 2 \inner{x}{v_2}\inner{x_0}{v_2}\inner{x_0}{v_1} \frac{1}{\norm{x_0}^2}\\
  \leq& -\inner{x}{v_1} + 2(\frac{1}{8}+\gamma)^2 \inner{x}{v_1} + 2 \cdot \frac{1}{8} \inner{x_0}{v_1}\\
  \leq & \left(2(\frac{1}{8}+\gamma)^2-1\right)\gamma + \frac{1}{4}(\frac{1}{8}+\gamma)\\
  <& 0\,,
  \end{align*}
where the last inequality holds since $\gamma =\frac{1}{8}$. Therefore, $(\Rfl_{K_{x_0}}(x),0)$ is labeled different from $(x,0)$ and $\Adv$ is a clean-label attacker.

\paragraph{Analysis.}
Observe that the probabilistic construction of the target function and the data distribution along with the random sampling of a test instance and the training set can be viewed in an equivalent way: first drawing the number of training samples on the half sphere $m'$ from a binomial distribution $\Bin(m,8\epsilon)$; then on a fixed known half sphere, drawing a test instance and the training set with $m'$ samples on the half sphere and $m-m'$ samples on $e_n$; and finally randomly selecting a coordinate system to decide the position of the true sphere and selecting a $j$ to decide the labels of the training samples. Formally, let us fix a half sphere $\Gamma^{n-1}_+ = \{ x\in \Gamma^{n-1}(\bZero,1)|\inner{x}{e_1} \geq 0\}$ and then sample $m'\sim \Bin(m,8\epsilon)$, 
$t_0\sim \Unif(\Gamma^{n-1}_+)$ and 
$Q_\trn \sim \Unif(\Gamma^{n-1}_+)^{m'}$. 
We denote by $\cE_3(Q_\trn,t_0)$ the event of $\left\{\forall q\in Q_\trn, \inner{q}{e_1}\leq \frac{1}{8} , \inner{q}{\frac{t_0 - \inner{t_0}{e_1}e_1}{\norm{t_0 - \inner{t_0}{e_1}e_1}} }\leq \frac{1}{8}\right\}\cap \{\inner{t_0}{e_1}\leq \frac{1}{8}\} \cap\{m' = |Q_\trn| \leq 32m\epsilon\} $. Then we sample $T_{n-1}\sim \Unif(O(n-1))$, where $O(n-1)$ is the orthogonal group. Finally we sample $j\sim \Unif(\{\pm 1\})$. We denote by $R_{t_0}$ the linear isometry that reflects across the hyperplane $K_{\gamma e_1 + t_0}(e_1)$ in $\R^{n-1}$, i.e., $R_{t_0} u = \Rfl_{K_{\gamma e_1+t_0}(e_1)}(u)$, for any $u\in \R^n$. 

Conditional on $m', t_0$ and $Q_\trn$ sampled in the above process, we consider two different coordinate systems and $j$'s, which lead to two groups of random variables $(j, T_{n-1}, w^*, x_0, S_{\trn,\cX}, m_{w^*,j})$ and $(\wt{j}, \wt{T}_{n-1},\wt{w}^*,\wt{x}_0,\wt{S}_{\trn,\cX}, \wt{m}_{\wt{w}^*,\wt{j}})$. Here for any random variable in the first group, we add a tilde to represent the corresponding random variable in the second group.
\begin{itemize}
    \item In the first group, we have $j,T_{n-1},w^* = T_{n-1}e_1,  x_0 = T_{n-1}(\gamma e_1 + t_0),S_{\trn,\cX} = T_{n-1}(\gamma e_1 +  Q_\trn) \times \{0\} \cup \{e_n\}^{m-m'}$ and $m_{w^*,j} = m' $.
    \item In the second group, we let $\wt{j} = -j,\wt{T}_{n-1} = T_{n-1}R_{t_0}, \wt{w}^* = \wt{T}_{n-1}e_1, \wt{x}_0 =\wt{T}_{n-1}(\gamma e_1 +t_0) , \wt{S}_{\trn,\cX} = \wt{T}_{n-1}(\gamma e_1 + Q_\trn)\times \{0\}\cup \{e_n\}^{m-m'}$ and $\wt{m}_{\wt{w}^*,\wt{j}} = m'$.
\end{itemize}
The above two groups provide two ways of realizing the random process of selecting $h^*$, $(x_0,0)$ and $S_{\trn,\cX}$: 
namely, $h^* = h_{w^*,j} = \ind{\inner{(jw^*,1)}{(x,z)}\geq j\gamma/2}, (x_0,0), S_{\trn,\cX}$ and $h^* = h_{\wt{w}^*,\wt{j}} =\ind{\inner{(\wt{j}\wt{w}^*,1)}{(x,z)}\geq \wt{j}\gamma/2}, (\wt{x}_0,0),\wt{S}_{\trn,\cX}$. 
Let $\wt{S}_{\trn} = \{ (x, h_{\wt{w}^*,\wt{j}}(x))| x\!\in\!  \wt{S}_{\trn,\cX}  \}$ denote the data set of instances in $\wt{S}_{\trn,\cX}$ labeled by $h_{\wt{w}^*,\wt{j}}$. 
Note that $(w^*,j,S_\trn,x_0)$ and $(\wt{w}^*,\wt{j}, \wt{S}_\trn,\wt{x}_0)$ are identical in distribution and that $x_0= T_{n-1}(\gamma e_1 + t_0) = T_{n-1} R_{t_0}(\gamma e_1 + t_0)  = \wt{x}_0$. 
We now argue that $S_\trn\cup \Adv( h_{w^*,j},S_\trn, (x_0,0))$ and $\wt{S}_\trn\cup \Adv(h_{\wt{w}^*,\wt{j}},\wt{S}_\trn, (\wt{x}_0,0) )$ are identical conditional on $\cE_3( Q_\trn,\gamma e_1 + t_0)$. To prove this, we propose and prove the following three claims.

\paragraph{Claim I} $\cE_3(Q_\trn, t_0) \Leftrightarrow \cE_1(w^*, S_\trn,x_0, m_{w^*,j}) \Leftrightarrow \cE_1(\wt{w}^*,\wt{S}_{\trn,\cX}, \wt{x}_0, \wt{m}_{\wt{w}^*.\wt{j}})$.
\begin{proof}[of Claim I]
This is true since $T_{n-1}, R_{t_0}\in O(n-1)$, thus they keep all the inner product properties. In particular, for any $(x,0)\in S_\trn$, $x = T_{n-1} (\gamma e_1 + q)$ for some $q\in Q_\trn$ by definition of $S_\trn$. Furthermore,
\begin{gather*}
    \inner{x}{w^*} = \inner{T_{n-1}(\gamma e_1 + q)}{T_{n-1} e_1} = \gamma + \inner{q}{e_1}\,.
\end{gather*}
Thus $\inner{x}{w^*} \leq \frac{1}{8} + \gamma \Leftrightarrow \inner{q}{e_1}\leq \frac{1}{8} $. All the other equivalences can be derived similarly, thus omitted here.
\end{proof}

\paragraph{Claim II} For any homogeneous hyperplane $L_{u_1}\in \R^{n-1}$ with normal vector $u_1$, for any $u_2 \in \R^{n-1}$ , we have $\Rfl_{T_{n-1} L_{u_1}} (T_{n-1}u_2) = T_{n-1} \Rfl_{L_{u_1}}(u_2)$.
\begin{proof}[of Claim II]
We consider two cases. If $u_2 \in L_{u_1}$, then we have,
\begin{align*}
    \Rfl_{T_{n-1} L_{u_1}} (T_{n-1}u_2) = T_{n-1}u_2 = T_{n-1} \Rfl_{L_{u_1}}(u_2).
\end{align*}
Else if $u_2\notin L_{u_1}$, we 
denote by $u_3 = \Rfl_{T_{n-1} L_{u_1}} (T_{n-1}u_2)$. Thus $u_3$ is the only point such that $u_3\neq  T_{n-1}u_2$, $\inner{u_3 - T_{n-1}u_2}{T_{n-1}u_1} = 0 $ and 
$ \norm{u_3} = \norm{T_{n-1} u_2}$. 
These immediately give us $T_{n-1}^\top u_3 \neq u_2$, $\inner{T_{n-1}^\top u_3 - u_2}{u_1} = 0 $ and 
$\norm{T_{n-1}^\top u_3} = \norm{u_2}$, which means $T_{n-1}^\top u_3 = \Rfl_{L_{u_1}}(u_2)$. Thus
\begin{align*}
     \Rfl_{T_{n-1} L_{u_1}} (T_{n-1}u_2) = T_{n-1} T_{n-1}^\top u_3 = T_{n-1} \Rfl_{L_{u_1}}(u_2)\,,
\end{align*}
which completes the proof.
\end{proof}

\paragraph{Claim III} Conditional on $\cE_3(Q_\trn, t_0)$, the poisoned datasets $S_\trn\cup \Adv( h_{w^*,j},S_\trn, (x_0,0))$ and $\wt{S}_\trn\cup \Adv(h_{\wt{w}^*,\wt{j}},\wt{S}_\trn, (\wt{x}_0,0) )$ are identical.
\begin{proof}[of Claim III]
Denote by $K = K_{\gamma e_1 +t_0}(e_1)$ the homogeneous hyperplane perpendicular to $e_1 - \inner{\gamma e_1 + t_0}{e_1} \frac{\gamma e_1 + t_0}{\norm{\gamma e_1 + t_0}^2}$. 
Thus we have,
\begin{align*}
    T_{n-1} K &= \left\{ T_{n-1}x \bigg| \inner{x}{e_1 - \inner{\gamma e_1 + t_0}{e_1} \frac{\gamma e_1 + t_0}{\norm{\gamma e_1 + t_0}^2}} = 0 \right\}  \\
        &= \left\{ x \bigg| \inner{x}{T_{n-1} \left( e_1 - \inner{\gamma e_1 + t_0}{e_1} \frac{\gamma e_1 + t_0}{\norm{\gamma e_1 + t_0}^2}\right)} = 0 \right\}\\
        &= \left\{ x \bigg| \inner{x}{w^* - \inner{x_0}{w^*} \frac{x_0}{\norm{x_0}^2}} = 0 \right\}\\
        &= K_{x_0}(w^*)\,.
\end{align*}
By Claim I, we know that $\cE_3(Q_\trn, t_0)$ and $\cE_1(w^*, S_\trn,x_0, m_{w^*,j})$ are equivalent. Thus, conditional on $\cE_3(Q_\trn,t_0)$, 
    \begin{align}
        \Adv( h_{w^*,j},S_\trn, (x_0,0)) &= \{((\Rfl_{K_{x_0}(w^*)}(x),0),1-j)|((x,0),j)\in S_\trn\}\nonumber\\
        &=\{((\Rfl_{T_{n-1} K}(T_{n-1}(\gamma e_1+ q)),0),1-j)|q \in Q_\trn\}\nonumber \\
        &= \{(( T_{n-1}\Rfl_{K}(\gamma e_1+ q),0),1-j)|q \in Q_\trn\} \label{eq:applyclm2}\\
        &= \{(( T_{n-1}R_{t_0}(\gamma e_1+ q),0),1-j)|q \in Q_\trn\}\nonumber\\
        &= \wt{S}_\trn \setminus  \{(e_n,1)\}^{m-m'}\nonumber\,,
    \end{align}
    where Eq.~\eqref{eq:applyclm2} holds by applying Claim II. 
Similarly, for $\Adv(h_{\wt{w}^*,\wt{j}},\wt{S}_\trn, (x_0,0))$, the plane of reflection is $K_{x_0}(\wt{w}^*) =  K_{\wt{T}_{n-1}(\gamma e_1 + t_0)}(\wt{T}_{n-1}e_1) = \wt{T}_{n-1}K_{\gamma e_1 +t_0}(e_1) = \wt{T}_{n-1} K$ and
\begin{align}
    \Adv(h_{\wt{w}^*,\wt{j}},\wt{S}_\trn, (x_0,0)) &= \{((\Rfl_{K_{x_0}(\wt{w}^*)}(x),0),1-j)|((x,0),\wt{j})\in \wt{S}_\trn\} \nonumber\\
    &=\{((\Rfl_{\wt{T}_{n-1} K}(\wt{T}_{n-1}(\gamma e_1+ q)),0),j)|q \in Q_\trn\} \nonumber\\
    &= \{(( T_{n-1}(\gamma e_1+ q),0),j)|q \in Q_\trn\}\label{eq:applyclm22}\\
    &= S_\trn \setminus  \{(e_n,1)\}^{m-m'}\nonumber\,,
\end{align}
where Eq.~\eqref{eq:applyclm22} holds by applying Claim II and $\wt{T}_{n-1} = {T}_{n-1}R_{t_0}$.
Thus,
\begin{align*}
    S_\trn\cup \Adv( h_{w^*,j},S_\trn, (x_0,0)) &= S_\trn\cup \wt{S}_\trn \setminus  \{(e_n,1)\}^{m-m'} \\  &=\wt{S}_\trn\cup \Adv(h_{\wt{w}^*,\wt{j}},\wt{S}_\trn, (\wt{x}_0,0) )\,.
\end{align*}
\end{proof}
Now we have proved that $S_\trn\cup \Adv( h_{w^*,j},S_\trn, (x_0,0))$ and $\wt{S}_\trn\cup \Adv(h_{\wt{w}^*,\wt{j}},\wt{S}_\trn, (\wt{x}_0,0) )$ are identical conditional on $\cE_3( Q_\trn,\gamma e_1 + t_0)$. Hence in this case any algorithm will behave the same given the input data being either $S_\trn$ or $\wt{S}_\trn$. 
Let $\cE_2(\A,S_\trn,\Adv,h^*,x_0)$ denote the event $\A(S_\trn\cup \Adv(h^*,S_\trn,$ $(x_0,0)),(x_0,0))\neq h^*((x_0,0))$. Since $h_{w^*,j}((x_0,0)) \neq  h_{\wt{w}^*,\wt{j}}((x_0,0))$, 
then conditional on $\cE_3(Q_\trn,t_0)$, for any algorithm $\A$, 
we have 
$$\ind{\cE_2(\A,S_\trn,\Adv,h_{w^*,j},x_0)} = \ind{\neg\cE_2(\A,\wt{S}_\trn ,\Adv,h_{\wt{w}^*,\wt{j}},\wt{x}_0)}\,.$$ 

If $m <  \frac{1}{8\epsilon}$, then we have $ \mathbb{E}_{t_0,m', Q_\trn}[ \ind{\cE_3(Q_\trn,t_0)\cup\{m'=0\}}] \geq \PP{m'=0} = (1-8\epsilon)^m > \frac{1}{4}$ when $\epsilon\leq 1/16$. Else since $\EE{m'} = 8m\epsilon$, by Chernoff bounds, we have $\PP{m'>32m\epsilon} \leq e^{-24m\epsilon}\leq e^{-3}$. Furthermore, by Lemma~\ref{lmm:ballconc} and the union bound, drawing $m'$ i.i.d. samples $S_0\sim \Unif(\Gamma^{n-1}_+\times\{0\})^{m'}$, with probability at least $1-3m'e^{-\frac{n-1}{128}}$, every $(x,0)\in S_0$ satisfy $\inner{x}{e_1}\leq \frac{1}{8}$ and $\inner{x}{\frac{t_0 - \inner{t_0}{e_1}e_1}{\norm{t_0 - \inner{t_0}{e_1}e_1}} }\leq \frac{1}{8}$. Thus in all, we have, for any algorithm $\A$,
\begin{align}
&\mathbb{E}_{w^*,j,S_{\trn}\sim \cD_{w^*,j}^m, (x,y)\sim \cD_{w^*,j},\cA}[\atk(h^*,S_\trn, \cA)] \nonumber \\
=&\mathbb{E}_{w^*,j,S_{\trn}\sim \cD_{w^*,j}^m, (x,y)\sim \cD_{w^*,j},\cA}[\ind{\A(S_\trn\cup \Adv(h^*,S_\trn,x),x)\neq h^*(x)}] \nonumber \\
\geq& 8 \epsilon\mathbb{E}_{w^*,j,S_{\trn}, x_0\sim \Unif(\Gamma_{w^*,\gamma}) ,\A}[\ind{\cE_2(\A,S_\trn,\Adv,h_{w^*,j},x_0)}\nonumber\\
&\cdot \ind{\cE_1(w^*,S_\trn,x_0,m_{w^*,j})\cup \{m_{w^*,j}=0\}}]   \nonumber \\
=&8\epsilon \mathbb{E}_{t_0, m',Q_\trn}[ \ind{\cE_3(Q_\trn,t_0)\cup\{m'=0\}} \mathbb{E}_{T_{n-1},j,\cA} [ \ind{\cE_2(\A,S_\trn,\Adv,h_{w^*,j},x_0)}  ]]\nonumber\\
=&4\epsilon \mathbb{E}_{t_0,m', Q_\trn}[ \ind{\cE_3(Q_\trn,t_0)\cup\{m'=0\}} \mathbb{E}_{T_{n-1},j,\cA} [ \ind{\cE_2(\A,S_\trn,\Adv,h_{w^*,j},x_0)}  ]]\nonumber\\
&+4\epsilon \mathbb{E}_{t_0,m', Q_\trn}[ \ind{\cE_3(Q_\trn,t_0)\cup\{m'=0\}} \mathbb{E}_{T_{n-1},j,\cA} [ \ind{\neg\cE_2(\A,\wt{S}_\trn ,\Adv,h_{\wt{w}^*,\wt{j}},\wt{x}_0)}  ]] \label{myeq1}\\
=&4\epsilon \mathbb{E}_{t_0,m', Q_\trn}[ \ind{\cE_3(Q_\trn,t_0)\cup\{m'=0\}} \mathbb{E}_{T_{n-1},j,\cA} [ \ind{\cE_2(\A,S_\trn,\Adv,h_{w^*,j},x_0)}  ]]\nonumber\\
&+4\epsilon \mathbb{E}_{t_0,m', Q_\trn}[ \ind{\cE_3(Q_\trn,t_0)\cup\{m'=0\}} \mathbb{E}_{T_{n-1},j,\cA} [ \ind{\neg\cE_2(\A, S_\trn ,\Adv,h_{w^*,j},x_0)}  ]]\label{myeq2}\\
=& 4\epsilon \mathbb{E}_{t_0,m', Q_\trn}[ \ind{\cE_3(Q_\trn,t_0)\cup\{m'=0\}}] \nonumber\\
\geq&\begin{cases}
4 \epsilon (1 -2e^{-\frac{n-1}{128}})(1-e^{-3})(1-96m\epsilon e^{-\frac{n-1}{128}})& \text{when }m \geq \frac{1}{8\epsilon}\\
4 \epsilon(1-8\epsilon)^m &\text{when }m < \frac{1}{8\epsilon}
\end{cases}
\nonumber \\
>& \epsilon,\nonumber 
\end{align}
when $ m \leq \frac{e^\frac{n-1}{128}}{192\epsilon}$ and $n\geq 257$. Here Eq.~\eqref{myeq1} holds due to the fact that $\Adv$ will make $S_\trn\cup \Adv( h_{w^*,j},S_\trn, (x_0,0))$ and $\wt{S}_\trn\cup S_\Adv(h_{\wt{w}^*,\wt{j}},\wt{S}_\trn, (\wt{x}_0,0) )$ identical conditional on $\cE_3(Q_\trn,t_0)$ and Eq.~\eqref{myeq2} holds because $(w,j,S_\trn,x_0)$ is identical to $(\wt{w}^*,\wt{j}, \wt{S}_\trn,\wt{x}_0)$ in distribution. Thus in all, we have shown that for $n\geq 256$, if $ m \leq \frac{e^\frac{n-1}{128}}{192\epsilon}$ then for all algorithm $\cA$,   the expected attackable rate is $ \EEs{w,j,S_\trn\sim \cD^m}{\atk(h^*,S_\trn,\cA)} > \epsilon$. Thus there exists a target function $h^*\in\cH$ and a distribution $\cD$ over $D_{h^*}$ with margin $\gamma = 1/8$ such that $\EEs{S_\trn\sim \cD^m}{\atk_\cD(h^*, S_\trn,\A)}$ $> \epsilon$.
\end{proof}
\section{Proof of Theorem~\ref{thm:timproper}}\label{appx:improper_finite}
We first introduce two lemmas for the proof of the theorem.
\begin{lm}
\label{lmm:bennett}
For any hypothesis class $\cH$ with finite VC dimensional $d$, any distribution $\cD$, with probability at least $1-\delta$ over $S_\trn\sim \cD^m$, for all $h\in \cH$,
\begin{align*}
\err(h) - \err_{S_\trn}(h) \leq \sqrt{\frac{18d(1-\err_{S_\trn}(h))\err_{S_\trn}(h)\ln (em/\delta)}{m-1}} + \frac{15d\ln(em/\delta)}{m-1}\,.
\end{align*}  
\end{lm}

\begin{proof}
This lemma is a direct result of empirical Bennett's inequality (Theorem~6 by~\cite{maurer2009empirical}) and Sauer's lemma. Let $\Lambda(\cdot)$ denote the growth function of $\cH$. Then by empirical Bennett's inequality (Theorem~6 by~\cite{maurer2009empirical}), we have with probability at least $1-\delta$ over $S_\trn\sim \cD^m$,
\begin{align*}
      \err(h)-\err_{S_\trn}(h)\leq \sqrt{\frac{18(1-\err_{S_\trn}(h))\err_{S_\trn}(h)\ln(\Lambda(m)/\delta)}{m-1}} +\frac{15\ln (\Lambda(m)/\delta)}{m-1},\forall h\in\cH\,.
\end{align*}
By Sauer's lemma, $\Lambda(m)\leq (\frac{em}{d})^d$, which completes the proof.
\end{proof}

\begin{lm}
\label{lmm:subsample}
For any hypothesis class $\cH$ with finite VC dimensional $d$, a fixed data set $S$ with $m$ elements, realizable by some $h^*\in \cH$. Let $S_0$ be a set with size $m_0<m$ drawn from $S$ uniformly at random without replacement. Then with probability at least $1-\delta$, for all $h\in \cH$ with $\err_{S_0}(h) = 0$, we have
\begin{align*}
\err_{S}(h)\leq \frac{d\ln(em/d)+\ln(1/\delta)}{m_0}\,.
\end{align*} 
\end{lm}

\begin{proof}
For any $h\in \cH$, we have
  \begin{align*}
      &\PP{ \err_{S}(h)> \epsilon, \err_{S_{0}}(h)=0}
      \leq \frac{{m-k\choose m_0}}{{m\choose m_0}}
      \leq (1-k/m)^{m_0}\,,
  \end{align*}
  where $k = \ceil{\epsilon\cdot m}$.
By Sauer's lemma, $\Lambda(m)\leq (\frac{em}{d})^d$. Taking the union bound completes the proof.
\end{proof}

\begin{proof}[of Theorem~\ref{thm:timproper}]
Let $m=|S_\trn|$ be the number of training samples. Let $h_{i}=\cL(S^{(i)})$ denote the output hypothesis of block $i$. Let $N_{c} =\{i_1,\ldots,i_{n_c}\}\subseteq [10t+1]$ denote the set of index of non-contaminated blocks without poisoning points with $n_c = \abs{N_{c}}$. Each block has $m_0=\floor{\frac{m}{10t+1}}$ or $m_0 = \ceil{\frac{m}{10t+1}}$ data points (dependent on the actual number of poison points injected by the attacker) and at least $9t+1$ blocks do not contain any poison points, i.e., $n_c\geq 9t+1$. If a point $x$ is predicted incorrectly, then it is predicted incorrectly by more than $4t+1$ non-contaminated classifiers. Given training data $S_\trn$, for any $x\in\cX$, any $m_0$ and any $t$-point attacker $\Adv$ to make each block has $m_0$ points, we have
\begin{align*}
      &\PPs{\A}{\A(S_\trn\cup \Adv(S_\trn,h^*,x),x)\neq h^*(x)} \\
      =& \PPs{\A}{\sum_{i=1}^{10t+1}\ind{h_{i}(x)\neq h^*(x)}\geq 5t+1}\\
      \leq& \PPs{\A}{\sum_{i\in N_c}\ind{h_{i}(x)\neq h^*(x)}\geq 4t+1}\\
      \leq& \frac{1}{4t+1}\EEs{N_c}{\EEs{\A}{\sum_{i\in N_c}\ind{h_{i}(x)\neq h^*(x)}\big|N_c}}\\
      \leq& 2.5\EEs{N_c}{\EEs{\A}{\ind{h_{i_1}(x)\neq h^*(x)}|N_c}}\,.
\end{align*}
Notice here, if $m_0$ is fixed, 
the randomness of $\A$ can be regarded as selecting $N_c$ first and drawing $n_cm_0$ samples uniformly at random from $S_\trn$ without replacement to construct $S^{(i_1)}, S^{(i_2)}\,\ldots,S^{(i_{n_c})}$.
More specifically, conditioned on $N_c$, the randomness of $\cA$ on $h_{i_1}$ is only through drawing $S^{(i_1)}$, i.e., drawing $m_0$ samples without replacement from the clean training examples $S_\trn$. The important thing is that if $m_0$ is fixed, this distribution does not depend on the attacker. Hence,
\begin{align}\label{eq:atkvserr}
    &\atk(t,h^*,S_\trn,\A)\nonumber\\
    \leq& \EEs{(x,y)\sim \cD}{\sup_{\Adv} \PPs{\cA}{\A(S_\trn\cup \Adv(S_\trn,h^*,x),x)\neq h^*(x)}}\nonumber\\
    \leq & \EEs{(x,y)\sim \cD}{\sup_{\Adv: m_0=\floor{\frac{m}{10t+1}}} \PPs{\cA}{\A(S_\trn\cup \Adv(S_\trn,h^*,x),x)\neq h^*(x)}}\nonumber\\
    &+ \EEs{(x,y)\sim \cD}{\sup_{\Adv: m_0=\ceil{\frac{m}{10t+1}}} \PPs{\cA}{\A(S_\trn\cup \Adv(S_\trn,h^*,x),x)\neq h^*(x)}}\nonumber\\
    \leq& 2.5  \EEs{(x,y)\sim \cD}{  \sup_{\Adv: m_0=\floor{\frac{m}{10t+1}}} \EEs{N_c}{\EEs{\A}{\ind{h_{i_1}(x)\neq h^*(x)}|N_c}} }\\ 
    &+ 2.5  \EEs{(x,y)\sim \cD}{  \sup_{\Adv: m_0=\ceil{\frac{m}{10t+1}}} \EEs{N_c}{\EEs{\cA}{\ind{h_{i_1}(x)\neq h^*(x)}|N_c}} }\\
    \leq& 2.5\EEs{N_c}{\EEs{S^{(i_1)}}{\err\left(h_{i_1}\right)\Big|N_c}; m_0 =  \floor{\frac{m}{10t+1}}}\nonumber\\
    &+ 2.5\EEs{N_c}{\EEs{S^{(i_1)}}{\err\left(h_{i_1}\right)\Big|N_c}; m_0 =  \ceil{\frac{m}{10t+1}}}\,.
\end{align}
In the following, we will bound the error of $h_{i_1}$ for each value of $m_0$. Let $\cE$ denote the event of $\err_{S_\trn}(h_{i_1})\leq \frac{(d+1)\ln(em/d)}{m_0}$ and by Lemma~\ref{lmm:subsample} we have $\PPs{\A}{\neg \cE|N_c}\leq \frac{d}{em}$. Then with probability at least $1-\delta$ over the choice of $S_\trn$, for each fixed value of $m_0$, we have
\begin{align}
    &\EEs{\cA}{\err(h_{i_1})|N_c}\nonumber\\
    = & \EEs{\cA}{\err(h_{i_1})\ind{\cE}|N_c} + \EEs{\cA}{\err(h_{i_1})\ind{\neg \cE}|N_c}\nonumber\\
    = & \EEs{\cA}{(\err(h_{i_1})-\err_{S_\trn}(h_{i_1})+\err_{S_\trn}(h_{i_1}))\ind{\cE}|N_c} +  \PPs{\cA}{\neg \cE|N_c}\nonumber\\
    \leq &\sqrt{\frac{18d\ln (em/\delta)(d+1)\ln(em/d)}{(m-1)m_0}} + \frac{15d\ln(em/\delta)}{m-1} +\frac{(d+1)\ln(em/d)}{m_0}+\frac{d}{em}\label{eq:applybennett} \\
    \leq &\frac{1}{2} \left(  \frac{6d\ln(em/\delta)}{m-1} + \frac{6d\ln(em/d)}{m_0} \right) + \frac{15d\ln(em/\delta)}{m-1} +\frac{(d+1)\ln(em/d)}{m_0}+\frac{d}{em}\nonumber \\
    \leq & \frac{24d\ln(em)}{m_0} + \frac{19d\ln(1/\delta)}{10tm_0}\,, \label{eq:errimproper}
\end{align}
where Eq.~\eqref{eq:applybennett} applies Lemma~\ref{lmm:bennett}. Then when $m_0\geq  \frac{960 d}{\epsilon} \ln \frac{2640 e td}{\epsilon}  + \frac{19d\ln(1/\delta)}{\epsilon t}  $, we have that $\EEs{\cA}{\err(h_{i_1})|N_c}\leq \frac{24d\ln(em)}{m_0} + \frac{19d\ln(1/\delta)}{10tm_0}\leq 0.2\epsilon$. Combined with Eq.~\eqref{eq:atkvserr}, we have that when $m\geq (10t+1)(\frac{960 d}{\epsilon} \ln \frac{2640 e td}{\epsilon} + \frac{19d\ln(1/\delta)}{\epsilon t}+1)  $, the $t$-point attackable rate is $\atk(t,h^*,S_\trn,\A)\leq \epsilon$.
\end{proof}
\section{Proofs and discussions for $(t,\epsilon,\delta)$-robust proper learners}\label{appx:proper_finite}
\subsection{Proof of Theorem~\ref{thm:tpropproj}}
\begin{proof}
Similar to the proof of Theorem~\ref{thm:timproper}, we let $m=|S_\trn|$ be the number of training samples, and let $N_{c} =\{i_1,\ldots,i_{n_c}\}\subseteq [10tk_p+1]$ denote the set of index of blocks without poisoning points. Each block has $m_0=\floor{\frac{m}{10tk_p+1}}$ or $m_0=\ceil{\frac{m}{10tk_p+1}}$ data points and at least $t(10k_p-1)+1$ blocks do not contain any attacking points, i.e., $n_c\geq t(10k_p-1)+1$. 
Given any fixed $x\in\cX$, if $\sum_{i=1}^{10tk_p+1}\ind{h_{i}(x)\neq h^*(x)}\leq 10t < \frac{10tk_p+1}{k_p}$, then $x\in \cX_{\cH',k_p}$, thus $\hat{h}(x) = \Major(\cH',x) = h^*(x)$. Thus we have, given training data $S_\trn$, for any $x\in\cX$, 
any $m_0$ and any $t$-point attacker $\Adv$ to make each block has $m_0$ points, we have
\begin{align*}
      &\PPs{\A}{\A(S_\trn\cup \Adv(S_\trn,h^*,x),x)\neq h^*(x)} \\
      \leq& \PPs{\A}{\sum_{i=1}^{10tk_p+1}\ind{h_{i}(x)\neq h^*(x)}\geq 10t+1}\\
      \leq& \PPs{\A}{\sum_{i\in N_c}\ind{h_{i}(x)\neq h^*(x)}\geq 9t+1}\\
      \leq& \frac{1}{9t+1}\EEs{N_c}{\EEs{\A}{\sum_{i\in N_c}\ind{h_{i}(x)\neq h^*(x)}\Big|N_c}}\\
      \leq& \frac{10}{9}k_p\EEs{N_c}{\EEs{\A}{\ind{h_{i_1}(x)\neq h^*(x)}\Big|N_c}}\,,
\end{align*}
which indicates 
\begin{align*}
    \atk(t,h^*,S_\trn,\A)&\leq \frac{10}{9}k_p \EEs{N_c}{\EEs{\cA}{\err\!\left(h_{i_1}\right)|N_c}; m_0\!=\!\floor{\frac{m}{10tk_p+1}}} \\
    &\quad + \frac{10}{9}k_p\EEs{N_c}{ \EEs{\cA}{\err\!\left(h_{i_1} \right)|N_c};m_0\!=\!\ceil{\frac{m}{10tk_p+1}}}. 
\end{align*}
Then we bound $\EEs{\cA}{\err(h_{i_1})|N_c}$ in the same way as the proof of Theorem~\ref{thm:timproper}. Following the same calculation process of Eq.~\eqref{eq:errimproper}, we have with probability at least $1-\delta$, $\EEs{\cA}{\err(h_{i_1})|N_c}\leq \frac{24d\ln(em)}{m_0} + \frac{19d\ln(1/\delta)}{10k_ptm_0}$ by using Lemma~\ref{lmm:subsample}
and Lemma~\ref{lmm:bennett}. Then when $m_0\geq \frac{960 dk_p}{\epsilon} \ln \frac{2640 e tdk_p^2}{\epsilon} + \frac{19d\ln(1/\delta)}{\epsilon t} $, we have $\EEs{\cA}{\err(h_{i_1})|N_c}\leq \frac{24d\ln(em)}{m_0} + \frac{19d\ln(1/\delta)}{10k_ptm_0}\leq  \frac{0.2}{k_p}\epsilon $ . Therefore, we have that when $m\geq (10k_pt+1)(\frac{960 dk_p}{\epsilon} \ln \frac{2640 e tdk_p^2}{\epsilon} + \frac{19d\ln(1/\delta)}{\epsilon t}+1)$, the $t$-point attackable rate is $\atk(t,h^*,S_\trn,\A)\leq \epsilon$.
\end{proof}

\subsection{A proper learner for hypothesis class with no limitation over $k_p$}
\begin{algorithm}[H] \caption{A robust proper learner for $t$-point attacker}\label{alg:propert}
  {\begin{algorithmic}[1]
      \STATE \textbf{input}: a proper ERM learner $\cL$, data $S$
      \STATE uniformly at random pick $\floor{\frac{\abs{S}}{3t/\epsilon}}$ points $S_0$ from $S$ \emph{with replacement}
      \STATE \textbf{return} $\cL(S_0)$
  \end{algorithmic}}
\end{algorithm}

\begin{thm}
  For any hypothesis class with VC dimension $d$, with any proper ERM learner $\cL$, Algorithm~\ref{alg:propert} can $(t,\epsilon,\delta)$-robustly learn $\cH$ using $m$ samples where
  \[
      m=O\left(\frac{dt}{\epsilon^2} \log \frac{d}{\epsilon}+\frac{d}{\epsilon}\log \frac{1}{\delta}\right) \,.
  \]
\end{thm}
\begin{proof}
  Let $\cE$ denote the event that every point in $S_0$ is selected from the training data $S_\trn$. Let $m=\abs{S_\trn}$ and $h_0=\cL(S_0)$. Let us denote the size of $S_0$ by $m_0 = \floor{\frac{\abs{S}}{3t/\epsilon}}$, which can be $\floor{\frac{m}{3t/\epsilon}}$ or $\ceil{\frac{m}{3t/\epsilon}}$. Since $\PPs{\A}{\cE} \geq (1-\frac{t}{3tm_0/\epsilon})^{m_0}\geq 1-\frac{\ln 4}{3}\epsilon$, we have $\PPs{\A}{\neg \cE}\leq \frac{\ln 4}{3}\epsilon  $. Then for any $t$-point attacker $\Adv$, we have
  \begin{align*}
    &\PPs{\A}{\A(S_\trn\cup \Adv(h^*,S_\trn,x),x)\neq h^*(x)}\\
    = &\PPs{\A}{h_0(x)\neq h^*(x)\cap \cE}+\PPs{\A}{h_0(x)\neq h^*(x)\cap \neg \cE}\\
    \leq &\PPs{\A}{h_0(x)\neq h^*(x)| \cE}+\PPs{\A}{\neg \cE}\\
    \leq &\PPs{\A}{h_0(x)\neq h^*(x)| \cE}+ \frac{\ln 4}{3}\epsilon\,,
  \end{align*}
  which indicates $\atk(t,h^*,S_\trn,\A)\leq \EEs{\cA}{\err(h_{0})\Big|\cE; m_0 \!=\! \floor{\frac{m}{3t/\epsilon}}}+\EEs{\cA}{\err(h_{0})\Big|\cE; m_0 \!=\! \ceil{\frac{m}{3t/\epsilon}}}$ $+ \frac{\ln 4}{3}\epsilon$. Conditioned on $\cE$, $S_0$ is a set of i.i.d. samples uniformly drawn from $S_\trn$. By classic uniform convergence bound, $\err_{S_\trn}(h_0)\leq \frac{2}{m_0}(d\log(2em_0/d)+\log(2/\delta_0))$ with probability at least $1-\delta_0$ over the choice of $S_0$ (for a fixed $S_\trn$). Let $\cE_1$ denote the event of $\err_{S_\trn}(h_0)\leq \frac{2}{m_0}(d+1)\log(2em_0/d)$ and it is easy to check that $\PPs{\A}{\neg \cE_1}\leq \frac{d}{em_0}$. Similar to the proof of Theorem~\ref{thm:timproper}, with probability at least $1-\delta$, we have 
  \begin{align}
      &\EEs{\cA}{\err(h_{0})|\cE}\nonumber \\
      = & \EEs{\cA}{\err(h_{0})\ind{\cE_1}|\cE} + \EEs{\cA}{\err(h_{0})\ind{\neg \cE_1}|\cE}\nonumber\\
      \leq &\sqrt{\frac{36d\ln (em/\delta)(d+1)\log(2em_0/d)}{(m-1)m_0}} + \frac{15d\ln(em/\delta)}{m-1} +\frac{2}{m_0}(d+1)\log\frac{2em_0}{d}+\frac{d}{em_0}\label{eq:applybennett2}\\
      \leq & \frac{1}{2}  \left(    \frac{6d\ln(em/\delta)}{m-1} + \frac{12d\ln(2em_0/d)}{m_0}    \right) +  \frac{15d\ln(em/\delta)}{m-1} +   \frac{6d \ln (2em_0/d)}{m_0}   +\frac{d}{em_0} \nonumber\\
    \leq & \frac{13d\ln(2em_0/d)}{m_0} + \frac{18d\ln(em/\delta)}{m-1}\nonumber\\
    \leq & \frac{31d\ln(2em_0)}{m_0} + \frac{18d\ln(1/\delta)}{(3t/\epsilon-1)m_0}\,,\nonumber
  \end{align}
  where Eq.~\eqref{eq:applybennett2} adopts Lemma~\ref{lmm:bennett}. When $m_0\geq  \frac{1120d}{\epsilon} \ln  \frac{560e d}{\epsilon} + \frac{72 d\ln (1/\delta)}{t} $,  $\EEs{\cA}{\err(h_{0})|\cE} \leq 0.25\epsilon$. Hence, with probability at least $1-\delta$, the $t$-point attackable rate is $\atk(t,h^*,S_\trn,\A)\leq \epsilon$ by using $m$ training samples where
  \[m= \frac{3t}{\epsilon}\left(\frac{1120d}{\epsilon} \ln  \frac{560e d}{\epsilon} + \frac{72 d\ln (1/\delta)}{t}+1\right)\,.\]
\end{proof}

\section{Proof of Theorem~\ref{thm:lbfnt}}\label{appx:lb_finite}
\begin{proof}[of Theorem~\ref{thm:lbfnt}]
Now we show that for any sample size $m>0$, there exists a hypothesis class $\cH$ with VC dimension $5d$, a target function $h^*\in \cH$ and a data distribution $\cD$ on $D_{h^*}$ such that $\EEs{S_\trn\sim \cD^m}{\atk(t,h^*,S_\trn, \A)}\geq \min(\frac{3td}{64m},\frac{3}{8})$. We start with proving this statement in the base case of $d=1$ and then extend it to $d\geq 2$. We divide the proof into four parts: a) construction of the hypothesis class, the target function and the data distribution in $d=1$, b) computation of the VC dimension of the hypothesis class, c) construction of the attacker, and d) generalization to $d\geq 1$.
  
\paragraph{The hypothesis class, the target function and the data distribution.}
  
We denote by $\sph = \sph^3(\bZero,1)$ the sphere of the $3$-d unit ball centered at the origin. First, consider a base case where the domain $\cX = \sph\cup \bZero$, which is the union of the sphere of a unit ball centered at the origin and the origin. 
For any point $q\in \sph$, We let $C_q=\sph^3(q,1)\cap \sph$ denote the circle of intersection of the sphere of two unit balls. Then we define $h_{q,1}=\ind{C_q}$, which only classifies the circle $C_q$ positive and $h_{q,0}=\ind{\sph\setminus C_q}$ only classifies the circle and the origin negative. Our hypothesis class is $\cH = \{h_{q,j}|q\in \sph, j\in\{0,1\}\}$. We draw our target $h^*$ uniformly at random from $\cH$, which is equivalent to drawing $q\sim \Unif(\sph)$ and $j\sim \Ber(1/2)$. 
  The marginal data distribution $\cD_{q,j,\cX}$ puts probability mass $\zeta\in (0,\frac{t}{8m}]$ uniformly on the circle $C_q$ and puts the remaining probability mass on $\bZero$, where the value of $\zeta$ is determined later. We draw $S_\trn\sim \cD_{q,j}^m$. 
  

\paragraph{The VC dimension of the hypothesis class.} Then we show that the VC dimension of $\cH$ is $5$. Since all classifiers in $\cH$ will classify $\bZero$ as negative, $\bZero$ cannot be shattered and thus, we only need to find shattered points on the sphere. Then we show that $\cH$ can shatter $5$ points. It is not hard to check that the following set of $5$ points can be shattered: $\left\{(\frac{1}{2}, \frac{\sqrt{3}}{2} \cos(\frac{2k\pi}{5}), \frac{\sqrt{3}}{2} \sin(\frac{2k\pi}{5})) \right\}_{k=1}^5$.

Then we show that $\cH$ cannot shatter $6$ points. For any $6$ points $P=\{p_1,\ldots,p_6\}$, if the $6$ points can be shattered, then for any subset $P_1\subseteq P$ with size $3$, there exists a hypothesis classifying $P_1$ as $0$s and $P\setminus P_1$ as $1$s. That is, there exists a circle of radius $\frac{\sqrt{3}}{2}$ such that either only $P_1$ is on the circle or only $P\setminus P_1$ is on the circle. Then we claim that no $4$ points can be on a circle of radius $\frac{\sqrt{3}}{2}$. If there are $4$ points, w.l.o.g., $\{p_1,p_2,p_3,p_4\}$ on a circle of radius $\frac{\sqrt{3}}{2}$, then $\{p_i, p_5, p_6\}$ has to be on a circle $C_i$ of radius $\frac{\sqrt{3}}{2}$, where $1\leq j\neq i\leq 4$, $p_j\notin C_i$. But since the radius is fixed, there are only two different circles passing through $\{p_5,p_6\}$. Hence, there exists $1\leq i\neq j \leq 4$ such that $C_i = C_j$, which contradicts that $p_j\notin C_i$. 
  
Then w.l.o.g., if $\{p_1,p_2,p_3\}$ is on a circle of radius $\frac{\sqrt{3}}{2}$. Consider $\{p_1,p_2,p_4\}$ and $\{p_3,p_5,p_6\}$, if $\{p_1,p_2,p_4\}$ is on a circle of radius $\frac{\sqrt{3}}{2}$, then $\{p_3,p_4,p_5,p_6\}$ is on a circle of radius $\frac{\sqrt{3}}{2}$ (to label $\{p_1,p_2\}$ different from $\{p_3,p_4,p_5,p_6\}$); if $\{p_3,p_5,p_6\}$ is on a circle of radius $\frac{\sqrt{3}}{2}$, then there are three sub-cases: $\{p_1,p_3,p_5\}$ is on a circle of radius $\frac{\sqrt{3}}{2}$, $\{p_2,p_3,p_5\}$ is on a circle of radius $\frac{\sqrt{3}}{2}$ and both $\{p_2,p_4,p_6\},\{p_1,p_4,p_6\}$ are on two circles of radius $\frac{\sqrt{3}}{2}$. For the first case, $\{p_1,p_2,p_4,p_6\}$ is on a circle of radius $\frac{\sqrt{3}}{2}$ (to label $\{p_3,p_5\}$ different from $\{p_1,p_2,p_4,p_6\}$). For the second case, similarly $\{p_1,p_2,p_4,p_6\}$ is on a circle of radius $\frac{\sqrt{3}}{2}$. For the third case, $\{p_1,p_2,p_3,p_5\}$ is on a circle of radius $\frac{\sqrt{3}}{2}$. Therefore, any $6$ points cannot be shattered.

\paragraph{The attacker.}
We adopt the reflection function $m_{x_0}(\cdot)$ defined in the proof of Theorem~\ref{thm:lblinear} where $m_{x_0}(x) = {2\inner{x_0}{x}}x_0-x$ for $x\in \sph$. For $S_\trn\sim \cD_{q,j}^m$, we let $S_q=C_q\cap S_{\trn,\cX}$ denote the training instances in $C_q$ (with replicants) and we further define $m_{x_0}(S_\trn) = \{ (m_{x_0}(x),1-y)|(x,y)\in S_q\times \cY \}$, and let
\begin{align*}
    \Adv(h^*,S_\trn,x_0)=\begin{cases}
      m_{x_0}(S_\trn)& \text{if } x_0\notin S_{\trn,\cX}, \abs{S_q}\leq t\,,\\
      \emptyset & \text{else}\,.
    \end{cases}
\end{align*}
If $x_0\notin S_{\trn,\cX}$, then $h_{q,j}$ is consistent with $S_\trn\cup \Adv(h_{q,j},S_\trn,x_0)$. That is, $\Adv(h_{q,j},S_\trn,x_0)$ is clean-labeled.

\paragraph{Analysis.}
Due to the construction, we have
\begin{align*}
      \EEs{S_\trn\sim \cD_{q,j}^m}{\abs{S_q}} = m\zeta\,.
\end{align*}
Then by Markov's inequality, we have
\begin{align*}
      \PPs{{S_\trn\sim \cD_{q,j}^m}}{\abs{S_q}\geq t}\leq \frac{m\zeta}{t}< \frac{1}{4}\,.
\end{align*}
Let $\cE_1(\A,\Adv,h_{q,j},S_\trn,x_0)$ denote the event of $\{\A(S_\trn\cup \Adv(h_{q,j},S_\trn,x_0), x_0)\neq h_{q,j}(x_0)\}$ and let $\cE_2(S_\trn,x_0,q)$ denote the event of $\{\abs{S_q}\leq t \cap x_0\notin S_{\trn,\cX}\}$. It is easy to check that $\cE_2(S_\trn,x_0,q) =\cE_2(m_{x_0}(S_\trn),x_0,m_{x_0}(q))$. Besides, conditional on $\cE_2(S_\trn,x_0,q)$, we have the poisoned data set $S_\trn\cup \Adv(h_{q,j},S_\trn,x_0) = m_{x_0}(S_\trn)\cup \Adv(h_{m_{x_0}(q),1-j},m_{x_0}(S_\trn),x_0)$ and thus, any algorithm $\A$ will behave the same at the test instance $x_0$ no matter whether the training set is $S_\trn$ or $m_{x_0}(S_\trn)$. Since $h_{q,j}(x_0)\neq h_{m_{x_0}(q),1-j}(x_0)$, we have $\ind{\cE_1(\A,\Adv,h_{q,j},S_\trn,x_0)} = \ind{\neg\cE_1(\A,\Adv,h_{m_{x_0}(q),1-j},m_{x_0}(S_\trn),x_0)}$ conditional on $\cE_2(S_\trn,x_0,q)$. 
Let $f_{q}(x)$ denote the probability density function of $\Unif(C_q)$ and then we have $f_{q}(x) = f_{m_{x_0}(q)}(m_{x_0}(x))$. 
For any fixed $x_0$, the distributions of $q$ and $m_{x_0}(q)$ and the distributions of $j$ and $1-j$ are the same respectively. 
Since $S_\trn$ are samples drawn from $\cD_{q,j}^m$, $m_{x_0}(S_\trn)$ are actually samples drawn from $\cD_{m_{x_0}(q),1-j}^m$. Then we have
\begin{align*}
    &\EEs{h^*\sim \Unif(\cH),S_\trn\sim \cD^m}{\atk_\cD(t,h^*,S_\trn, \A)}\\
    =&\zeta\EEs{q\sim \Unif(\sph), j\sim\Ber(\frac{1}{2}),S_\trn\sim \cD_{q,j}^m,x\sim \Unif(C_q),\A}{\ind{\cE_1(\A,\Adv,h_{q,j},S_\trn,x)}}\\
    \geq &\zeta\EEs{q\sim \Unif(\sph), j\sim\Ber(\frac{1}{2}),S_\trn\sim \cD_{q,j}^m,x\sim \Unif(C_q),\A}{\ind{\cE_1(\A,\Adv,h_{q,j},S_\trn,x)\cap \cE_2(S_\trn,x,q)}}\\
    =&\zeta\int_{x\in \sph}\EEs{q\sim \Unif(\sph), j\sim\Ber(\frac{1}{2}),S_\trn\sim \cD_{q,j}^m,\A}{f_{q}(x)\ind{\cE_1(\A,\Adv,h_{q,j},S_\trn,x)\cap \cE_2(S_\trn,x,q)}}dx\\
    =&\zeta\int_{x\in \sph}{\E}_{q\sim \Unif(\sph), j\sim\Ber(\frac{1}{2}),S_\trn\sim \cD_{q,j}^m,\A} [f_{m_x(q)}(x)\ind{\neg \cE_1(\A,\Adv,h_{m_x(q),1-j},m_x(S_\trn),x)}\\
    &\cdot \ind{\cE_2(m_x(S_\trn),x,m_{x}(q))}]dx\\
    =&\zeta\int_{x}\EEs{q\sim \Unif(\sph), j\sim\Ber(\frac{1}{2}),S_\trn\sim \cD_{q,j}^m,\A}{f_{q}(x)\ind{\neg \cE_1(\A,\Adv,h_{q,j},S_\trn,x)} \cdot \ind{\cE_2(S_\trn,x,q)}}dx\\
    =&\frac{\zeta}{2}\int_{x}\EEs{q\sim \Unif(\sph), j\sim \Ber(\frac{1}{2}),S_\trn\sim \cD_{q,j}^m}{f_{q}(x)\ind{\cE_2(S_\trn,x,q)}}dx\\
    >&\frac{3\zeta}{8}\,,
\end{align*}
which completes the proof for $d=1$ by setting $\zeta = \min(\frac{t}{8m},1)$. 
  
\paragraph{Extension to general $d\geq 1$.}
To extend the base case to $d>1$, we construct $d$ separate balls and repeat the above construction on each ball individually. For $i\in[d]$, let $\sph_i = \sph^3(3ie_1,1)$ denote the sphere of a ball with radius $1$ centered at $3ie_1$. 
Consider the domain $\cX = \cup_{i\in [d]}\sph_i\cup \{\bZero\}$ as the union of $d$ non-overlapping unit balls and the origin. 
For $q_i\in \sph_i$, let $h^1_{q_i}=\ind{\sph^3(q_i,1)\cap \sph_i}$ denote the hypothesis classifying only points on the circle of $\sph^3(q_i,1)\cap \sph_i$ positive and $h^0_{q_i}=\ind{\sph_i\setminus \sph^3(q_i,1)}$ denote the hypothesis classifying only points on $\sph_i$ positive except the circle $\sph^3(q_i,1)\cap \sph_i$. Let $h^s_{q_1,\ldots,q_{d}}=\sum_{i\in[d]}h^{s_i}_{q_i}$, where $s\in\{0,1\}^{d}$ denote the hypothesis combining all $d$ balls and $\cH = \{h^s_{q_1,\ldots,q_{d}}|q_i\in \sph_i,\forall i\in [d],s\in\{0,1\}^{d}\}$. We have the VC dimension of $\cH$ is $5d$. Our target function is selected uniformly at random from $\cH$ and similar to the case of $d=1$, we assign probability $\zeta = \min(\frac{1}{d},\frac{t}{8m})$ to each circle on the balls and the remaining probability mass on the origin. Since every ball is independent with other balls and thus, we have $\EEs{h^*\sim \Unif(\cH),S_\trn\sim \cD^m}{\atk_\cD(t, h^*, S_\trn, \A)}> \frac{3d\zeta}{8}=  \min(\frac{3td}{64m},\frac{3}{8})$. 
  
 In all, there exists a target function $h^*\in \cH$ and a data distribution $\cD$ over $D_{h^*}$ such that $\EEs{S_\trn}{\atk_\cD(t, h^*, S_\trn, \A)}> \epsilon$ when $m<\frac{3td}{64\epsilon}$ for $\epsilon\leq \frac{3}{8}$.
\end{proof}
\end{document}